\documentclass[twoside,11pt]{article}

\usepackage{blindtext}

\usepackage[preprint]{jmlr2e}
\usepackage{algorithm,algpseudocode}
\usepackage{amsopn}
\usepackage{amssymb}
\usepackage{array}
\usepackage{comment}
\usepackage{dsfont}
\usepackage{mathtools}
\usepackage{makecell}
\usepackage{pifont}
\usepackage{subcaption}
\usepackage[listings,breakable]{tcolorbox}
\usepackage[normalem]{ulem}
\usepackage{xcolor}
\usepackage[capitalise]{cleveref}
\usepackage{tikz}
\usetikzlibrary{arrows.meta,positioning,shapes.geometric,calc}

\newif\ifpreprint
\preprinttrue  

\newcommand{\bigO}{\mathcal{O}}
\newcommand{\bigOt}{\tilde{\bigO}}
\newcommand{\pcg}{PCG}

\newcommand{\fal}{Falkon}
\newcommand{\chol}{Cholesky}
\newcommand{\sap}{SAP}

\newcommand{\nsap}{NSAP}

\newcommand{\epro}{EigenPro}
\newcommand{\sko}{\texttt{Skotch}}
\newcommand{\asko}{\texttt{ASkotch}}

\newcommand{\nys}{Nystr\"{o}m}
\newcommand{\algnys}{\texttt{\nys}}

\newcommand{\orc}{\mathfrak{o}}

\newcommand{\codeurl}{https://github.com/pratikrathore8/fast_krr}
\newcommand{\codeurldisplaytext}{https://github.com/pratikrathore8/fast\_krr}

\newtcolorbox{relationbox}[1][]{
  colback=gray!10,
  colframe=black,
  sharp corners,
  boxrule=1pt,
  left=5pt, right=5pt, top=2pt, bottom=2pt,
  parbox=false,     
  breakable,         
  title=\textit{\relation} #1,
}

\long\def\preprintcontent#1{\ifpreprint #1 \fi}

\long\def\nonpreprintcontent#1{\ifpreprint\else #1 \fi}

\newcolumntype{C}[1]{>{\centering\arraybackslash}m{#1}}

\newcommand{\iter}[3]{%
  #1
  \ifx\\#2\\
  \else
    _{#2}
  \fi
  \ifx\\#3\\
  \else
    ^{#3}
  \fi
}
\newcommand{\iterd}[2]{%
  \iter{d}{#1}{#2}%
}

\newcommand{\iterw}[2]{%
  \iter{w}{#1}{#2}%
}
\newcommand{\iterv}[2]{%
  \iter{v}{#1}{#2}%
}

\newcommand{\iterz}[2]{%
  \iter{z}{#1}{#2}%
}

\newcommand{\iterS}[2]{%
  \iter{S}{#1}{#2}%
}

\newcommand{\wstar}{\iterw{\star}{}}

\newcommand{\B}{\mathcal B}
\newcommand{\I}{\mathcal I}
\newcommand{\D}{\mathcal D}
\newcommand{\E}{\mathbb E}

\newcommand{\R}{\mathbb R}
\renewcommand{\S}{\mathcal S}

\newcommand{\Ec}{\mathcal E}
\newcommand{\Hc}{\mathcal H}
\newcommand{\Lfull}{\mathcal L_{\mathrm{full}}}
\newcommand{\Lind}{\mathcal L_{\mathrm{ind}}}
\newcommand{\Pc}{\mathcal P}

\newcommand{\blksz}{b}
\newcommand{\blkszstr}{blocksize}
\newcommand{\Blkszstr}{Blocksize}
\newcommand{\bg}{b_g}
\newcommand{\lamunsc}{\lambda_{\mathrm{unsc}}}

\NewDocumentCommand{\arls}{o o}{%
  \ensuremath{%
    \mathrm{ARLS}%
    \IfValueT{#2}{_{#2}}%
    \IfValueT{#1}{^{#1}}%
  }%
}
\newcommand{\rls}{RLS}

\newcommand{\Ib}{I_\B}
\newcommand{\lra}[2]{\lfloor #1 \rfloor_{#2}}
\newcommand{\lmin}{\lambda_{\textup{min}}}
\newcommand{\lmax}{\lambda_{\textup{max}}}
\renewcommand{\null}{\textup{null}}
\newcommand{\psd}{\mathbb S_{+}}
\newcommand{\pd}{\mathbb S_{++}}
\newcommand{\Kbb}{K_{\B \B}}
\newcommand{\Knys}{\hat K_{\B\B}}

\newcommand{\deff}[1]{d^{#1}}
\newcommand{\dmof}[1]{d_{\textup{max}}^{#1}}

\newcommand{\tailcond}[2]{\bar{\kappa}_{#1:#2}}

\newcommand{\aproj}{\hat{\Pi}}
\newcommand{\aprojB}{\aproj_{\B, \rho}}
\newcommand{\proj}{\Pi}
\newcommand{\projB}{\Pi_\B}

\newcommand{\Klambd}{K_{\lambda}}
\newcommand{\pinv}[1]{#1^{+}}
\DeclareMathOperator{\Tr}{Tr}
\DeclareMathOperator{\diag}{diag}

\newcommand{\cmark}{\textcolor{blue}{\ding{51}}}%
\newcommand{\xmark}{\textcolor{purple}{\ding{55}}}%

\newcommand{\derezisnki}{Derezi\'nski}
\newcommand{\michal}{Micha\l{}}
\newcommand{\stanford}{Stanford University}
\newcommand{\umich}{University of Michigan, Ann Arbor}
\newcommand{\libsvm}{LIBSVM}
\newcommand{\openml}{OpenML}
\newcommand{\sgdml}{sGDML}
\newcommand{\torchv}{torchvision}

\newcommand{\relation}{Relation to our work.}
\newcommand{\mtrn}{Mat\'{e}rn}

\newcommand{\minimz}{\textup{minimize}}

\newcommand{\dpp}{\mathrm{DPP}}
\newcommand{\bdpp}{\B_{\dpp}}
\newcommand{\tbdpp}{{\tilde \B}_{\dpp}}
\newcommand{\hdpp}[1]{#1\text{-}\dpp}

\newcommand{\prfn}[1]{\Pr(#1)}

\usepackage{lastpage}
\jmlrheading{27}{2026}{1-\pageref{LastPage}}{2/25; Revised }{}{21-0000}{Pratik Rathore, Zachary Frangella, Jiaming Yang, \michal{} \derezisnki, and Madeleine Udell}

\ShortHeadings{A Neat Solution for KRR}{Rathore, Frangella, Yang, \derezisnki, and Udell}
\firstpageno{1}

\begin{document}

\title{Have ASkotch: A Neat Solution for Large-scale Kernel Ridge Regression}

\author{\name Pratik Rathore \email pratikr@stanford.edu \\
       \addr Department of Electrical Engineering\\
       \stanford{}\\
       \AND
       \name Zachary Frangella \email zachary.frangella@granica.ai \\
       \addr Granica\\
       \AND
       \name Jiaming Yang \email jiamyang@umich.edu \\
       \addr Department of Electrical Engineering and Computer Science\\
       \umich{}\\
       \AND
       \name \michal{} \derezisnki{} \email derezin@umich.edu \\
       \addr Department of Electrical Engineering and Computer Science\\
       \umich{}\\
       \AND
       \name Madeleine Udell \email udell@stanford.edu \\
       \addr Department of Management Science \& Engineering\\
       \stanford{}\\
       }

\editor{}

\maketitle

\begin{abstract}
  Kernel ridge regression (KRR) is a fundamental computational tool, appearing in problems that range from computational chemistry to health analytics, with a particular interest due to its starring role in Gaussian process regression.
  However, full KRR solvers are challenging to scale to large datasets: both direct (i.e., \chol{} decomposition) and iterative methods (i.e., \pcg{}) incur prohibitive computational and storage costs.
  The standard approach to scale KRR to large datasets chooses a set of inducing points and solves an approximate version of the problem, inducing points KRR. 
  However, the resulting solution tends to have worse predictive performance than the full KRR solution.
  In this work, we introduce a new solver, \asko{}, for full KRR that provides better solutions faster than state-of-the-art solvers for full and inducing points KRR. 
  \asko{} is a scalable, accelerated, iterative method for full KRR that provably obtains linear convergence.
  Under appropriate conditions, we show that \asko{} obtains condition-number-free linear convergence.
  This convergence analysis rests on the theory of ridge leverage scores and determinantal point processes.
  \asko{} outperforms state-of-the-art KRR solvers on a testbed of 23 large-scale KRR regression and classification tasks derived from a wide range of application domains, demonstrating the superiority of full KRR over inducing points KRR.
  Our work opens up the possibility of as-yet-unimagined applications of full KRR across a number of disciplines.
\end{abstract}


\section{Introduction}
Kernel ridge regression (KRR) is one of the most popular methods in machine learning, with important applications in computational chemistry \citep{stuke2019chemical,blücher2023reconstructing,parkinson2023linear}, healthcare \citep{cheng2020sparse,wu2021brain,townes2023nonnegative}, and, more recently, scientific machine learning \citep{raissi2017machine,meanti2024estimating,batlle2024kernel}. 
Given a kernel function $k(x, x')$ and training set $\{(x_i,y_i)\}^{n}_{i=1}$, \textit{full KRR} seeks the function $f$ in a reproducing kernel Hilbert space $\Hc$ that satisfies
\begin{equation}
\label{eq:rkhs_regress}
\underset{{f\in \Hc}}{\minimz}~\frac{1}{2}\sum_{i=1}^{n}\left(f(x_i)-y_i\right)^2+\frac{\lambda}{2} \|f\|_\Hc^2.
\end{equation}
The celebrated representer theorem \citep{kimeldorf1970correspondence, scholkopf2002learning} says that the solution of \eqref{eq:rkhs_regress} lies in the subspace 
\[
\Hc_{n} = \left \{ f \in \Hc: f(x) = \sum_{j=1}^{n}w^{(j)} k(x,x_j),~ \text{where}~ w^{(j)} \in \R \right \}.
\]
With this reduction, the infinite-dimensional optimization problem in \eqref{eq:rkhs_regress} collapses to a finite-dimensional, convex, least-squares problem:
\begin{equation}
\label{eq:primal_krr}
    \underset{{w \in \R^n}}{\minimz}~ \Lfull(w) \coloneqq \frac{1}{2} \| K w-y\|^2 + \frac{\lambda}{2} \| w \|_K^2.
\end{equation}
Here $K$ is a $n \times n$ kernel matrix whose entries are given by $K_{ij} = k(x_i,x_j)$. 
The optimal weights $\wstar$ in \eqref{eq:primal_krr} are the solution of the linear system
\begin{equation}
\label{eq:krr_linsys}
    \Klambd \wstar = y,
\end{equation}
where $\Klambd \coloneqq K + \lambda I$.

Despite its popularity, full KRR is challenging to scale to large datasets: the state-of-the-art methods for solving \eqref{eq:krr_linsys} have costs that are superlinear in $n$.
Direct methods (e.g., \chol{} decomposition) have $\bigO(n^3)$ computational complexity and $\bigO(n^2)$ storage complexity.
Therefore, when $n \gtrsim 10^4$, \chol{} decomposition becomes unsuitable for solving \eqref{eq:krr_linsys}.
Iterative methods such as preconditioned conjugate gradient (PCG) have per-iteration complexity $\bigO(n^2)$.
In addition to expensive iterations, most PCG methods employ low-rank preconditioners, which require $\bigO(nr)$ storage, where $r$ is a rank parameter which needs to be sufficiently large to ensure effective preconditioning.
Therefore, when $n \gtrsim 10^6$, PCG is either (i) too slow or (ii) requires too much memory to solve \eqref{eq:krr_linsys}.

The standard approach for addressing the scalability issues of full KRR is to select a set of $m$ \textit{inducing points} from the training set, and solve the \textit{inducing points KRR} problem \citep{scholkopf2002learning,rudi2015less}.
The inducing points KRR objective is
\begin{equation}
\label{eq:inducing_krr}
\underset{{w} \in \R^{m}}{\textup{minimize}}~ \Lind(w) \coloneqq \frac{1}{2}\|K_{nm}w-y\|^2+\frac{\lambda}{2}\|w\|^2_{K_{mm}},
\end{equation}
where $K_{nm}$ is the $n\times m$ matrix of the columns of $K$ identified by the inducing points, and $K_{mm}$ is the corresponding principal submatrix of $K$. The optimal weights $\wstar$ in \eqref{eq:inducing_krr} are the solution of the linear system
\begin{equation}
\label{eq:ikrr_lin_sys}
(K_{nm}^{T}K_{nm}+\lambda K_{mm}) \wstar = K_{nm}^{T}y.
\end{equation}
The linear system \eqref{eq:ikrr_lin_sys} accesses $K$ only through $K_{nm}$ and $K_{mm}$, 
reducing the storage and per-iteration computation needed to solve the KRR problem.
As a result, \chol{} decomposition has $\bigO(nm^2 + m^3)$ computational complexity and requires $\bigO(m^2)$ storage.
Consequently, when $m \gtrsim 10^5$, \chol{} is unsuitable for solving \eqref{eq:ikrr_lin_sys}.
State-of-the-art \pcg{} methods for inducing points KRR, like \fal{} \citep{rudi2017falkon,meanti2020kernel} and KRILL \citep{diaz2023robust}, 
require $\bigO(m^3)$ time to construct the preconditioner and $\bigO(m^2)$ storage.
Thus, when $m \gtrsim 10^5$, these \pcg{} methods are also unsuitable for solving \eqref{eq:ikrr_lin_sys}.

Alas, scale matters: previous work has established that increasing the number of inducing points leads to better predictive performance \citep{frangella2023randomized,diaz2023robust,abedsoltan2023toward}.
Moreover, full KRR often allows better predictive performance than inducing points KRR \citep{wang2019exact,frangella2023randomized,diaz2023robust}.
Therefore, a new, more scalable approach to KRR is needed for better prediction on large-scale tasks.

\begin{figure}[t]
    \centering
    \nonpreprintcontent{\includegraphics[width=0.8\linewidth]{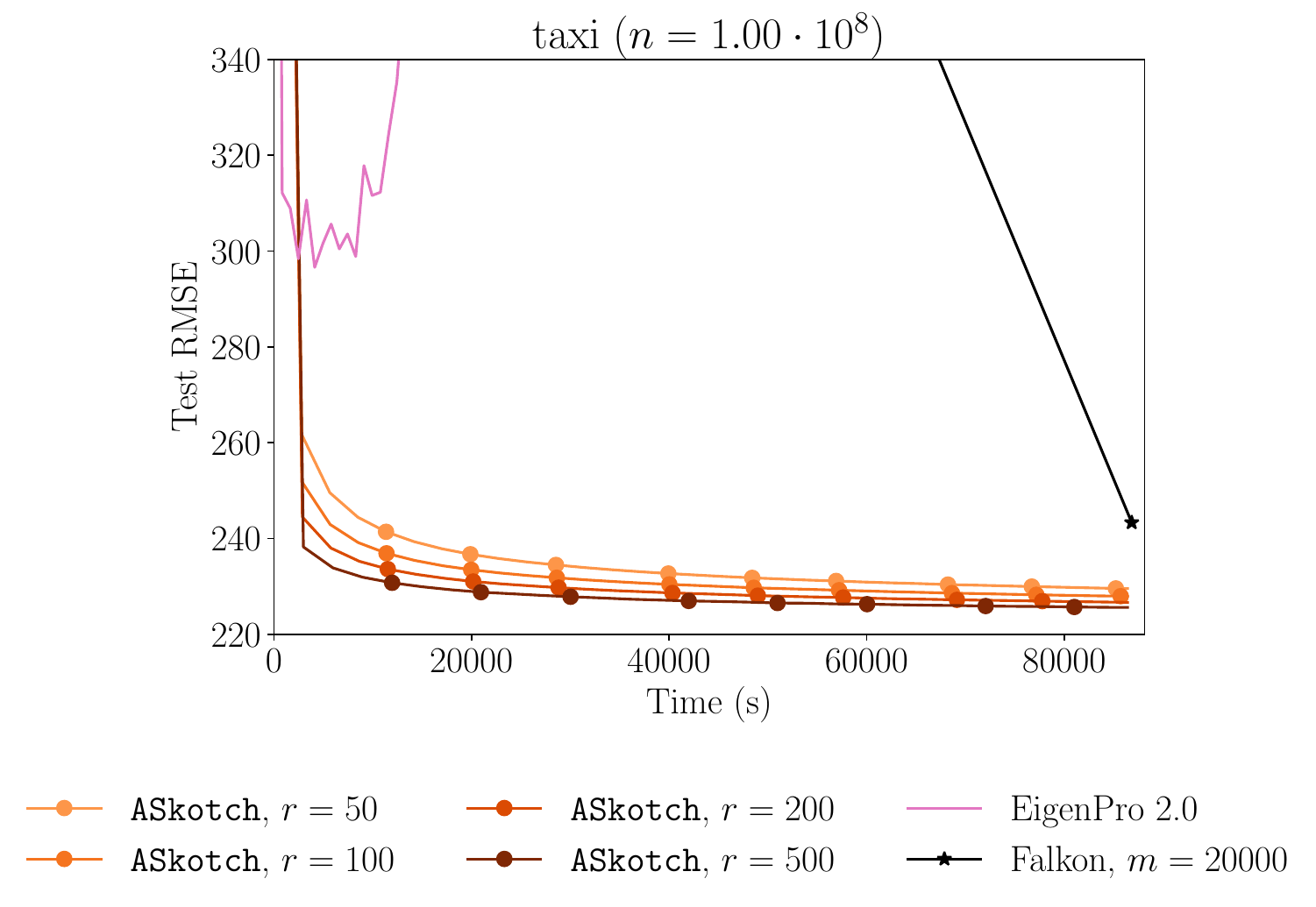}}%
    \preprintcontent{\includegraphics[width=\linewidth]{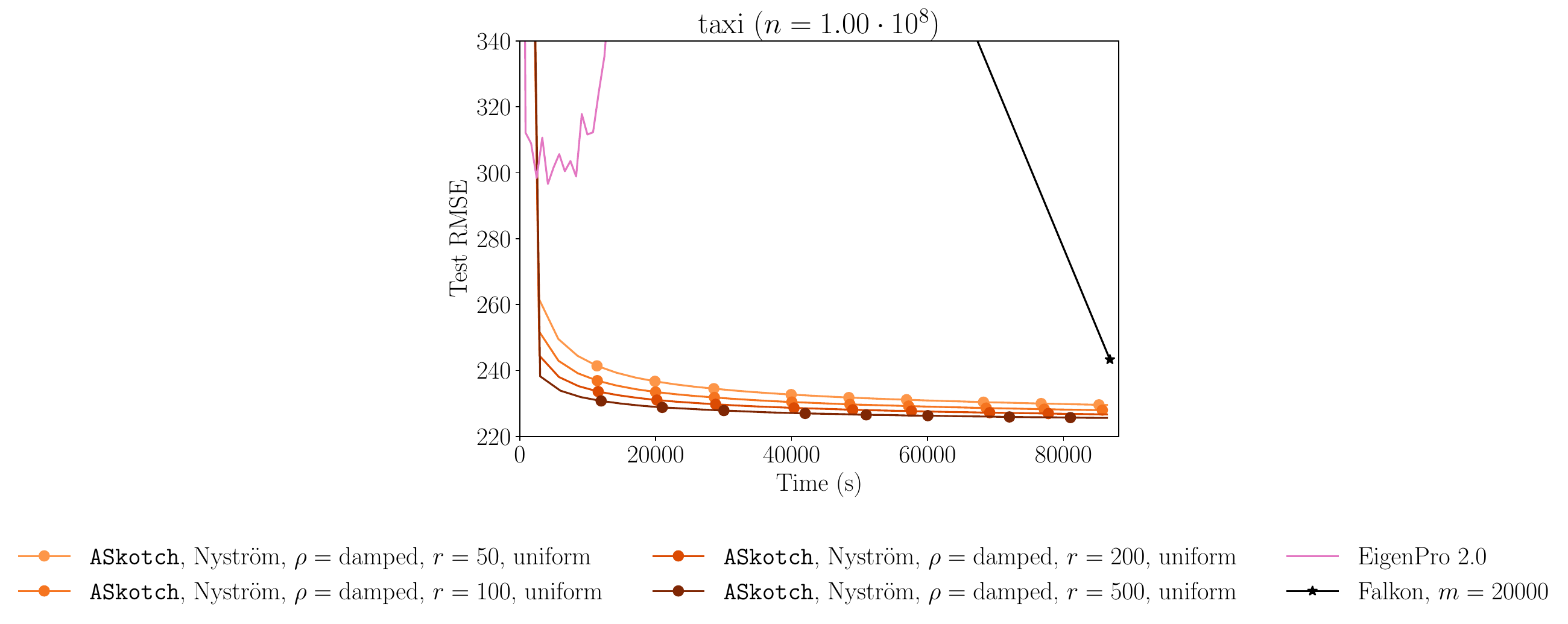}}%
    \caption{
    Full KRR is advantagenous over inducing points KRR, even for large problems.
    Our method \asko{}, run with its default hyperparameters, outperforms the state-of-the-art for both full and inducing points KRR on a subsample of the taxi dataset.
    \fal{} is limited to $m = 2 \cdot 10^4$ inducing points due to memory constraints.
    State-of-the-art Nystr{\"o}m \pcg{} methods \citep{frangella2023randomized} Gaussian Nystr{\"o}m \citep{frangella2023randomized} and Randomly Pivoted Cholesky \citep{diaz2023robust,epperly2025embrace}, each with a rank $r=50$ preconditioner, fail to complete a single iteration.
    \epro{} 2.0 and \epro{} 3.0 (not shown) diverge on their default hyperparameters.
    All methods have a 24-hour time limit and are run on a single 48 GB NVIDIA RTX A6000 GPU.}
    \label{fig:taxi_intro}
\end{figure}

In this work, we turn conventional wisdom around, demonstrating that better, faster solutions
can be achieved by targeting the full KRR problem rather than the inducing points KRR approximation.
We do so by developing the iterative method \asko{} to solve full KRR.
\asko{} has (i) per-iteration costs that are linear in $n$ and (ii) storage costs that are linear in the preconditioner rank $r$ but are independent of $n$.
These properties allow \asko{} to scale to larger datasets than existing state-of-the-art methods for solving \eqref{eq:krr_linsys}.
\asko{} possesses strong theoretical guarantees and easy-to-set hyperparameters that deliver reliable performance.
Furthermore, \asko{} exploits parallelism on modern hardware accelerators like GPUs, which is essential for scaling to large KRR problems.


\cref{fig:taxi_intro} demonstrates the core messages of our work---full KRR can scale to massive datasets and outperform inducing points KRR with respect to predictive performance.
On a dataset with $n  = 10^8$ samples, \asko{} scales better than the standard approach to full KRR: \pcg{} is unable to complete a single iteration within a 24-hour time limit.
Moreover, \asko{} is more reliable than the KRR solvers \epro{} 2.0 and \epro{} 3.0 \citep{ma2019kernel, abedsoltan2023toward} when all three methods are run with their default hyperparameters: \epro{} 2.0 and \epro{} 3.0 diverge, while \asko{} reaches a root mean square error (RMSE) of approximately 230.
Finally, \asko{} outperforms \fal{} on RMSE.
While \cite{meanti2020kernel} scales an optimized version of \fal{} to the full taxi dataset $(n \approx 10^9)$ with $m = 10^5$ inducing points, \asko{} is solving a problem of size $10^8 \cdot 10^8 = 10^{16}$, which is two orders of magnitude larger than the problem solved by \fal{} ($10^5 \cdot 10^9 = 10^{14}$).
For additional discussion, please see \cref{subsec:ikrr_related}.

\cref{tab:algo_comparison} compares the capabilities of \asko{} to state-of-the-art methods for KRR.
\asko{} is the only method that solves full KRR while having a moderate memory requirement, reliable default hyperparameters, and theoretical convergence guarantees.
In the \href{https://github.com/EigenPro}{EigenPro GitHub repos}, there is no option for the user to set hyperparameters such as the learning rate or gradient batchsize. 
Since the \epro{} methods diverge on taxi and several other datasets in this work, we conclude that the default hyperparameters for \epro{} can be unreliable.

\begin{table}
\centering
    \begin{tabular}{C{2.5cm}C{2cm}C{3cm}C{3cm}C{2cm}}
        Algorithm & Full KRR? & Memory-efficient? & Reliable defaults? & Converges? \\ 
        \hline
        \asko{} & \cmark & \cmark & \cmark & \cmark \\ 
        \hline \epro{} 2.0 & \cmark & \cmark & \xmark & \cmark \\
        \hline \epro{} 3.0 & \xmark & \cmark & \xmark & \xmark \\
        \hline \pcg{} & \cmark & \xmark & \cmark & \cmark \\
        \hline \fal{} & \xmark & \xmark & \cmark & \cmark
    \end{tabular}
    \caption{ 
    A comparison of the capabilities of \asko{} and state-of-the-art methods for KRR.
    ``Reliable defaults'' indicates whether the method comes with default hyperparameters that work well in practice.
    ``Converges'' indicates whether the method has a rigorous linear convergence guarantee.
    }
    \label{tab:algo_comparison}
\end{table}

\begin{table}[H]
    \centering
    \scriptsize
    \begin{tabular}{C{4cm}C{4cm}C{3cm}C{2.5cm}}
        Algorithm & Per-iteration complexity & Storage complexity & Total computational complexity  \\
        \hline
        \pcg{} & $\bigO(n^2)$ & $\bigO(\deff{\lambda}(K)n)$ & $\tilde{\bigO}(n^2)$ \\
        \hline
        \epro{} 2.0  & $\bigO(n  \bg +rs)$ & $\bigO(rs)$ & $\bigOt \left( \frac{\lambda_r(K)}{\lmin(K)} n^2 \right)$ \\
        \hline
        \asko{} & $\bigO(nb)$ & $\bigO\left(\deff{\lambda}(K)b\right)$ & $\bigOt\left(n^2\right)$  \\
        \hline
    \end{tabular}
    \caption{Iteration and storage complexities for state-of-the-art full KRR methods.
    Storage complexity refers to memory that is explicitly used by the algorithm (this excludes matrix-vector product oracles).
    Total computational complexity refers to the total cost required by the algorithm to compute an $\epsilon$-approximate solution, omitting logarithmic dependencies upon $\epsilon$.
    Here, $b_g \leq n$ is the stochastic gradient batchsize used by \epro{} 2.0, $s \leq n$ is the sample size used to construct the rank-$r$ preconditioner in \epro{} 2.0, $b \leq n$ is the \blkszstr{} in \asko{}, and
    $\deff{\lambda}(K)$ is the $\lambda$-effective dimension of $K$ (\cref{def:rls}). 
    \asko{} has the same total cost as PCG but has lower storage and per-iteration costs.}
\label{tab:complexity_comparison}
\end{table}

\subsection{Contributions}
We make significant algorithmic and theoretical contributions to the literature on KRR solvers and sketch-and-project methods.

\begin{enumerate}
    \item We develop the iterative method \asko{} for solving full KRR. 
    The key to our approach lies in carefully combining two methodologies: sketch-and-project iterative solvers \citep{gower2015randomized} and \nys{} matrix
    approximations \citep{williams2000using,tropp2017fixed}.
    We incorporate the \nys{} approximation into the sketch-and-project update in a way that preserves convergence, and combine this with Nesterov acceleration to solve \eqref{eq:krr_linsys} with reduced time and memory requirements.
    \item We establish fine-grained convergence guarantees for \asko{} based on ridge leverage score sampling, determinantal point processes, and careful spectral approximation arguments.
    A detailed overview of our theoretical contributions is presented in \cref{subsec:our_techniques}. 
    \item We perform an extensive empirical evaluation comparing \asko{} to the existing state-of-the-art for full and inducing points KRR.
    We show that \asko{} outperforms  \epro{} 2.0, \epro{} 3.0, \pcg{}, and \fal{} on predictive performance for 23 large-scale KRR problems (typically, $n \geq 10^5$). 
    Moreover, we show that full KRR consistently obtains equivalent or better predictive performance than inducing points KRR, establishing the superiority of full KRR when combined with \asko{}.
    We provide \href{\codeurl}{open-source code} in PyTorch to make it easy to reproduce our experiments and run \asko{} on new problems on CPU and GPU hardware. 
\end{enumerate}

\subsection{Our Techniques}
\label{subsec:our_techniques}
Our main theoretical contributions lie in the convergence analysis of \asko{}.
Establishing convergence of \asko{} requires overcoming two fundamental challenges, discussed below.

\textbf{Technical Contribution 1:} \textit{First-moment analysis of \sap{} projector via ARLS-to-DPP reduction}. 
Our first contribution overcomes a fundamental limitation in the existing theory of sketch-and-project (\sap{}) solvers.
The convergence of \sap{} methods is controlled by a random projection matrix that we call the \textit{\sap{} projector}, with the convergence rate being determined
by the smallest eigenvalue of the expected \sap{} projector.
Existing approaches for analyzing the smallest eigenvalue of the expected \sap{} projector \citep[e.g.,][]{mutny2020convergence,derezinski2024solving, derezinski2024fine} require either (1) an expensive, impractical sub-sampling scheme known as a Determinantal Point Process (DPP), or (2) preprocessing the input matrix with the Randomized Hadamard Transform, which is prohibitive for large problems due to its $\bigO(n^2)$ storage cost.

To address this challenge, we adopt a different sampling scheme called approximate ridge leverage score (\arls{}{}) sampling \citep{alaoui2015fast}, which is less expensive (both in compute and storage) than either of these approaches. 
Our convergence analysis develops a reduction from \arls{}{}s to DPPs (\cref{lem:rls_sample}), by relating the former to the marginals of a DPP through a novel measure concentration argument. 
This reduction allows us to establish a non-trivial lower bound on the smallest eigenvalue of the expected \sap{} projector when using \arls{}{} sampling (\cref{lem:convergence_rls}).
We believe this result is of independent interest and could be used to improve the convergence analysis of other iterative solvers that use the sketch-and-project paradigm.

\textbf{Technical Contribution 2:} \textit{First-moment analysis of \nys{} projector}.
\asko{} replaces the linear system solve in \sap{} with an approximate solution based on \nys{} sketch-and-solve \citep{bach2013sharp, alaoui2015fast, frangella2023randomized}.
In contrast, prior sketch-and-project solvers \citep{gower2015randomized,gower2018accelerated,derezinski2024solving, derezinski2024fine} assume the \sap{} linear system is either solved exactly or to sufficient accuracy at each iteration using an iterative solver.
The use of Nystr{\"o}m sketch-and-solve complicates the analysis in two ways:
1) the convergence rate is no longer controlled by an expected projection matrix;
2) the \nys{} solution is generally not close to the exact solution \citep{frangella2023randomized}.
The first issue prevents us from directly applying the classical convergence analysis of \sap{}, which is deeply reliant on the \sap{} projector actually being a projection matrix.
The second issue means we cannot appeal to existing inexact \sap{} theory, which requires the approximate solution to be close to the exact solution.

Remarkably, although the matrix that controls convergence of \asko{} is not the expectation of a projection matrix, it is the expectation of a matrix that is \emph{nearly} a projection matrix.
We refer to this matrix as the \textit{\nys{} projector}.
Using careful spectral approximation arguments, we show that the expected \nys{} projector is lower bounded in the Loewner ordering by the expected \sap{} projector multiplied by a factor that captures the price of replacing the \sap{} linear system solve with \nys{} sketch-and-solve.
Thus, we are able to transfer the first-moment projector analysis from Theoretical Contribution 1 (which is for exact \sap{}) to \asko{}, which yields a lower bound on the smallest eigenvalue of the expected \nys{} projector (\cref{thm:appx_proj_small_eig}). 
This lower bound on the smallest eigenvalue of the expected \nys{} projector lets us establish a fine-grained convergence rate for \asko{} (\cref{thm:mast_conv_thm}).
Using this theory, we prove that when the effective dimension of the kernel matrix is not too large, \asko{} enjoys fast condition-number-free linear convergence, for a near-optimal $\bigOt(n^2\log \left( \frac{1}{\epsilon} \right))$ computational complexity to reach an $\epsilon$-approximate full KRR solution (\cref{corr:log_lin_conv}).
This convergence rate is comparable to that of \pcg{} while improving over \epro{} 2.0 (\cref{tab:complexity_comparison}).

\subsection{Roadmap}
\cref{sec:algorithms} introduces \asko{} (and its non-accelerated variant, \sko{}), and describes the techniques (sketch-and-project, \nys{} approximation, automatic computation of stepsizes) that are used in this algorithm.
\cref{sec:related_work} reviews existing methods to solve full and inducing points KRR and places \asko{} in the context of these works.
\cref{sec:asko_key_quantities} introduces several key quantities used in the convergence analysis of \asko{}.
\cref{sec:nys_proj_analysis} develops lower bounds on the quantities introduced in \cref{sec:asko_key_quantities} by exploiting connections between ARLS sampling and DPPs.
\cref{sec:asko_cvg} uses the lower bounds from \cref{sec:nys_proj_analysis} to establish a fine-grained convergence rate for \asko{} on full KRR.
\cref{sec:experiments} demonstrates the superior performance of \asko{} over state-of-the-art full and inducing-points KRR solvers.

\subsection{Notation}
We use $[n]$ to denote the set $\{1, 2, \ldots, n\}$. 
Throughout the paper, $\B$ is a subset of $[n]$ that is sampled randomly according to some distribution.
$\Ib$ concatenates the rows of $I_n$ indexed by $\B$. 
Given a square matrix $A$, $A_{\B\B}$ denotes the principal submatrix of $A$ indexed by $\B$.


$\Tr(\cdot)$ is the trace of a matrix, $\det(\cdot)$ is the determinant of a matrix, and $\null(\cdot)$ is the null space of a matrix. 
$\pd^n$ ($\psd^n$) denotes the convex cone of pd (psd) matrices in $\R^{n \times n}$.
The symbol $\preceq$ denotes the Loewner order on the convex cone of psd matrices: $A \preceq B$ means $B - A$ is psd.
For a square matrix $A$ and constant $\lambda$, we define $A_\lambda \coloneqq A + \lambda I$.
We use $\| \cdot \|$ to denote norms; $\| \cdot \|$ with no subscript is the Euclidean norm.
For a matrix $A \in \psd^n$, its eigenvalues in decreasing order are $\lambda_1(A) \geq \lambda_2(A) \geq \ldots \geq \lambda_n(A)$. 
$\pinv{A}$ is the Moore-Penrose pseudoinverse of a matrix $A$.
If $A \in \psd^{n}$, the smallest non-zero eigenvalue is denoted by $\lmin^{+}(A)$.
For a positive integer $b$ and matrix $A$, $\lra{A}{b}$ is the best rank-$b$ approximation of $A$ with respect to the spectral norm.

The symbols $\Omega(\cdot)$, $\Theta(\cdot)$, and $\bigO(\cdot)$ are the standard quantities from asymptotic complexity analysis.
$\tilde \Omega(\cdot)$, $\tilde \Theta(\cdot)$, and $\bigOt(\cdot)$ are used to hide poly-log factors.
\section{Algorithms}
\label{sec:algorithms}
We propose \asko{} (\textbf{A}ccelerated \textbf{s}calable \textbf{k}ernel \textbf{o}ptimization and \textbf{t}raining with \textbf{c}oordinate sketch-and-project and approximate \textbf{H}essians, \cref{alg:askotch}) to solve full KRR.
We introduce the \asko{} algorithm in \cref{subsec:asko}, and we follow this by describing the key building blocks of the method, such as approximate sketch-and-project (\cref{subsec:sap_approx}), randomized \nys{} approximation (\cref{subsec:nys_appx}), and automatic computation of the stepsize  (\cref{subsec:auto_stepsize}).
\asko{} comes with default hyperparameter settings (\cref{subsec:default_hyperparams}), which reduces the effort required by practitioners to tune the method; we use these defaults in our experiments (\cref{sec:experiments}).

\subsection{\asko{}}
\label{subsec:asko}

\begin{algorithm}[]
    \caption{\asko{}}
    \label{alg:askotch}
    \scriptsize
    \begin{algorithmic}
        \Require \blkszstr{} $\blksz$, coordinate sampling distribution $\Pc$, acceleration parameters $\hat \mu, \hat \nu$, acceleration flag \textrm{use\_accel}, rank $r$, damping $\rho$, kernel oracle $K$, targets $y$, ridge parameter $\lambda$, number of iterations $N$, initialization $\iterw{0}{}$
        \State
        \State \textbf{\# Compute acceleration parameters}
        \State $\beta \gets 1 - \sqrt{\hat \mu / \hat \nu}$
        \State $\gamma \gets 1 / \sqrt{\hat \mu \hat \nu}$
        \State $\alpha \gets 1 / (1 + \gamma \hat \nu)$
        \State $\iterv{0}{} \gets \iterw{0}{}, \iterz{0}{} \gets \iterw{0}{}$
        \State 
        \For{$i = 0, 1, \ldots, N - 1$}

            \State Sample a block of coordinates $\B \subseteq [n]$ according to $\mathcal P$
            \hfill
            \Comment{$|\B| = \blksz$}
            \State $\Knys \gets \texttt{\nys}( \Kbb, r)$
             \hfill \Comment{Low-rank approximation of $\Kbb$}

            \State $L_{P_\B} \gets \texttt{get\_L}(\Kbb + \lambda I, \hat{K}_{\B \B}, \rho)$
             \hfill \Comment{Via powering}
            \State
            \State \textbf{\# Update iterates}
            \If{\textrm{use\_accel}}
                \State $\iterd{i}{} \gets \Ib^T (\Knys + \rho I)^{-1} ((\Klambd)_{\B :} \iterz{i}{} - y_\B)$  \Comment{Approximate projection; costs $\bigO(nb + rb)$}
                \State $\iterw{i+1}{} \gets \iterz{i}{} - (L_{P_\B})^{-1} \iterd{i}{}$ \Comment{Costs $\bigO( \blksz)$}
                \State $\iterv{i+1}{} \gets \beta \iterv{i}{} + (1 - \beta) \iterz{i}{} - \gamma (L_{P_\B})^{-1} \iterd{i}{}$
                 \hfill \Comment{Costs $\bigO(n)$}
                \State $\iterz{i+1}{} \gets \alpha \iterv{i}{} + (1 - \alpha) \iterw{i+1}{}$
                 \hfill \Comment{Costs $\bigO(n)$}
                \Else 
                \State \textbf{\# Non-accelerated version (\sko{})}
                \State $w_{i + 1} \gets w_i - (L_{P_\B})^{-1} \iterd{i}{}$ \Comment{Costs $\bigO( \blksz)$}
             \EndIf
        \EndFor
        \State \Return $\iterw{N}{}$ \Comment{Approximate solution to $\Klambd w = y$}
    \end{algorithmic}
\end{algorithm}

At each iteration, \asko{} randomly samples a block $\B$ containing $\blksz$ distinct coordinates (we omit the dependence of $\B$ on $i$ to avoid notational clutter).
\asko{} then computes a rank-$r$ \nys{} approximation of $\Kbb$ and computes a
\textit{preconditioned smoothness constant} for the block, $L_{P_\B}$, which is used to set the stepsize.
Next, \asko{} computes the approximate projection step $\iterd{i}{}$.
Finally, \asko{} combines the approximate projection step $\iterd{i}{}$ with Nesterov acceleration \citep{nesterov2018lectures} to update its estimate of the solution.

\asko{} is compatible with any coordinate sampling distribution $\Pc$.
Our implementation uses two approaches inspired by popular sampling distributions in the kernel literature: uniform sampling and ARLS sampling.
The uniform sampling distribution places an equal weight of $1/n$ on each coordinate, while ARLS sampling weights coordinates with probabilities proportional to their approximate RLSs (\cref{def:rls}). 
In our implementation, we approximate the RLSs using BLESS \citep{rudi2018fast,gautier2019dppy}. 

\asko{} has a flag indicating whether or not to use Nesterov acceleration: when this flag is False, acceleration is disabled.
Throughout the rest of the paper, we will use the name \sko{} to refer to the non-accelerated version of \asko{}.

\subsection{Approximate Sketch-and-Project}
\label{subsec:sap_approx}
Sketch-and-project (\sap{}) \citep{gower2015randomized,richtarik2020stochastic} is a randomized, iterative framework to solve consistent linear systems.
\sap{} for the full KRR linear system \eqref{eq:krr_linsys} takes in two parameters: (i) a fixed, pd matrix $Q \in \pd^n$ and (ii) a distribution $\D$ over matrices in $\R^{b \times n}$, where $b$ is a positive integer.
At iteration $i$, \sap{} draws a sketching matrix $\iter{S}{i}{} \stackrel{\mathrm{i.i.d.}}{\sim} \D$ and computes the next iterate by solving the optimization problem
\begin{equation}
    \label{eq:sap_opt_problem}
    \begin{array}{lll}
        \iterw{i+1}{} = & \arg\min\limits_{w} & \|w - \iterw{i}{}\|_Q^2 \\[1.5ex]
        & \mbox{subject to} & \iterS{i}{} \Klambd w = \iterS{i}{} y. \\ 
    \end{array}   
\end{equation}
In other words, \sap{} progresses by projecting the current iterate $\iterw{i}{}$ (with respect to the $Q$-norm) onto the solution space of a sketched version of \eqref{eq:krr_linsys}.
Importantly, when $Q = \Klambd$ and $\mathcal D$ is the uniform distribution over row selection matrices of size $b \times n$,
the \sap{} update has the following closed form: 
\begin{align}
       \iterw{i+1}{} &= \iterw{i}{} - Q^{-1} \Klambd \iterS{i}{}^T \pinv{(\iterS{i}{} \Klambd Q^{-1} \Klambd \iterS{i}{}^T)} (\iterS{i}{} \Klambd \iterw{i}{} - \iterS{i}{}y) \label{eq:sap_update} \\
       &= \iterw{i}{} - \Ib^T (\Kbb + \lambda I)^{-1}((\Klambd)_{\B:} \iterw{i}{} - y_{\B}). \label{eq:rand_newton_full_krr}
\end{align}

Nesterov acceleration has been applied to a wide range of iterative optimization methods to obtain algorithms with faster convergence rates \citep{allenzhu2016even, liu2016accelerated,allenzhu2018katyusha,ye2020nesterov}.
\sap{} is no exception---\cite{gower2018accelerated} combines \sap{} with Nesterov acceleration and shows that Nesterov-accelerated \sap{} (\nsap{}) converges at least as fast as its non-accelerated counterpart.

In theory, \nsap{} improves over PCG when $b = o(n^{2/3})$.
However, a larger \blkszstr{} $b$ is often beneficial for \sap{} methods---for example, the convergence rate for the randomized Newton method enjoys a superlinear speedup as $b$ increases \citep{gower2015randomized}.
Using a direct method to solve the linear system in \eqref{eq:sap_update} requires $\bigO(b^3)$ computation, which becomes slow for $b \gtrsim 10^4$.
\cite{derezinski2024fine, derezinski2024solving} propose approximating the solution of this linear system propose inexactly solving this linear system
to a tolerance $\epsilon$ using PCG with a preconditioner $P$. 
Doing so requires $\bigO(b^2 \sqrt{\kappa_P} \log(1 / \epsilon))$ computation, where $\kappa_P$ is the preconditioned condition number.
However, for $b \gtrsim 10^4$, repeatedly applying PCG may be too slow to be practical.

\asko{} also approximates the \sap{} update, but it does so in a way that allows for a \blkszstr{} $b \gtrsim 10^4$.
The matrix $\Kbb + \lambda I$ in the \sap{} update \eqref{eq:rand_newton_full_krr} is replaced by a regularized rank-$r$ \nys{} approximation (\cref{subsec:nys_appx}), 
$\Knys + \rho I$, where $\rho > 0$.
This \nys{} approximation can be computed rapidly, and the resulting linear system can be solved in $\bigO(br)$ time.

The theoretical analysis of \sap{} uses the fact that $Q^{-1} \Klambd \iterS{i}{}^T \pinv{(\iterS{i}{} \Klambd Q^{-1} \Klambd \iterS{i}{}^T)} \iterS{i}{} \Klambd$ in \eqref{eq:sap_update} is a projection with respect to the $Q$-inner product \citep{gower2015randomized}. However, by replacing $\iterS{i}{} \Klambd Q^{-1} \Klambd \iterS{i}{}^T$ by a low-rank approximation, \asko{} loses this projection property.
Consequently, \asko{} requires a (dynamic) stepsize to converge, unlike \sap{}, which is guaranteed to converge with a constant stepsize of 1.
We provide a principled, automated method for computing this stepsize in \cref{subsec:auto_stepsize}.

\subsection{Randomized \nys{} Approximation}
\label{subsec:nys_appx}
The \nys{} approximation \citep{williams2000using,bach2013sharp, alaoui2015fast, tropp2017fixed} is a well-known technique for producing low-rank approximations to psd matrices.
Since kernel matrices have fast spectral decay (i.e., they have approximate low-rank structure) \citep{caponnetto2007optimal,bach2013sharp,tu2016large,ma2017diving, belkin2018approximation} we replace $\Kbb$ in \eqref{eq:rand_newton_full_krr} with a \nys{} approximation, which is easier to invert.

Given a symmetric psd matrix $M \in \psd^p$, the randomized \nys{} approximation of $M$ with respect to a random test matrix $\Omega \in \R^{p\times r}$ is
\begin{equation*}
    \hat M= (M\Omega) \pinv{\left(\Omega^{T}M\Omega\right)}(M\Omega)^{T}.
\end{equation*}
Common choices for $\Omega$ include standard Gaussian random matrices, randomized Hadamard transforms, and sparse sign embeddings \citep{tropp2017fixed}. 
The latter two test matrices reduce the computational cost of the sketch $M\Omega$.
Our theoretical analysis uses sparse sign embeddings due to their computational efficiency, but our implementation uses Gaussian random matrices for simplicity.
In the future, we will extend our implementation to use randomized Hadamard transforms and sparse sign embeddings. 

Our practical implementation, \algnys{} (\cref{alg:nystrom}), follows \citet[Algorithm 3]{tropp2017fixed}. 
\algnys{} takes in a psd matrix $M \in \psd^p$ and rank $r$ and outputs a low-rank approximation $\hat M = \hat U \diag(\hat \Lambda) \hat U^T$ to $M$ in $\bigO(p^2 r + p r^2)$ time\footnote{For sparse sign embeddings, this complexity is reduced to $\bigOt(p r + p r^2)$.}, where $\hat U \in \R^{p \times r}$ is an orthogonal matrix containing approximate top-$r$ eigenvectors of $M$ and $\hat \Lambda \in \R^r$ is a vector containing approximate top-$r$ eigenvalues of $M$.
Note that \algnys{} never forms $\hat M$ as a matrix; rather, it returns the factors $\hat U$ and $\hat \Lambda$.
For more details, please see \cref{subsec:nystrom}.

At each iteration, \asko{} computes a randomized \nys{} approximation of the block kernel $\Kbb$ in the \sap{} update \eqref{eq:rand_newton_full_krr}.
This approximation is always used in conjunction with a damping $\rho > 0$ to ensure positive definiteness.
For any vector $v \in \R^p$, $(\hat{M} + \rho I)^{-1} v$ can be computed in $\bigO(p r)$ time using the Woodbury formula, which lets \asko{} cheaply compute the approximate projection step $\iterd{i}{}$.

\subsection{Automatic Computation of the Stepsize}
\label{subsec:auto_stepsize}
The stepsize is often challenging to set in iterative optimization methods: if the stepsize is too large, \asko{} will diverge, while if the stepsize is too small, \asko{} will make little progress towards the solution.
In practice, tuning the stepsize can require an expensive hyperparameter search, which is inconvenient for practitioners.
We present a practical approach for automatically computing the stepsize in \asko{}.

Our idea is to use a preconditioned smoothness constant to set the stepsize in \asko{}. 
The preconditioned smoothness constant in \asko{} is defined as
\[
L_{P_\B} \coloneqq \lambda_1 \left( \left(\Knys + \rho I \right)^{-1/2} \left( \Kbb + \lambda I \right) \left( \Knys + \rho I \right)^{-1/2} \right).
\]
\asko{} uses randomized powering \citep{kuczynski1992estimating,martinsson2020randomized} to compute $L_{P_\B}$. 
The total cost of the procedure is $\bigOt(b^2)$.
In practice, \asko{} uses just 10 iterations of powering. 
Hence, the cost of computing $L_{P_\B}$ is dominated by the $\bigO(nb)$ cost of computing the search direction $d_i$ in \asko{}. 
\asko{} computes $L_{P_\B}$ using the subroutine \texttt{get\_L} (\cref{alg:get_L}).

In \cref{sec:asko_key_quantities,sec:asko_cvg}, we show that our approach for setting the stepsize guarantees that \sko{} and \asko{} obtain linear convergence.

\subsection{Default Hyperparameters}
\label{subsec:default_hyperparams}

We provide default recommendations for setting hyperparameters in \asko{}.
These recommendations are summarized in \cref{tab:asko_default_hyperparams}.

\begin{itemize}
    \item \Blkszstr{} $b$ and rank $r$: Choose as large as hardware allows.
    A good, general default is $b = n/100$ and $r=100$.
    \item Damping $\rho$: Set adaptively for each block $\B$; we use $\rho = \lambda + \lambda_r(\hat{K}_{\B \B})$.
    \item Sampling distribution $\Pc$: Use the uniform sampling distribution. 
    While we use ARLS sampling to prove fast convergence rates for \asko{}, it (i) performs similarly to uniform sampling in our ablations (\cref{sec:experiments}) and (ii) computing approximate RLSs via BLESS costs $\bigOt(nb^2)$.
    \item Acceleration parameters $\hat \mu$ and $\hat \nu$: Set $\hat \mu = \lambda$ and $\hat \lambda = n / b$.
    These choices are guided by the convergence analysis of \asko{}.
    However, the user must ensure that $\hat \mu \leq \hat \nu$ and $\hat \mu \hat \nu \leq 1$.
    We believe it is possible to make the acceleration in \asko{} parameter-free, as done in \cite{derezinski2025faster}, but we leave this extension to future work.
\end{itemize}

\begin{table}[]
    \centering
    \begin{tabular}{cc}
        Hyperparameter & Default recommendation \\
        \hline
        $b$ & $n / 100$  \\
        $\Pc$ & Uniform \\
        $\hat \mu$ & $\lambda$ \\
        $\hat \nu$ & $n / b$ \\
        $r$ & $100$ \\
        $\rho$ & $\lambda + \lambda_r(\hat{K}_{\B \B})$
    \end{tabular}
    \caption{Default hyperparameters for \asko{}.}
    \label{tab:asko_default_hyperparams}
\end{table}
\section{Related Work}
\label{sec:related_work}
We review existing solvers for full and inducing points KRR and discuss how our methods compare to the literature. 

\subsection{Full KRR}

Many prior works have developed methods to solve full KRR that avoid the $\bigO(n^3)$ cost of direct methods such as Cholesky decomposition.
The various approaches can be roughly divided into 3 categories: (i) PCG methods, (ii) stochastic gradient methods, and (iii) SAP methods.

\label{subsec:krr_related}
  
For PCG, much work has been done on developing efficient preconditioners \citep{cutajar2016preconditioning,avron2017faster,frangella2023randomized,diaz2023robust}. 
Among the numerous proposals, the most popular preconditioners are based on the \nys{} approximation \citep{nystrom1930praktische}. 
Nystr{\"o}m approximations can be constructed from the kernel matrix via uniform sampling of columns \citep{williams2000using}, greedy pivoting \citep{harbrecht2012low}, leverage-score sampling \citep{musco2017recursive}, random projection \citep{tropp2017fixed}, and randomized pivoting \citep{chen2025randomly,epperly2025embrace}.
Preconditioners have also been constructed using random-features approximations \citep{avron2017faster}, but these do not perform as well in practice as \nys{} preconditioners \citep{diaz2023robust,frangella2023randomized}.
Typically, if a \nys{} preconditioner is constructed with a rank $r=\bigO(\deff{\lambda}(K))$, then \pcg{} converges at a condition-number-free rate with high probability \citep{zhao2022nysadmm,frangella2023randomized}.
The downside of these \pcg{} approaches is that they exhibit an $\bigO(n^2)$ per-iteration cost and require $\bigO(nr)$ storage, making it difficult to scale these methods beyond $n \gtrsim 10^{5}$.

\epro{} and \epro{} 2.0 \citep{ma2017diving, ma2019kernel} are the most well-known stochastic gradient methods for solving full KRR.
Both \epro{} and \epro{} 2.0 set the regularization $\lambda = 0$ and solves the resulting KRR problem via preconditioned stochastic gradient descent. 
\epro{} 2.0 reduce the per-iteration cost to $\bigO(nb_g + rs)$, a significant improvement over the $\bigO(n^2)$ cost of \pcg{}.

The third approach to solve full KRR problems is based on \sap{}. 
\cite{tu2016large} proposes randomized block Gauss-Seidel for full KRR.
In follow-up work, \cite{tu2017breaking} develops \nsap{} for solving the full KRR problem.
\cite{gazagnadou2022ridgesketch} applies \sap{} for solving full KRR but uses count sketches instead of coordinate sampling to form $S \Klambd S^{T}$ in the \sap{} update.
Unfortunately, these methods have $\bigO(b^3)$ computational cost per iteration, which is too expensive for our largest example, taxi, where $b=5 \cdot 10^4$.
More recently, \cite{lin2024stochastic} proposes a variant of \nsap{} where $K_{\B\B} + \lambda I$ is replaced by the identity matrix.
While this replacement removes the $\bigO(b^3)$ iteration cost associated with matrix factorization, its convergence becomes much slower.
\preprintcontent{Indeed, we verify in \cref{subsec:ablation} that replacing $K_{\B \B} + \lambda I$ with the identity matrix degrades the performance of \sko{} and \asko{}. 
}

Among these (approximate) SAP methods, \sko{} and \asko{} are the only methods that guarantee condition-number-free convergence under reasonable assumptions on the kernel matrix, while also avoiding $\bigO(b^3)$ per-iteration costs. 

\subsection{Inducing Points KRR}
\label{subsec:ikrr_related}

PCG is also popular for solving inducing points KRR, with 
\fal{} \citep{rudi2017falkon,meanti2020kernel} and KRILL \citep{diaz2023robust} being the current state-of-the-art.
KRILL is similar to \fal{}, but solves \eqref{eq:ikrr_lin_sys} via PCG with a preconditioner constructed using a sparse sketch.
KRILL has robust theoretical guarantees, and
numerical tests in \cite{diaz2023robust} show it yields comparable or better performance than \fal{}.
Nevertheless, KRILL still has a $\bigOt(nm+m^3)$ runtime and requires $\bigO(m^2)$ storage.
Consequently, it encounters the same runtime and storage barriers as \fal{} for large $m$. 

\cite{meanti2020kernel} develops a software package for \fal{} that relies on extensive code and hardware optimizations.
These optimizations include stabilizing \fal{} in single precision, out-of-core matrix operations, parallelized memory transfers, and distribution over multiple GPUs.
Consequently, \cite{meanti2020kernel} scale \fal{} to $n = 10^9$ with $m = 10^5$.
Despite these optimizations, \fal{} still encounters fundamental memory barriers:
\cite{abedsoltan2023toward} shows that the optimized version of \fal{} cannot scale beyond $2.56 \times 10^5$ inducing points when given 340 GB of RAM, which is larger than the memory budget available to most practitioners.

\cite{abedsoltan2023toward} develops the stochastic gradient method \epro{} 3.0 for inducing points KRR.
\epro{} 3.0 has a lower iteration and storage complexity than \fal{}, which allows it to use more inducing points. 
Unfortunately, this method has no convergence guarantees.

\section{Key Quantities in the Analysis of \asko{}}
\label{sec:asko_key_quantities}
Establishing convergence of \asko{} involves quantities whose relevance is not obvious a priori.
Therefore, we show how the key quantities in the convergence analysis arise. 
For simplicity, our analysis focuses on \asko{} without acceleration (i.e., \sko{}).
Throughout the analysis, we omit the dependence of $\B$ on $i$ to avoid notational clutter.

\subsection{\sap{} and \nys{} Projectors}
We begin by introducing two matrices that will play a central role in the convergence analysis of \asko{}.

\begin{definition}[\sap{} projector]
\label[definition]{def:sap_proj}
For a block $\B$, kernel matrix $K$, and regularization $\lambda$, the \sap{} projector is
\[
\projB \coloneqq K_{\lambda}^{1/2}\Ib^{T}(\Kbb+\lambda I)^{-1}\Ib K_{\lambda}^{1/2}.
\]
\end{definition}

The \sap{} projector $\projB$ is an \textit{exact} projection matrix, which controls the convergence rate of \sap{} methods.

\begin{definition}[\nys{} projector]
\label[definition]{def:nys_proj}
For a block $\B$, kernel matrix $K$, regularization $\lambda$, damping $\rho > 0$, and stepsize $\eta_\B > 0$, the \nys{} projector is
\[
\aprojB \coloneqq \eta _\B K_{\lambda}^{1/2}\Ib^{T}(\Knys+\rho I)^{-1}\Ib K_{\lambda}^{1/2}.
\]
\end{definition}
The \nys{} projector controls the convergence rate of \asko{}.
Note that the name \nys{} projector is a slight abuse of terminology, as $\aprojB$ is not idempotent.
However, for appropriately chosen $\eta_\B$, we will show that $\aprojB$ is closely related to the \sap{} projector.

\subsection{One-Step Analysis of \sko{} and \sap{}}
\label{subsec:sko_sap_one_step_analysis}
We now show how the \sap{} and \nys{} projectors arise by performing a one-step analysis of \sap{} and \sko{}.

We begin by writing the updates for \sap{} and \sko{}:
\begin{align*}
   & \iterw{i}{} = \iterw{i-1}{}-\Ib^{T}\left(\Kbb+ \lambda I\right)^{-1}\Ib(\Klambd \iterw{i}{}-y), \tag{\sap{}} \\
   & \iterw{i}{} = \iterw{i-1}{}-\eta_\B \Ib^{T}\left(\Knys+\rho I\right)^{-1}\Ib(\Klambd \iterw{i}{}-y). \tag{\sko{}}
\end{align*}

Some elementary manipulations show that
\begin{align*}
    & \|w_{i}-w_\star\|_{\Klambd}^2 \leq (w_{i-1}-w_\star)^{T}K_{\lambda}^{1/2}\left(I-\projB\right)^2 K_{\lambda}^{1/2}(w_{i-1}-w_\star),
    \tag{\sap{}} \\
    & \|w_{i}-w_\star\|_{\Klambd}^2 \leq (w_{i-1}-w_\star)^{T}K_{\lambda}^{1/2}\left(I-\aprojB\right)^2K_{\lambda}^{1/2}(w_{i-1}-w_\star). \tag{\sko{}}
\end{align*}
As $\projB$ is a projection, we can further deduce for \sap{} that
\[
\|w_{i}-w_\star\|_{\Klambd}^2 \leq (w_{i-1}-w_\star)^{T}K_{\lambda}^{1/2}\left(I-\projB\right)K_{\lambda}^{1/2}(w_{i-1}-w_\star).
\tag{\sap{}}
\]
Define\footnote{If $\Klambd$ was only psd rather than pd, then this quantity must be replaced by $\lmin^{+}(\E[\projB])$, as $\E[\projB]$ would be singular. 
For our setting, this does not matter as $\Klambd$ is pd.}
\[\mu \coloneqq \lmin(\E\left[\projB\right]).
\]
Then, taking the expectation conditioned on $w_{i-1}$ yields
\[
\E[\|w_{i}-w_\star\|_{\Klambd}^{2}\mid w_{i-1}] \leq \left(1-\mu\right)\|w_{i-1}-w_\star\|_{\Klambd}^2.
\tag{\sap{}}
\]
Thus, the convergence of \sap{} is controlled by the smallest positive eigenvalue of the expected projection matrix $\E\left[\projB \right]$.
As \sko{} is an approximate version of \sap{}, we would like to deduce a similar result.
Indeed, if we can ensure
\[
\aprojB \preceq I,
\]
then
\[
\|w_{i}-w_\star\|_{\Klambd}^2 \leq (w_{i-1}-w_\star)^{T}K_{\lambda}^{1/2}\left(I-\aprojB\right)K_{\lambda}^{1/2}(w_{i-1}-w_\star).
\tag{\sko{}}
\]
Defining 
\[\hat \mu \coloneqq \lmin(\E[\aprojB]),\]
it immediately follows from the preceding display that
\[
\E[\|w_{i}-w_\star\|_{\Klambd}^{2}\mid w_{i-1}] \leq \left(1-\hat \mu \right)\|w_{i-1}-w_\star\|_{\Klambd}^2.
\tag{\sko{}}
\]

\subsubsection{Setting the Stepsize $\eta_B$}
The convergence of \sko{} is controlled by $\hat \mu$, the smallest positive eigenvalue of $\E[\aprojB]$, similar to how $\mu$ controls the convergence of \sap{}.
To conclude this relation, we require
\[
\aprojB \preceq I.
\]
This condition tells us that the stepsize $\eta_\B$ cannot be arbitrarily large.
However, we also want to characterize the behavior of $\hat \mu$: when $\Knys$ is a good approximation to $\Kbb$, $\hat \mu$ should be closely related to $\mu$.
The following lemma formalizes this intuition. 
 \begin{lemma}
 \label[lemma]{lemma:appx_proj}
    Let $\rho > 0$ and $\B \subseteq [n]$ be fixed.
    Let $\Knys$ be a \nys{} approximation of $\Kbb$.
    Define
    \begin{align*}
        & L_{P_\B} \coloneqq \lmax\left((\Kbb+\lambda I)^{1/2}\left(\Knys+\rho I\right)^{-1}(\Kbb+\lambda I)^{1/2}\right), \\
        & \hat L_{P_\B} \coloneqq \max\{1, L_{P_\B}\}, \\        
        & \sigma_{P_\B} \coloneqq \lmin\left((\Kbb+\lambda I)^{1/2}\left(\Knys+\rho I\right)^{-1}(\Kbb+\lambda I)^{1/2}\right).
    \end{align*}
    
    If we set $\eta_B = 1 / \hat L_{P_\B}$ in \cref{def:nys_proj}, then the \nys{} projector satisfies
    \begin{equation}
    \label{eq:loewn_bnd}
    \frac{\sigma_{P_\B}}{\hat L_{P_\B}} \projB \preceq \aprojB \preceq \projB.
    \end{equation}
\end{lemma}

Setting the \asko{} stepsize to $\eta_\B = 1/\hat L_{P_\B}$ ensures convergence and allows for a lower bound on $\lmin(\aprojB)$ in terms of $\lmin(\projB)$.
Recall that we set $\eta_\B = 1/L_{P_\B}$ as the stepsize for \asko{} in \cref{sec:algorithms}.
The conclusions of \cref{lemma:appx_proj} still hold for this simpler stepsize.
We only require $1/\hat L_{P_\B}$ for our theoretical analysis to show \asko{} achieves a better upper bound on the number of iterations to reach an $\epsilon$-approximate solution of \eqref{eq:krr_linsys} than \sko{}.

\section{First-Moment Analysis of the Expected \nys{} Projector}
\label{sec:nys_proj_analysis}

The argument in \cref{subsec:sko_sap_one_step_analysis} shows that the convergence of \asko{} is controlled by $\hat \mu$, which is the smallest eigenvalue of the first moment of the \nys{} projector, $\E[\aprojB]$.
Analyzing $\hat \mu$ and $\E[\aprojB]$ directly seems impenetrable. 
Instead, we would like to derive a lower bound on $\E[\aprojB]$ in terms of the first moment of the \sap{} projector, $\E[\projB]$:
\begin{equation}
\label{eq:ideal_bnd}
    \E[\aprojB] \succeq g(\rho, \lambda)\E[\projB],
\end{equation}
where $g(\rho, \lambda)$ is a scalar function that captures the price of replacing $\Kbb$ with $\Knys$.
Once \eqref{eq:ideal_bnd} has been established, we immediately obtain
\begin{equation}
\label{eq:ideal_mu_bnd}
\hat \mu \geq g(\rho,\lambda) \mu. 
\end{equation}

Provided that we have a non-trivial lower bound on $\mu$, we can combine \eqref{eq:ideal_mu_bnd} with the one-step analysis in \cref{subsec:sko_sap_one_step_analysis} to obtain an explicit convergence rate for \sko{}.
We can obtain an explicit convergence rate for \asko{} using a similar approach, with some additional work to account for Nesterov acceleration.

There are two challenges in establishing the lower bound on $\hat \mu$ in \eqref{eq:ideal_mu_bnd}:
\begin{enumerate}
    \item \textbf{Deriving a meaningful lower bound on $\mu$ is non-trivial}.
    Existing works that establish such a lower bound rely on either expensive DPP sampling techniques \citep{mutny2020convergence} or preprocessing the kernel matrix using a Hadamard transform \citep{derezinski2024solving, derezinski2024fine}, which incurs a $\bigO(n^2)$ cost and requires materializing a $n \times n$ matrix in memory.
    Therefore, we must develop a different approach for establishing a lower bound on $\mu$.
    \item \textbf{Identifying the form of $g(\rho, \lambda)$ requires care}.
    The analysis to determine $g(\rho, \lambda)$ requires careful spectral approximation arguments based on the theory of the \nys{} approximation.
\end{enumerate}

\begin{figure}[t] 
  \centering
  \begin{tikzpicture}[
    >=Latex,
    font=\small,
    box/.style={
      draw, rectangle, rounded corners,
      align=center, inner xsep=6pt, inner ysep=4pt
    },
    elab/.style={
      font=\small, align=center,
      fill=white, fill opacity=0.95, text opacity=1,
      inner sep=1.5pt
    }
  ]
    \newcommand{\LBW}{30mm}

    \node[box] (arls) {ARLS sampling};
    \node[box, below=15mm of arls] (rls) {RLS sampling};

    \node[box, right=46mm of $(arls)!0.5!(rls)$] (dpp) {DPPs};

    \node[box, right=34mm of dpp, text width=\LBW] (mu)
      {Lower bound on $\mu$ (\cref{lem:convergence_rls})};
    \node[box, below=16mm of mu, text width=\LBW] (mu_hat)
      {Lower bound on $\hat \mu$ (\cref{thm:appx_proj_small_eig})};

    \draw[->, dashed] (arls) -- (rls)
      node[pos=.55, left=2mm, elab] {Approximation\\(\cref{def:rls-approx})};

    \draw[->] (arls) -- (dpp)
      node[pos=.5, above, elab, yshift=2mm] {Reduction\\(\cref{lem:rls_sample})};

    \draw[->, dashed] (rls) -- (dpp)
      node[pos=.5, below, elab, xshift=2mm, yshift=-4mm] {Exact correspondence\\(\cref{lem:dpp_rls})};

    \draw[->] (dpp) -- (mu)
      node[midway, above, elab] {Expected projection\\(\cref{lem:dpp_proj})};

    \draw[->] (mu) -- (mu_hat)
      node[midway, left=1mm, elab] {Spectral approximation\\(\cref{prop:appx_proj})};
  \end{tikzpicture}

  \caption{Summary of reductions and correspondences between ARLS/RLS sampling, DPPs, and the resulting lower bounds on $\mu$ and $\hat \mu$. 
  The dashed arrows are for intuition only: ARLS is closely related to RLS and RLS is related to DPPs, so we expect ARLS to be related to DPPs. 
  The solid arrows correspond to our approach for obtaining a lower bound on $\hat \mu$.}
  \label{fig:arls-rls-dpp-chain}
\end{figure}

In the following subsections, we address these challenges.
In \cref{subsec:arls_dpp_background}, we provide background on (approximate) ridge leverage scores and DPPs.
In \cref{subsec:appx_rls}, we develop a novel \arls{}-to-DPP reduction (\cref{lem:rls_sample}) to establish a non-trivial lower bound on $\mu$ (\cref{lem:convergence_rls}).
This reduction and its associated results can be applied to any psd matrix, making them of independent interest outside of this manuscript.
In \cref{subsec:appx_mu}, we relate $\aprojB$ to $\projB$ using spectral approximation arguments based on the \nys{} approximation.
Combining these ideas, we show the following lower bound on $\hat \mu$:

\[
\hat \mu \geq \frac{\lambda}{16\tailcond{k}{n}\rho}\frac{k}{n},
\]
where $\tailcond{k}{n} \coloneqq \frac{1}{n-k} \sum_{i>k} \lambda_i(K_\lambda)/\lambda_{\min}(K_\lambda)$.
The exact statement is presented in \cref{thm:appx_proj_small_eig}.

For clarity, we illustrate our argument to lower bound $\hat \mu$ in \cref{fig:arls-rls-dpp-chain}.

\subsection{Approximate Ridge Leverage Score Sampling and DPPs}
\label{subsec:arls_dpp_background}
In this subsection, we formally introduce approximate ridge leverage scores and \arls{} sampling and discuss their efficient computation.
As our argument relies on a reduction from ARLS sampling to DPP sampling, we also define determinantal point processes and introduce some fundamental facts about them.

\subsubsection{Approximate Ridge Leverage Score Sampling}
We begin by introducing ridge leverage scores.

\begin{definition}[Ridge leverage scores, effective dimension, and degrees of freedom]
\label[definition]{def:rls}
Given $A \in \psd^n$ and $\lambda >0$, its $i$-th $\lambda$-ridge leverage score $\ell_i^{\lambda}(A)$ is the $i$-th diagonal entry of matrix $A(A+\lambda I)^{-1}$:
\[
    \ell_i^{\lambda}(A) = e_i^{T} A(A + \lambda I)^{-1} e_i.
\]
The sum of all its $\lambda$-ridge leverage scores is called its $\lambda$-effective dimension:
\[
    \deff{\lambda}(A)\coloneqq \Tr\big(A(A+\lambda I)^{-1}\big) = \sum_{i=1}^n\ell_i^{\lambda}(A).
\]
The $\lambda$-maximal degrees of freedom of $A$ is given by:
\begin{align*}
    \dmof{\lambda}(A) \coloneqq n \max_{i\in [n]} \ell_i^{\lambda}(A).
\end{align*}
\end{definition}
The effective dimension measures the degrees of freedom of $A$, taking into account the regularization $\lambda$.
For matrices with decaying eigenvalues (e.g., kernel matrices), $\deff{\lambda}(A)$ is much smaller than the ambient dimension $n$ \citep{alaoui2015fast}.
The $i$-th \rls{} captures how much a row $i$ of $A$ contributes to the effective dimension.
The maximal degrees of freedom is always larger than the effective dimension, and can be significantly larger when the leverage score distribution is highly non-uniform.

The \rls{} scores naturally define a distribution on the rows of $A$: sample row $i$ with probability $p_i = \ell_i^{\lambda}(A)/\deff{\lambda}(A)$.
This sampling distribution admits strong theoretical guarantees \citep{alaoui2015fast, rudi2018fast}, but is impractical as computing the leverage scores requires an eigendecomposition at a cost of $\bigO(n^3)$, which is more expensive than the problem we wish to solve.
Fortunately, good guarantees can still be obtained by using approximate ridge leverage score sampling (\arls{} sampling), which we now define. 

\begin{definition}[Approximate RLSs and {$\arls[\lambda][c]$-sampling}]
\label[definition]{def:rls-approx}
Given $A \in \psd^n$ and $\lambda >0$, we say that $\{\tilde{\ell}_i\}_{i=1}^n$ are a $c$-approximation of $\lambda$-ridge leverage scores $\{\ell_i^{\lambda}(A)\}_{i=1}^n$ for some $c\geq 1$, if they satisfy
\begin{align*}
\tilde{\ell}_i\geq \ell_i^{\lambda}(A)\ \ \forall i\qquad\text{and}\qquad \tilde \ell \coloneqq \sum_{i=1}^n\tilde\ell_i\leq c\cdot \deff{\lambda}(A).
\end{align*}
Define $p_i := \frac{\tilde{\ell}}{n}\left\lceil\frac{n}{\tilde{\ell}} \tilde{\ell}_i\right\rceil$ for all $i \in [n]$, and sample indices $\{i_j\}_{j=1}^{b}$ i.i.d.~according to $\{p_i\}_{i=1}^n$. Then $\B= \bigcup_{j=1}^{b} \{i_j\}$ is an $\arls[\lambda][c]$-sampling.\footnote{The indices in $\B$ are actually distinct since our analysis allows us to discard duplicates.} 
\end{definition}


Approximate \rls{}s were first introduced by \cite{alaoui2015fast} for constructing randomized \nys{} approximations of kernel matrices. 
The key advantage of \arls{}s over \rls{}s is that they can be computed efficiently. 
In particular, \asko{} uses the BLESS algorithm proposed by \cite{rudi2018fast} to compute the \arls{}s.
BLESS admits the following complexity guarantee:
\begin{lemma}[\cite{rudi2018fast}, Theorem 1]
\label[lemma]{lem:rls_approx}
Given $A \in \psd^n$, a positive integer $k$, $c\geq 1$ and $\lambda>0$ such that the $\lambda$-effective dimension of $A$ satisfies $\deff{\lambda}(A) \leq k$, there is an algorithm (BLESS) that with high probability returns $c$-approximations for all $n$ $\lambda$-ridge leverage scores of $A$ in  $\tilde\bigO(nk^2)$ time.
\end{lemma}

\cref{lem:rls_approx} shows we can obtain accurate $\lambda$-\arls{}s to construct the sampling distribution in $\tilde O(nk^2)$ time.
This stands in contrast to SAP solvers in \cite{derezinski2024solving, derezinski2024fine}, which require $\bigO(n^2)$ time to preprocess $K$ by a Hadamard transform $H$, and $\bigO(n^2)$ storage to store $HKH^{T}$, from which rows are subsampled.
Thus, \arls{}s yield a significant improvement in computational and storage complexities over Hadamard preprocessing.



\subsubsection{Determinantal Point Processes and their Connection to \rls{}s}
The heart of our analysis is a novel \arls{}-to-DPP reduction argument that
exploits a fundamental connection between RLS and determinantal point processes (DPPs).
DPPs are a family of non-i.i.d.~probability measures used in machine learning to sample diverse subsets of data points \citep{kulesza2012determinantal}. 

\begin{definition}[Determinantal point process]
\label[definition]{def:dpp}
Given $A \in \psd^n$, a (random-size) determinantal point process $\dpp(A)$ is a distribution over all $2^n$ index subsets $\B \subseteq [n]$ such that $\prfn{\B} = \frac{\det(A_{\B\B})}{\det(A+I)}$. Moreover, a (fixed-size) $k$-$\dpp(A)$ is a distribution over subsets $\B \subseteq {[n] \choose k}$ such that $\prfn{\B} \propto \det(A_{\B\B})$.
\end{definition}
DPP-based row selection matrices yield strong convergence guarantees for SAP methods thanks to the following formula connecting random projectors with DPP sampling.
\begin{lemma}[\cite{derezinski2020improved}, Lemma 5]
\label[lemma]{lem:dpp_proj}
    Given $A\in\psd^n$, let $\B\sim\dpp(A)$, and define the random projection matrix $\proj_\B:= A^{1/2}S^T \pinv{(SAS^T)} SA^{1/2}$, where $S$ is a row selection matrix for set $\B$. Then
    \begin{equation}
    \label{eq:dpp_expect_formula}
            \E[\proj_\B] = A(A+ I)^{-1}.
    \end{equation}
\end{lemma}
\cite{mutny2020convergence} gives a sharp convergence analysis for DPP-based SAP methods using  \eqref{eq:dpp_expect_formula}.
Unfortunately, directly sampling from a DPP is impractical due to its prohibitive computational cost \citep{derezinski2019exact}.

The goal of our \arls-to-DPP reduction is to show that when using \arls{} sampling, a relation similar to \eqref{eq:dpp_expect_formula} holds:
\[
\E[\proj_\B] \succeq A(A+ I)^{-1} - \delta \pinv{A}A.
\]
That is, we want to obtain a perturbed version of the DPP expectation formula \eqref{eq:dpp_expect_formula} for a small $\delta$.
Our arguments show that this is indeed the case.
The key idea is to exploit the following connection between the \rls{}s of $A$ and $\dpp(A)$: 

\begin{lemma}[\cite{derezinski2021determinantal}, Theorem 10]
\label[lemma]{lem:dpp_rls}
For $\B \sim \dpp(A)$ and index $i \in [n]$, the marginal probability of $i \in \B$ is the $i$-th ridge leverage score of $A$, i.e.,
\begin{align*}
\prfn{i \in \B} = \left[A(A+I)^{-1}\right]_{i,i} = \ell_i^1(A).
\end{align*}
\end{lemma}




\subsection{Analysis of $\mu$ with ARLS Sampling}
\label{subsec:appx_rls}

The main goal of this subsection is establish the following result, which provides tight control on $\E[\projB]$ and $\mu$ when using \arls{} sampling.

\begin{lemma}[Projection analysis for ARLS sampling]
\label[lemma]{lem:convergence_rls}
Given $A\in\psd^n$ with rank $r$, $k = \Omega(\log n)$, let $\bar\lambda,\tilde\lambda>0$ be such that $\deff{\tilde\lambda}(A)\geq 2\deff{\bar\lambda}(A)= 4k$.
If $\B$ is $\arls[\tilde\lambda][c]$-sampled with blocksize $b=\Omega(c k\log^3 n)$,
then the projection $\projB \coloneqq A^{1/2}I_{\B}^T \pinv{(I_{\B} A I_{\B}^T)} I_{\B} A^{1/2}$ satisfies 
\begin{align}
\label{eq:rls_proj_lwr_bnd}
\E[\projB] ~\succeq~ \frac12 A \left(A + \bar{\lambda} I\right)^{-1}.
\end{align}
Furthermore, this implies the following lower bound for the \sap{} convergence parameter $\mu$:
\begin{align*}
\mu = \lambda_{\min}^+(\E[\projB]) \geq \frac{k}{2r\bar\kappa_{k:r}},\quad\text{where}\quad \tailcond{k}{r} \coloneqq \frac{1}{r-k} \sum_{i>k} \lambda_i(A)/\lambda_{\min}^+(A).
\end{align*}
\end{lemma}

\cref{lem:convergence_rls} gives a result similar to \cref{lem:dpp_proj}, only instead of  exact equality as in \eqref{eq:dpp_expect_formula}, we  obtain a lower bound. 
Nevertheless, this is sufficient for our needs, and allows us to obtain a lower bound on $\mu$ comparable to what is obtained in \cite{derezinski2024solving, derezinski2024fine} using Hadamard preprocessing. 
Thus, using ARLS sampling to select rows of $A$ yields the same control over $\E[\projB]$, but at much lower computational and storage costs compared to \sap{} methods that use Hadamard preprocessing.

\cref{lem:convergence_rls} immediately yields a projection analysis for uniform sampling, since uniform sampling corresponds to ARLS sampling with a specific choice of scores.


 \begin{corollary}[Projection analysis for uniform sampling]
 \label[corollary]{corollary:proj_analysis_uniform}
     Suppose $A$, $\bar \lambda$, $\tilde\lambda$, and $k$ are set as in \cref{lem:convergence_rls}.
     If $\B$ is sampled according to the uniform distribution with \blkszstr{} $b = \Omega(d_{\max}^{\tilde\lambda}(A)\log^3n)$,
     then $\projB$ satisfies
     \begin{align}
        \E[\projB] ~\succeq~ \frac12 A \left(A + \bar{\lambda} I\right)^{-1}.
    \end{align}
 \end{corollary}
 
 Under uniform sampling, the \blkszstr{} has to be proportional to the maximal degrees of freedom: $b=\bigO(\dmof{\tilde{\lambda}}(A)\log^3n)$.
 When the ridge leverage scores are relatively uniform, i.e., $\max_{i\in [n]}\ell_i^{\tilde {\lambda}}(A) \approx \frac{\deff{\tilde{\lambda}}(A)}{n}$, then $d^{\tilde{\lambda}}_{\max}(A) = \Theta(\deff{\tilde{\lambda}}(A))$, so uniform sampling performs very well without requiring a large \blkszstr{}.
 As uniform sampling works well empirically across a broad range of learning tasks (\cref{sec:experiments}), we believe this setting is most relevant for practice.  

\subsubsection{ARLS-to-DPP Reduction}
To prove \cref{lem:convergence_rls}, we rely on connections between RLSs and DPPs.
Recall from \cref{lem:dpp_proj} that for any $A\in\psd^n$ and $\bar\lambda>0$, the sample $\bdpp \sim \dpp(A / \bar\lambda)$ satisfies $\E[\proj_{\bdpp}] = A(A+\bar\lambda I)^{-1}$. 
Moreover, from \cref{lem:dpp_rls}, the marginal probability of sampling element $i$ into $\B$ is the $i$-th $\bar\lambda$-ridge leverage score of $A$, i.e., $\Pr(i \in \bdpp) = \ell_i^{\bar\lambda}(A)$. 
We would like to exploit this connection to transfer the projection analysis from DPPs to \arls{}s.
Indeed, the following lemma shows that a sufficiently large ARLS sample contains a DPP sample, which is crucial for the proof of \cref{lem:convergence_rls}:

\begin{lemma}[ARLS-to-DPP reduction]
\label[lemma]{lem:rls_sample}
Let $k, \tilde{\lambda}$, and $\B$ be defined as in \cref{lem:convergence_rls}. Given $\delta = n^{-\bigO(1)}$, for any $j \in \left[ 2k - \sqrt{6k \log \left( \frac{2}{\delta} \right)}, 2k + \sqrt{6k \log \left( \frac{2}{\delta} \right)} \right]$, we can couple a DPP sample $\B_{\hdpp{j}} \sim \hdpp{j}(A / \tilde \lambda)$ with $\B$ such that with probability at least $1-\delta$, $\B_{\hdpp{j}} \subseteq \B$.
\end{lemma}

The proof of \cref{lem:rls_sample} is given in \cref{appendix:rls_sample}.
Our analysis is inspired by \cite{derezinski2024solving}, who use a Markov Chain Monte Carlo (MCMC) argument \citep{anari2024optimal} to show that after conditioning $\bdpp$ on a fixed sample size $j$, the i.i.d.~sampling distribution over the marginals $p_{i \mid j}=\Pr(i\in \bdpp \mid |\bdpp|=j)$ is an effective proxy for $\hdpp{j}(A / \bar \lambda)$.
We modify the approach in \cite{derezinski2024solving} by showing that the approximate ridge leverage scores $\tilde \ell_i^{\bar\lambda}(A)$ are close to the marginals $p_{i \mid j}$ as long as the size $j$ is close to the expected size of $\bdpp$, $\E[|\bdpp|]$.
Fortunately, this occurs with high probability (\cref{lem:dpp-size}).
We then use the law of total expectation to decompose the distribution of $\dpp(A / \bar\lambda)$ into a mixture of fixed-size DPP distributions $\hdpp{j}(A)$, and apply an MCMC argument to each fixed-size DPP.
Combining all of these ideas yields the result.

\subsubsection{Proof of \cref{lem:convergence_rls}}
To complete the proof of \cref{lem:convergence_rls}, we leverage the \arls{}-to-DPP reduction from \cref{lem:rls_sample} and apply the DPP projection formula from \cref{lem:dpp_proj}.

\begin{proof}
Define the interval $\I = \left[ 2k - \sqrt{6k \log \left( \frac{2}{\delta} \right)}, 2k + \sqrt{6k \log \left( \frac{2}{\delta} \right)} \right]$ where $k=\deff{\bar\lambda}(A)/2$, and denote event $\Ec_j = \{\B_{j\text{-}\dpp} \subseteq \B\}$. 
According to \cref{lem:rls_sample}, for any $j \in \I$ we have $\prfn{\Ec_j} \geq 1-\delta$. By using the fact that $\proj_{\B_1} \preceq \proj_{\B_2}$ whenever $\B_1 \subseteq \B_2$ \cite[Lemma 6.3]{derezinski2024solving},  we have
\begin{align}
\label{eq:couple_j_dpp}
\E[\projB] & \succeq \E[\projB\mid \Ec_j] \prfn{\Ec_j} \nonumber \\
& \succeq \E[\proj_{\B_{j\text{-}\dpp}}\mid \Ec_j] \prfn{\Ec_j} \nonumber \\
& = \E[\proj_{\B_{j\text{-}\dpp}}] - \E[\proj_{\B_{j\text{-}\dpp}}\mid \Ec_j^c] \prfn{\Ec_j^c} \nonumber \\
& \succeq \E[\proj_{\B_{j\text{-}\dpp}}] - \delta \cdot \pinv{A} A,
\end{align}
where the last step follows since $\proj_{\B_{j\text{-}\dpp}}$ is an orthogonal projector onto a subspace of $\textrm{range}(A^T)$, thus $\E[\proj_{\B_{j\text{-}\dpp}}\mid \Ec_j^c] \preceq \Pi_{[n]} = \pinv{A} A$.
Let $\B_{\dpp} \sim \dpp(A / \bar\lambda)$ and recall that $\E[|\B_{\dpp}|] = \deff{\bar\lambda}(A) = 2k$. 
Also, define $w_{\bar\lambda,j} \coloneqq \prfn{|\B_{\dpp}| = j}$. 
Using the law of total expectation, we can express the expected projection matrix for a DPP in terms of fixed-size $\hdpp{j}$s:

\begin{align*}
    \E[\proj_{\bdpp}] & = \sum_{j \in \I} w_{\bar{\lambda},j}\E[\proj_{\bdpp} \mid |\bdpp|=j]+ \sum_{j \notin \I} w_{\bar{\lambda},j}\E[\proj_{\bdpp} \mid |\B_{\dpp}|=j] \\
    &= \sum_{j \in \I} w_{\bar{\lambda},j}\E[\proj_{\B_{\hdpp{j}}}]+ \sum_{j \notin \I} w_{\bar{\lambda},j}\E[\proj_{\B_{\hdpp{j}}}],
\end{align*}
where the second equality uses that $\E[\proj_{\bdpp} \mid |\B_{\dpp}|=j] = \E[\proj_{\B_{\hdpp{j}}}]$ by definition.
Recalling \eqref{eq:couple_j_dpp} and using $\proj_{\B_{\hdpp{j}}} \preceq \Pi_{[n]} = \pinv{A} A$, we can bound the preceding display as
\begin{align*}
    \E[\proj_{\bdpp}] &\preceq \left(\sum_{j \in \I} w_{\bar{\lambda},j}\right)\left(\E[\projB]+\delta\cdot  \pinv{A} A\right)+\left(\sum_{j\notin \mathcal I}w_{\bar{\lambda},j}\right) \pinv{A} A \\ &\preceq \E[\projB]+\delta \cdot \pinv{A} A + \left(\sum_{j\notin \mathcal I}w_{\bar{\lambda},j}\right) \pinv{A} A,
\end{align*}
where the last relation follows as $\sum_{j \in \I} w_{\bar{\lambda},j} \leq 1$.
Now, as $\prfn{j\notin \I} = \sum_{j\notin \I}w_{\bar{\lambda},j}$, \cref{lem:dpp-size} yields $\sum_{j \notin \I}w_{\bar{\lambda},j}\leq \delta$.
Combining this with the preceding display, we deduce
\begin{align*}
    \E[\proj_{\bdpp}] \preceq \E[\projB]+2\delta \pinv{A} A.
\end{align*}
Rearranging the preceding display yields
\begin{align*}
   \E[\projB] \succeq \E[\proj_{\bdpp}]-2\delta \pinv{A} A. 
\end{align*}
Applying \cref{lem:dpp_proj} and setting $\delta\leq \min\{ \lambda_{\min}^+(A) / (4 (\lambda_{\min}^+(A)+\bar{\lambda})), n^{-c} \}$ for some constant $c>0$, so that $\delta=n^{-\bigO(1)}$, we conclude
\begin{align*}
\label{eq:rls_lem_proj_bnd}
   \E[\projB] \succeq \frac{1}{2}A\left(A+\bar\lambda I\right)^{-1}, \tag{$*$} 
\end{align*}
which is precisely \eqref{eq:rls_proj_lwr_bnd}.

We now establish the lower bound on $\mu$. 
To this end, we first show that $\bar \lambda \leq \frac{1}{k}\sum_{i>k}\lambda_i(A)$.
Indeed, if we suppose $\bar{\lambda} > \frac{1}{k} \sum_{i>k} \lambda_i(A)$, then
\begin{align*}
2k \leq d^{\bar{\lambda}}(A) = \sum_{i=1}^{k} \frac{\lambda_i}{\lambda_i + \bar{\lambda}} + \sum_{i > k} \frac{\lambda_i}{\lambda_i + \bar{\lambda}}
< k + \frac{\sum_{i>k} \lambda_i}{\bar{\lambda}} < k+k = 2k,
\end{align*}
resulting in a contradiction. Hence $\bar\lambda \leq \frac{1}{k} \sum_{i>k} \lambda_i(A)$.

Finally, it follows from \eqref{eq:rls_lem_proj_bnd} that
\begin{align*}
    \mu & = \lmin^{+}\left(\E[\projB]\right) \geq \frac{1}{2}\lmin^{+}\left(A\left(A+\bar\lambda I\right)^{-1}\right) = \frac{1}{2}\frac{\lmin^{+}(A)}{\lmin^{+}(A)+ \bar \lambda} \geq \frac{1}{2\left(1+\frac{1}{k}\sum_{i>k}\frac{\lambda_i(A)}{\lmin^{+}(A)}\right)} \\
    &= \frac{1}{2\left(1+\frac{r-k}{k}\frac{1}{r-k}\sum_{i>k}\frac{\lambda_i(A)}{\lmin^{+}(A)}\right)} \geq \frac{k}{2\left(\frac{1}{r-k}\sum_{i>k}\frac{\lambda_i(A)}{\lmin^{+}(A)}\right)r} = \frac{k}{2\tailcond{k}{r}(A)r}.
\end{align*}
\end{proof}

\subsection{Lower Bound on $\hat \mu$}
\label{subsec:appx_mu}
Having established a lower bound on $\mu$, we now turn to lower bounding $\hat \mu$, which will allow us to establish a fine-grained convergence rate for \asko{}.

The main goal of this subsection is to establish the following theorem:
\begin{theorem}[Lower bound on $\hat \mu$]
\label[theorem]{thm:appx_proj_small_eig}
   Suppose the blocksize satisfies $b = \Omega(k\log^{3} n)$ for some $k\in [n]$. 
    Then using ARLS sampling as in \cref{lem:convergence_rls} with $A=K_\lambda$ and $c=\bigO(1)$, the \nys{} approximation as in \cref{prop:appx_proj},
    and defining $\bar \lambda$ as in \cref{lem:convergence_rls},
    \[
    \E[\aprojB] \succeq \frac{\lambda}{8\rho}K_\lambda(K_\lambda+\bar \lambda I)^{-1}.
    \]
    Consequently, recalling $\tailcond{k}{n} = \frac{1}{n-k} \sum_{i>k} \lambda_i(K_\lambda)/\lambda_{\min}(K_\lambda)$, we have
    \[
    \hat \mu \geq \frac{\lambda}{16\tailcond{k}{n}\rho}\frac{k}{n}.
    \]
\end{theorem}
\cref{thm:appx_proj_small_eig} provides us with the desired lower bound on $\hat \mu$.
When $b$ is  sufficiently large, the following corollary shows that $\hat \mu$ is lower bounded by a quantity independent of the conditioning of $\Klambd$.

\begin{corollary}
\label[corollary]{corollary:cnd_free_hat_mu}
    Fix a failure probability $\delta = n^{-\bigO(1)}$.
    Suppose the \blkszstr{} satisfies $b= \Omega\left(k\log^{3} n \right)$ where $k\geq 2\deff{\lambda}(K)$. 
    If the rank of the Nystr{\"o}m approximation satisfies $r =\bigO\left(\deff{\rho}(\lfloor K \rfloor_b)\log\left(\frac{\deff{\rho}(\lra{K}{b})}{\delta}\right)\right)$ and $\rho \geq \lambda$, then
    \[
    \hat \mu \geq \frac{\lambda}{32\rho} \frac{k}{n}.
    \] 
\end{corollary}

When the effective \blkszstr{} $k = \Omega(\deff{\lambda}(K))$, \cref{corollary:cnd_free_hat_mu} shows the lower bound on $\hat \mu$ no longer depends on the conditioning of $K_\lambda$.
Instead, $\hat \mu$ is controlled by $k$ and $\rho$.
Using a smaller $\rho$ (larger $r$) leads to a larger $\hat \mu$. 
Conversely, larger $\rho$ corresponds to using a smaller rank, and the lower bound on $\hat \mu$ decreases.
We shall see shortly that \cref{corollary:cnd_free_hat_mu} implies \sko{} and \asko{} enjoy condition-number-free convergence rates.

\subsubsection{The \nys{} Projector is Close to the \sap{} Projector}
\cref{lem:convergence_rls} guarantees a tight lower bound on $\mu$ when $\B$ is an $\arls[\tilde\lambda][c]$-sampling for any $c =\bigO(1)$ with respect to $K_\lambda$. 
Thus, if we can determine the form of $g(\rho, \lambda)$, we can obtain a lower bound on $\hat \mu$.

We begin with \cref{prop:appx_proj}, which explicitly characterizes the \textit{shrinkage factor} $\sigma_{P_\B} / \hat L_{P_\B}$ in \cref{lemma:appx_proj} with high probability.
The proof of \cref{prop:appx_proj} appears in \cref{subsec:appx_proj_pf}.

\begin{proposition}[Controlling the shrinkage factor]
\label[proposition]{prop:appx_proj}
    Let $\B\subseteq [n]$ satisfy $|\B| = b<n$, and
    set $\rho \geq \lambda$.
    If the \nys{} approximation $\Knys$ of $\Kbb$ is constructed from a sparse sign embedding $\Omega$ with $r=\bigO\left(\deff{\rho}(\lfloor K \rfloor_b)\log\left(\frac{\deff{\rho}(\lra{K}{b})}{\delta} \right)\right)$ columns and $\zeta = \bigO\left(\log\left(\frac{\deff{\rho}(\lra{K}{b})}{\delta} \right)\right)$ non-zeros per column, then
    \[
    \frac{\lambda}{2\rho}\projB \preceq \aprojB \preceq \projB, \quad \text{with probability at least}~1-\delta.
    \]
    \end{proposition}

 The shrinkage factor obtained in \cref{prop:appx_proj} depends upon $\lambda/\rho$. 
 If $\rho = \bigO(\lambda)$, then the shrinkage factor is \emph{constant} with high probability, in which case little is lost by replacing $\Kbb$ with the \nys{} approximation $\Knys$.
Selecting a larger $\rho$ allows smaller Nystr\"om rank $r$, but incurs a trade-off in the shrinkage factor. 
We note that $\deff{\rho}(\lra{K}{b})$ is always upper-bounded by 
$\deff{\rho}(K)$, and in some cases, may be substantially smaller.

\begin{remark}
While the guarantee in \cref{prop:appx_proj} is for when $\Knys$ is constructed with a sparse sign embedding, the same guarantee also holds when $\Omega$ is a Gaussian embedding (which we use in our experiments) or a randomized Hadamard transform.       
\end{remark}

\subsubsection{Proof of \cref{thm:appx_proj_small_eig}}

Here we prove \cref{thm:appx_proj_small_eig}, which provides a lower bound on $\hat \mu$.

\begin{proof} 
Define the event
    \[
    \Ec = \left\{\frac{\lambda}{2\rho} \projB \preceq \aprojB \preceq \projB \right\}.
    \]
    Then \cref{prop:appx_proj} tells us that by setting $\delta$ appropriately, we can guarantee $\prfn{\Ec} \geq 1-\frac{\mu}{2}.$
    Thus, the law of total expectation yields
    \begin{align*}
        \label{eq:aproj_lwr_bnd}
        \E[\aprojB] & = \E[\aprojB \mid \Ec]\prfn{\Ec}+ \E[\aprojB \mid, \Ec^{C}]\prfn{\Ec^{C}} \\
        &\succeq \E[\aprojB \mid \Ec]\prfn{\Ec} \\
        &\succeq \frac{\lambda}{2\rho}\E\left[\projB \mid \Ec \right]\prfn{\Ec} \tag{$*$}.
    \end{align*}
    We now lower bound $\E\left[\projB \mid \Ec \right]$ in the Loewner ordering. 
    To this end, the law of total expectation and a simple rearrangement imply
    \[
    \E\left[\projB \mid \Ec \right] = \frac{1}{\prfn{\Ec}}\left(\E[\projB]-\E[\projB \mid \Ec^{C}]\prfn{\Ec^C}\right).
    \]
    The preceding display, along with the facts that $\projB \preceq I$ and $\prfn{\Ec^{C}}\leq \mu/2$,  yields
    \[
      \E\left[\projB \mid \Ec \right] \succeq  \frac{1}{\prfn{\Ec}}\left(\E[\projB]-\prfn{\Ec^C}I\right) \succeq \frac{1}{\prfn{\Ec}}\left(\E[\projB] - \frac{\mu}{2}I\right).
    \]
    As $\frac{\mu}{2}I\preceq \frac{1}{2}\E[\projB]$, we obtain
    \[
     \E\left[\projB \mid \Ec \right] \succeq \frac{1}{2\Pr(\Ec)}\E[\projB].
    \]
    Thus, combining this last display with \eqref{eq:aproj_lwr_bnd} and applying \cref{lem:convergence_rls} with $A = K_\lambda$, we conclude
    \[
    \E[\aprojB] \succeq \frac{\lambda}{4\rho}\E[\projB] \succeq \frac{\lambda}{8\rho}K_\lambda\left(K_\lambda+\bar\lambda I\right)^{-1}.
    \]
    The claimed lower bound on $\hat \mu$ follows from this last display and \cref{lem:convergence_rls}.
\end{proof}

\section{Convergence of \asko{}}
\label{sec:asko_cvg}
We have established all the preliminary results  for the convergence analysis of \asko{}.
We begin with the following general convergence result (\cref{thm:mast_conv_thm}).
In \cref{subsubsec:log_lin_runtime}, we discuss when \sko{} and \asko{} achieve near-optimal log-linear runtimes and how they can use large \blkszstr{}s to leverage the benefits offered by modern computing hardware.
\begin{theorem}[Convergence of \asko{}]
\label{thm:mast_conv_thm}
    Consider \sko{} and \asko{} (\cref{alg:askotch}) under the following hyperparameter settings: the \blkszstr{} satisfies $b = \Omega(k\log^3 n)$, where $k\geq 2\deff{\lambda}(K)$, $\rho\geq \lambda$, $\Pc$ is a $\arls[\tilde\lambda][\bigO(1)]$-sampling distribution for $K_\lambda$ with $\tilde\lambda>0$ such that $\deff{\tilde\lambda}(K_\lambda)\geq 4k$, $\Knys$ is constructed from a sparse sign embedding with $r=\bigO\left(\deff{\rho}(\lfloor K \rfloor_b)\log\left(\frac{\deff{\rho}(\lra{K}{b})}{\delta}\right)\right)$ columns and $\zeta = \bigO\left(\log\left(\frac{\deff{\rho}(\lra{K}{b})}{\delta}\right)\right)$ non-zeros per column, and $\eta_\B = 1/\hat L_{P_\B}$.
    Then the following statements hold:
    \begin{enumerate}
        \item After $t$ iterations, the output of \sko{} satisfies
        \[
        \E[\|w_t-\wstar\|_{\Klambd}^2]\leq \left(1-\frac{\lambda}{32\rho}\frac{k}{n}\right)^{t}\|w-w_0\|_{\Klambd}^2.
        \]
        Thus, the total number of iterations required to achieve an $\epsilon$-approximate solution is bounded by $\bigO\left(\frac{\rho}{\lambda}\frac{n}{k}\log\left(\frac{1}{\epsilon}\right)\right).$
        \item After $t$ iterations, the output of \asko{} satisfies
        \[
        \E[\|w_t-\wstar\|_{K_\lambda}^2] \leq 2\left(1-\frac{1}{16\sqrt{2}}\min\left\{\sqrt{\frac{\lambda}{\rho}}\frac{k}{n}, \sqrt{\frac{k}{n}}\frac{\lambda}{\rho}\right\}\right)^{t}\|w_0-\wstar\|_{\Klambd}^2.
        \]
        Thus, the total number of iterations required to achieve an $\epsilon$-approximate solution is bounded by $\bigO\left(\max\left\{\sqrt{\frac{\rho}{\lambda}}\frac{n}{k},\frac{\rho}{\lambda}\sqrt{\frac{n}{k}}\right\}\log\left(\frac{1}{\epsilon}\right)\right).$
    \end{enumerate}
\end{theorem}
The proof of \cref{thm:mast_conv_thm} is provided in \cref{subsubsec:mast_conv_pf}.

\cref{thm:mast_conv_thm} shows \sko{} and \asko{} converge linearly to the solution of \eqref{eq:krr_linsys}.
\sko{} requires $\bigO(\frac{\rho}{\lambda} \frac{n}{k} \log\left(\frac{1}{\epsilon}\right))$ iterations to produce an $\epsilon$-approximate solution, while \asko{} requires $\bigO\left(\max\left\{\sqrt{\frac{\rho}{\lambda}}\frac{n}{k},\frac{\rho}{\lambda}\sqrt{\frac{n}{k}}\right\}\log\left(\frac{1}{\epsilon}\right)\right)$ iterations.
As $\rho/\lambda, n/k\geq 1$, the upper bound for \asko{} is always an improvement over the upper bound for \sko{}, demonstrating the benefit of acceleration.
\preprintcontent{The predicted improvement offered by acceleration holds in practice---our experiments in \cref{subsec:ablation} show that in the vast majority of cases, \asko{} performs better than \sko{}.}

\subsection{\asko{}: Two Convergence Regimes} 
\cref{thm:mast_conv_thm} shows that \asko{}'s convergence exhibits two different regimes: (i) $n/k > \rho/\lambda$, (ii) $\rho/\lambda>n/k$.
In regime (i), the ``low-rank approximation condition number'' $\rho/\lambda$ is dominated by the ``dimensional condition number'' $n/k$, and \asko{} converges in $\bigOt\left(\sqrt{\frac{\rho}{\lambda}}\frac{n}{k}\right)$ iterations.
In this regime, \asko{} can use a larger $\rho$ (and hence, a smaller Nystr\"om rank $r$) than \sko{}, as the convergence rate depends only upon $\sqrt{\frac{\rho}{\lambda}}$.
However, if $\rho$ is set too large, so that $\rho > (n/k)\lambda$, \asko{} enters regime (ii), where the low-rank approximation condition number $\rho/\lambda$ dominates the dimensional condition number $n/k$, leading to an iteration complexity of $\bigOt\left(\frac{\rho}{\lambda}\sqrt{\frac{n}{k}}\right)$.

Which convergence regime of \asko{} is preferable? 
In the worst case, a $n / k$ dependence in the iteration complexity is necessary, since \asko{} must make a full pass through $\Klambd$ to solve \eqref{eq:krr_linsys}. 
Thus, regime (i) is preferable as it allows \asko{} to reduce the price $\rho/\lambda$ of using a low-rank approximation. 
\asko{} is guaranteed to be in regime (i) provided it does not over-regularize, since the phase transition to regime (ii) occurs when $\rho>(n/k)\lambda$.

When the effective \blkszstr{} $k$ is small, i.e., $k = o(n)$, $n/k$ is large, so 
\asko{} can still use large $\rho$ and small rank $r$ while avoiding over-regularization. 
When the effective \blkszstr{} is large, i.e., $k = \Omega(n)$, $n/k = \bigO(1)$, so \asko{} must set $\rho = \bigO(\lambda)$ to stay in regime (i).
Consequently, if $K$ exhibits heavy-tailed spectral decay, the rank $r$ might have to increase significantly for \asko{} to remain in regime (i).
Fortunately, most kernel matrices exhibit fast spectral decay---indeed, under standard regularity  assumptions on the kernel function, $\deff{\lambda}(K) = \bigO(\sqrt{n})$ \citep{rudi2017falkon, rudi2018fast}.
Thus, even if $k = \Omega(n)$, \asko{} remains in regime (i) when it uses rank $r=\bigO(\sqrt{n})$.
Overall, convergence regime (i) for \asko{} is both preferable and also more likely to occur, compared to regime (ii).

\subsubsection{When Does \asko{} Converge in Log-Linear Time?}
\label{subsubsec:log_lin_runtime}
A natural question to ask is if there is a reasonable setting where $\sko{}$ and $\asko{}$ enjoy a near-optimal runtime of $\bigOt(n^2)$.
A priori, the answer is not obvious, as we use a low-rank approximation to $\Kbb$.
Consequently, the price of using this approximation may be too steep to ensure an $\bigOt(n^2)$ runtime.
Fortunately, the following corollary shows that for matrices whose effective dimension is not too large, the low-rank approximation has a minimal effect on the overall runtime. 
As the effective dimension of kernel matrices is typically $\bigO(\sqrt{n})$ \citep{rudi2017falkon,rudi2018fast}, this corollary shows that \sko{} and \asko{} can achieve a near-optimal runtime  for most KRR tasks, despite approximating $\Kbb$. 
\begin{corollary}[Log-linear convergence of \asko{}]
\label[corollary]{corr:log_lin_conv}
    Let $\Pc$ in \asko{} be the uniform distribution.
    Let $k$ and $\tilde \lambda$ be as in \cref{lem:convergence_rls} and suppose that $\max_{i\in [n]} \ell_i^{\tilde{\lambda}}(K_\lambda) = \Theta(\deff{\tilde{\lambda}}(K_\lambda)/n)$, $\rho = c_\rho \lambda$ for $c_\rho \in [1,n/k]$, and $\deff{\lambda}(K) = \bigO(\sqrt{n})$.
    Then under the assumptions of \cref{thm:mast_conv_thm}, with \blkszstr{} $b=\Theta(k\log^3 n)$, 
    the total runtime of \asko{} to produce an $\epsilon$-approximate solution is bounded by
    \[
    \bigOt\left(\sqrt{c_\rho}n^2\log\left(\frac{1}{\epsilon}\right)\right).
    \]
      
\end{corollary}
The proof of \cref{corr:log_lin_conv} appears in \cref{subsec:log_lin_conv}.
\cref{corr:log_lin_conv} shows that when (i) the effective dimension grows like $\bigO(\sqrt{n})$, (ii) the ridge leverage scores are relatively uniform, 
and (iii) \asko{} uses uniform sampling with $\rho = \bigO(\lambda)$,
\asko{} obtains a near-optimal runtime.

Another powerful consequence of \cref{corr:log_lin_conv} is that the total runtime is independent of the \blkszstr{}.
\asko{} can use effective \blkszstr{} $k = \bigO(n)$ (which corresponds to $b=\bigOt(n)$) and still have runtime $\bigOt\left(n^2\right)$.
Thus, when the leverage scores are relatively uniform, \asko{} can use large \blkszstr{}s and obtain a log-linear runtime. 
From a practical perspective, this is invaluable, as large \blkszstr{}s can exploit the massive parallelism offered by modern computing architectures.
Indeed, in our experiments, we use $b = n/100$ and observe excellent performance. 
The ability of \asko{} to use a large \blkszstr{} separates it from \sap{}, which incurs an $\bigO(n^3)$ cost per iteration when $b = \bigO(n)$.

\begin{remark}
    The same runtime bounds hold for ARLS sampling, which allows us to avoid the assumption that the ridge leverage scores satisfy  $\max_{i\in [n]}\ell_i^{\tilde\lambda}(K_\lambda) = \Theta\left(\frac{\deff{\tilde{\lambda}}(K_\lambda)}{n}\right)$.
    Instead, we require that $k = \Theta(\deff{\lambda}(K))$, since this ensures that the cost of running BLESS is bounded by $\bigOt(n^2)$. 
    Thus, ARLS cannot use blocksizes as large as uniform sampling while maintaining near-optimal runtime.
    We focus on uniform sampling, as this is what works best in practice. 
\end{remark}

\subsection{Proof of \cref{thm:mast_conv_thm}}
\label{subsubsec:mast_conv_pf}
Here, we provide the convergence proof of \sko{}, which corresponds to the first claim in \cref{thm:mast_conv_thm}.
The convergence proof for \asko{} is more involved and requires additional background, so it is deferred to  \cref{subsec:ASkotch_Conv}. 

\begin{proof}
    We have seen in \cref{subsec:sko_sap_one_step_analysis} that at iteration $i$,
    \[
    \E\left[\|w_i-w_\star\|_{\Klambd}^2\mid w_{i-1}\right] \leq (1-\hat \mu)\|w_{i-1}-w_\star\|_{\Klambd}^2.
    \]
    Invoking the lower bound on $\hat\mu$ in \cref{corollary:cnd_free_hat_mu}, this becomes
    \[
    \E\left[\|w_i-w_\star\|_{\Klambd}^2\mid w_{i-1}\right] \leq \left(1-\frac{\lambda k}{32\rho n}\right)\|w_{i-1}-w_\star\|_{\Klambd}^2.
    \]
    Applying the law of total expectation in this last display, we deduce
    \[
    \E\left[\|w_t-w_\star\|_{\Klambd}^2\right] \leq \left(1-\frac{\lambda k}{32\rho n}\right)^{t}\|w_{0}-w_\star\|_{\Klambd}^2.
    \]
    The claimed iteration complexity bound follows immediately. 
\end{proof}
\section{Experiments}
\label{sec:experiments}
We perform an empirical evaluation of \asko{}, \epro{} 2.0, \epro{} 3.0, \pcg{}, and \fal{} on KRR problems 
drawn from diverse application domains, 
including computer vision, particle physics, ecological modeling, online advertising, computational chemistry, music, socioeconomics, and transportation. 
We use three different kernels in our experiments: Laplacian, \mtrn{}-5/2, and radial basis function (RBF), which are among the most popular in the literature \citep{rasmussen2005gaussian,gardner2018gpytorch}.
Across this broad spectrum of application domains and kernels, \asko{} obtains better predictive performance than the competing methods. 
These results show that (i) full KRR can effectively scale to large problems and obtain better predictive performance than inducing points KRR and (ii) \asko{} is a new state-of-the-art solver for large-scale full KRR. 

We present the following results: 
\begin{itemize}
    \item Performance comparisons (\cref{subsec:performance_comparisons}):
    We compare \asko{} to the state-of-art methods for full KRR and inducing points KRR: \epro{} 2.0 \citep{ma2019kernel} and \pcg{} for full KRR, and \fal{} \citep{rudi2017falkon,meanti2020kernel} and \epro{} 3.0 \citep{abedsoltan2023toward} for inducing points KRR.
    For \pcg{}, we consider two of the most effective types of preconditioners: Gaussian \nys{} (\nys{}) \citep{frangella2023randomized} and randomly pivoted Cholesky (RPC) \citep{diaz2023robust,epperly2025embrace}.
    \asko{} consistently obtains better predictive performance than the competitors.
    \item Showcase on huge-scale transportation data analysis (\cref{subsec:showcase}): We run \asko{} on KRR for a subsample of the New York City taxi dataset $(n = 10^8)$.
    To the best of our knowledge, no previous work has scaled full KRR to a dataset of this size.
    \asko{} once again outperforms the competition.
    \item Verification of linear convergence (\cref{subsec:lin_cvg}): We show that \asko{} achieves global linear convergence on three KRR problems with $n \geq 5 \cdot 10^5$ samples, which verifies the linear convergence guarantee in \cref{sec:asko_cvg}.
    \asko{} reaches machine precision in at most 100 passes through each KRR linear system, demonstrating its potential as a high-precision solver.
    \preprintcontent{\item Ablation study (\cref{subsec:ablation}): We study how the approximate projector, acceleration, and sampling scheme affect the performance of \asko{}.
    We find that using the \nys{} approximation as the approximate projector and using acceleration both improve the performance of \asko{}.
    On the other hand, the choice of sampling scheme (uniform vs. ARLS) has little to no impact.}%
\end{itemize}

Each experiment is run on a single 48 GB NVIDIA RTX A6000 GPU with PyTorch 2.5.1 \citep{paszke2019pytorch}, CUDA 12.5, PyKeOps 2.2.3 \citep{charlier2021kernel}, and Python 3.10.12.
The code for our experiments is available at \href{\codeurl}{\codeurldisplaytext}.

\nonpreprintcontent{
Our \href{https://arxiv.org/abs/2407.10070v3}{arXiv report (https://arxiv.org/abs/2407.10070v3)} includes expanded experimental results and an ablation study that are omitted from this version of the paper.
We also provide additional information about datasets, kernel hyperparameters, regularizations, inducing points, preprocessing, and time limits in Appendix C of our arXiv report.}%

\textit{Optimizer hyperparameters.} 
Throughout the experiments, we use the default hyperparameters that we recommend for \asko{} (\cref{subsec:default_hyperparams}), unless stated otherwise.
\preprintcontent{We set the \blkszstr{} $b = n/100$ and rank $r = 100$.
We set $\rho$ in two different ways in \asko{}: \emph{damped} (the default) and \emph{regularization}.
Damped sets $\rho = \lambda + \lambda_r(\Knys)$, while
regularization sets $\rho = \lambda$.
We also apply both damping strategies to \pcg{} when using the Gaussian \nys{} preconditioner to ensure a fair comparison.}%
\nonpreprintcontent{We compute the damping for \pcg{} with the Gaussian \nys{} preconditioner using the same procedure as \asko{}, ensuring a fair comparison.}%
We also run \pcg{}, \epro{} 2.0, and \epro{} 3.0 with the same rank as \asko{} to ensure a fair comparison.
We set $b_g$, $s$, and the stepsize $\eta$ in \epro{} 2.0 and \epro{} 3.0 using the defaults in the \href{https://github.com/EigenPro/EigenPro-pytorch/tree/master}{EigenPro 2.0} and \href{https://github.com/EigenPro/EigenPro3/tree/main}{EigenPro 3.0 GitHub repositories}. 
We also run EigenPro 2.0 and 3.0 with regularization $\lambda = 0$, as recommended in \cite{ma2019kernel,abedsoltan2023toward}.

\textit{Numerical precision.} 
\asko{}, \epro{} 2.0, and \epro{} 3.0 are stable in single precision, while \fal{} and \pcg{} often require double precision for satisfactory performance.
Consequently, the figures in the main paper show results when \asko{}, \epro{} 2.0, and \epro{} 3.0 run in single precision and \fal{} and \pcg{} run in double precision.
\preprintcontent{We show in \cref{subsec:pcg_single_precision} that \asko{} still outperforms \fal{} and \pcg{} when all methods are run in single precision.}%
\nonpreprintcontent{We show in Appendix C of our \href{https://arxiv.org/abs/2407.10070v3}{arXiv report} that \asko{} still outperforms \fal{} and \pcg{} when all methods are run in single precision.}%

\textit{Time limits.} 
Each optimizer is run until it reaches a specified time limit, which depends on the size of the dataset.
These time limits range between 0.5 to 24 hours.
We implement all of the competing methods ourselves to ensure fair runtime comparisons with \asko{}.
The most expensive computations in all methods are performed using KeOps \citep{charlier2021kernel}, which ensures our runtime comparisons are fair.


\preprintcontent{Additional experimental details are in \cref{sec:experiments_appdx}.}%

\subsection{Performance Comparisons}
\label{subsec:performance_comparisons}

We evaluate predictive performance on 10 classification and 13 regression tasks from established benchmarks.
\nonpreprintcontent{The classification tasks are from computer vision, particle physics, ecological modeling, and online advertising, while the regression tasks are from computational chemistry, music analysis, and socioeconomics.
We assess predictive performance on classification and regression tasks by computing classification accuracy and mean-absolute error (MAE) between the predictions and targets, respectively, on the test set.}%

\preprintcontent{
\textit{Classification.} The classification tasks are from computer vision (\cref{fig:vision_float64}), particle physics (\cref{fig:particle_physics_float64}), ecological modeling (\cref{fig:tabular_classification_float64}), and online advertising (\cref{fig:tabular_classification_float64}).
We assess predictive performance on classification tasks by computing the classification accuracy between the predicted labels $\hat y$ and target labels $y$ on the test set:
\[
\mathrm{Accuracy}(\hat{y}, y) = \frac{1}{n_{\mathrm{tst}}} \sum_{i = 1}^{n_{\mathrm{tst}}} \mathds{1}(\hat{y}_i = y_i).
\]
\asko{} outperforms the competition on computer vision, ecological modeling, and online advertising, while delivering comparable performance on particle physics.

\textit{Regression.} The regression tasks are from computational chemistry (\cref{fig:qm9_float64,fig:molecules_float64}), music analysis (\cref{fig:tabular_regression_float64}), and socioeconomics (\cref{fig:tabular_regression_float64}).
We assess predictive performance on regression tasks by computing the mean-absolute error (MAE) between the predictions $\hat y$ and targets $y$ on the test set:
\[
    \mathrm{MAE}(\hat{y}, y) = \frac{1}{n_{\mathrm{tst}}} \sum_{i = 1}^{n_{\mathrm{tst}}} |\hat{y}_i - y_i|.
\]
\asko{} outperforms the competition on computational chemistry and music analysis, while delivering comparable performance on socioeconomics.
}%

\begin{figure}[t]
    \centering
    \includegraphics[width=\linewidth]{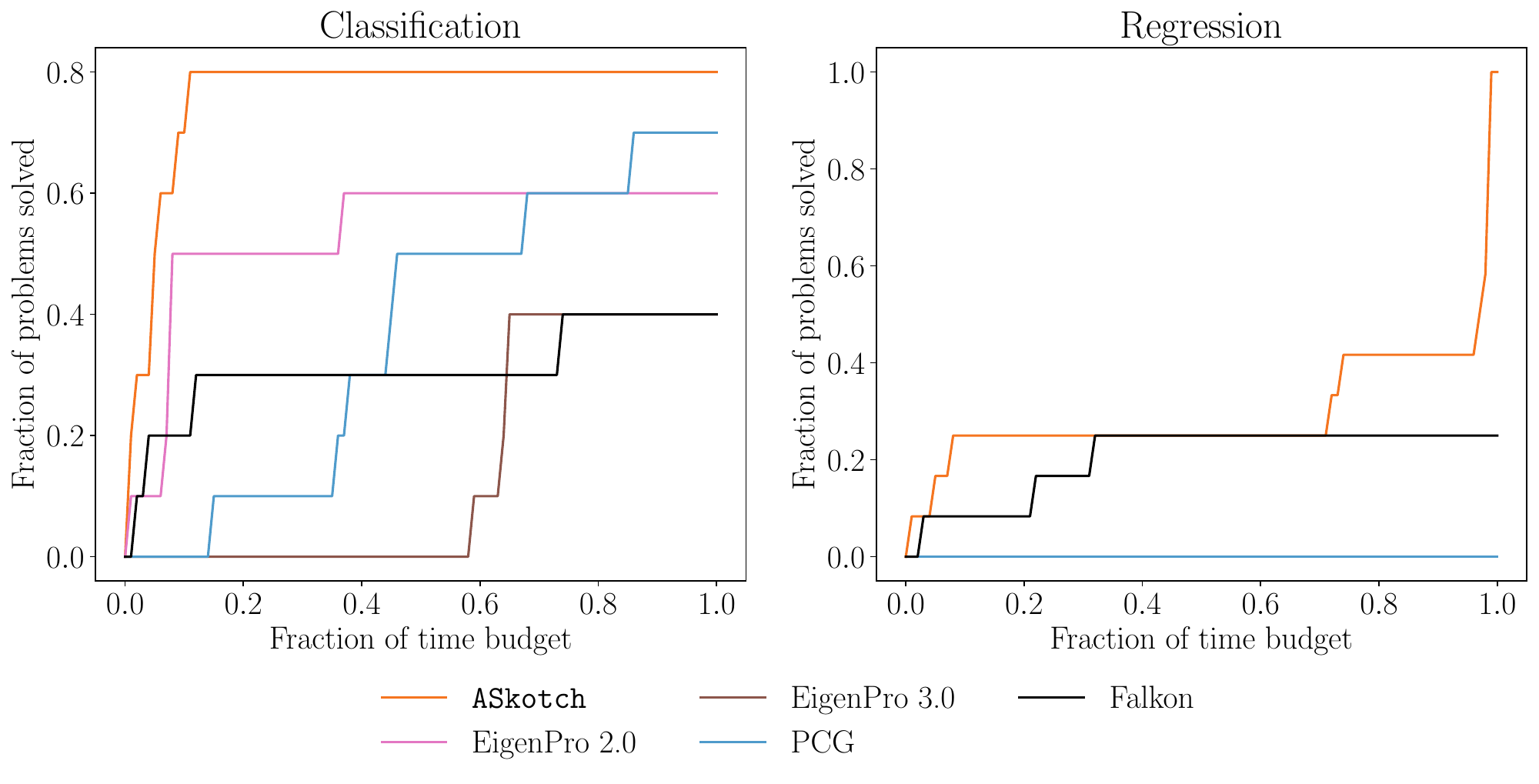}
    \caption{Performance comparison between \asko{} and competitors on 10 classification and 13 regression tasks.
    We designate a classification problem as ``solved'' when the method reaches within 0.001 of the highest classification accuracy found across all the optimizer + hyperparameter combinations.
    We designate a regression problem as ``solved'' when the method reaches within 1\% of the lowest MAE (in a relative sense) found across all the optimizer + hyperparameter combinations.
    \pcg{} and \fal{} are run in double precision.
    \epro{} 2.0, \epro{} 3.0, and \pcg{} do not solve any of the regression problems within the tolerance.
    \asko{} outperforms the competition on both classification and regression.
    }
    \label{fig:perf_float64}
\end{figure}

We summarize the results of our performance comparisons in \cref{fig:perf_float64}.
\asko{} outperforms the competition on both classification and regression.
Notably, \asko{} outperforms the state-of-the-art inducing point methods \epro{} 3.0 and \fal{}, demonstrating the value of using full KRR over inducing points KRR.
\preprintcontent{
The performance of \asko{} rapidly improves towards the end of the time budget on regression tasks.
This sudden improvement occurs because \asko{} does not saturate the MAE on several computational chemistry tasks within the time budget (\cref{fig:molecules_float64}).

A detailed description of each classification and regression task can be found in \cref{subsec:data_hyperparams}.
}%

\preprintcontent{
\begin{figure}[htbp]
    \centering
    \includegraphics[width=0.9\linewidth]{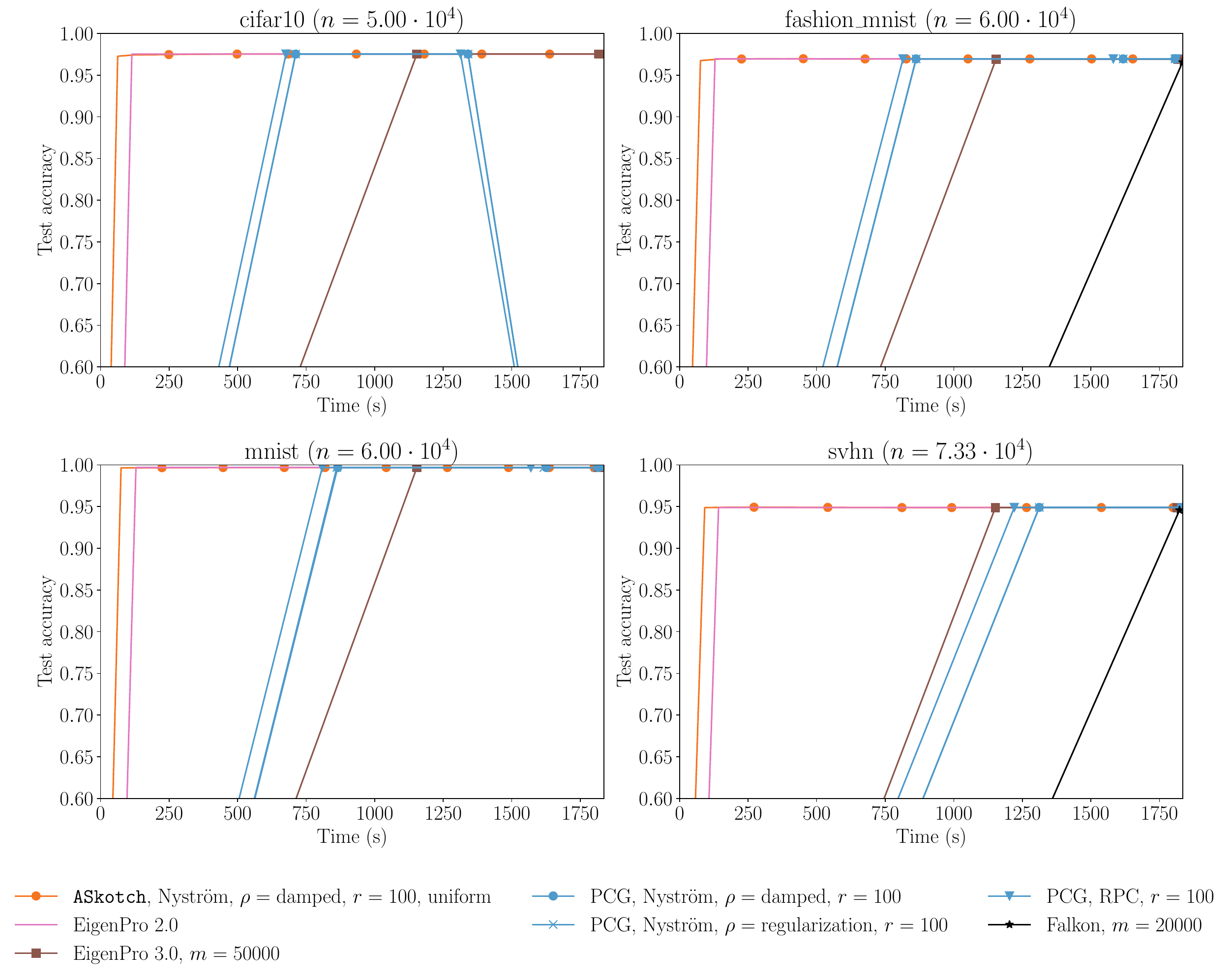}
    \caption{Comparison between \asko{} and competitors on computer vision tasks.
    \asko{} and the competing methods all reach similarly high classification accuracies, but \asko{} achieves this accuracy in less time than the competition.
    The classification accuracy for \pcg{} and \fal{} sometimes peaks and then goes towards 0---this is unsurprising since Krylov methods can diverge if they are run for too many iterations.}
    \label{fig:vision_float64}
\end{figure}



\begin{figure}[htbp]
    \centering
    \includegraphics[width=0.9\linewidth]{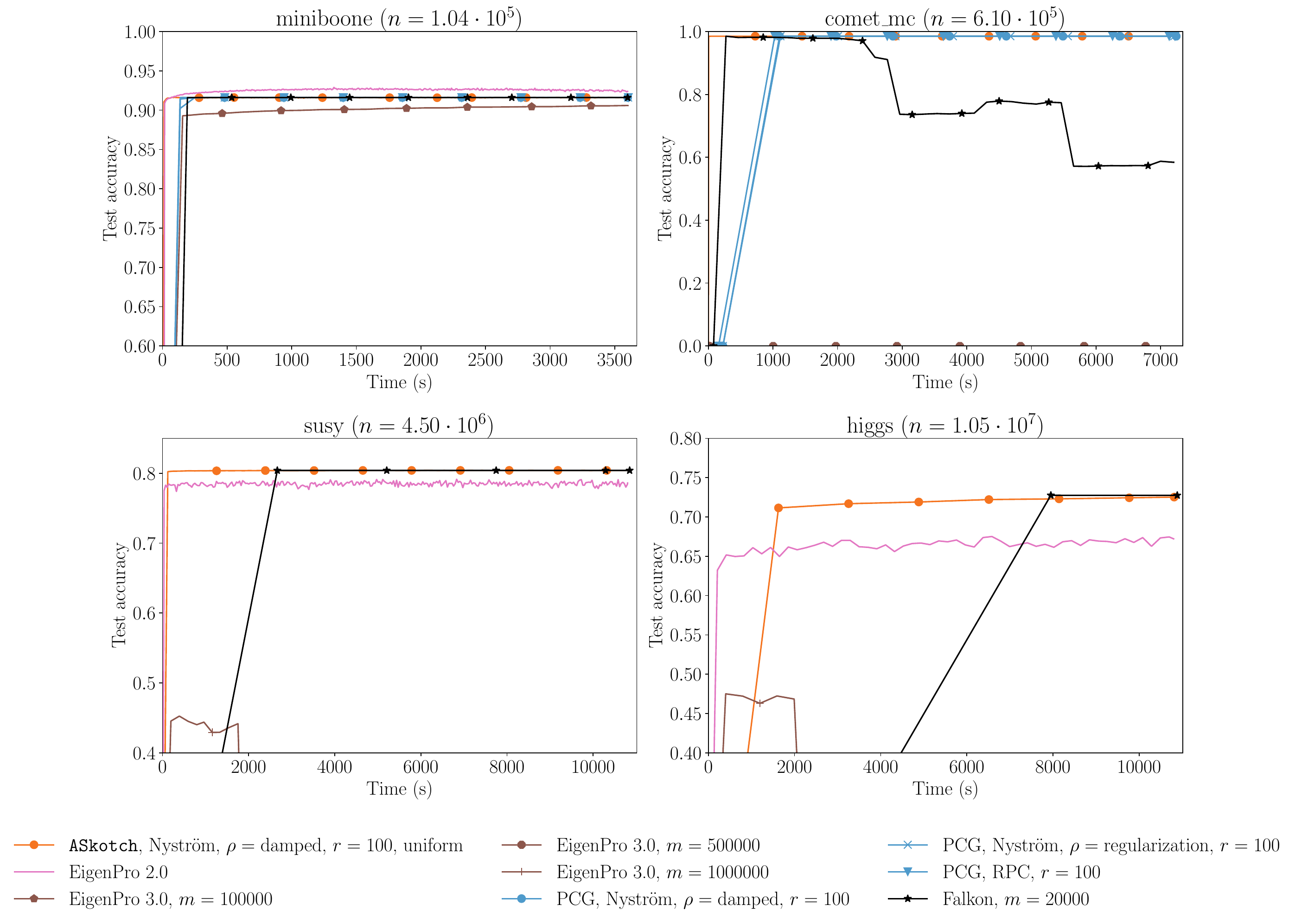}
    \caption{Comparison between \asko{} and competitors on particle physics tasks.
    \asko{} reaches a similar classification accuracy as the competition on both comet\_mc and susy, while taking much less time to reach this level of accuracy.
    However, \epro{} 2.0 and \fal{} outperform \asko{} on miniboone and higgs, respectively.
    On the other hand, \epro{} 2.0 and \epro{} 3.0 both diverge on comet\_mc, which shows that these methods do not always work well with their default hyperparameters, while \asko{} consistently provides good results with its defaults.
    Finally, \pcg{} does not even reach 0.4 classification accuracy on susy and it does not complete a single iteration on higgs.}
    \label{fig:particle_physics_float64}
\end{figure}



\begin{figure}[htbp]
    \centering
    \includegraphics[width=0.9\linewidth]{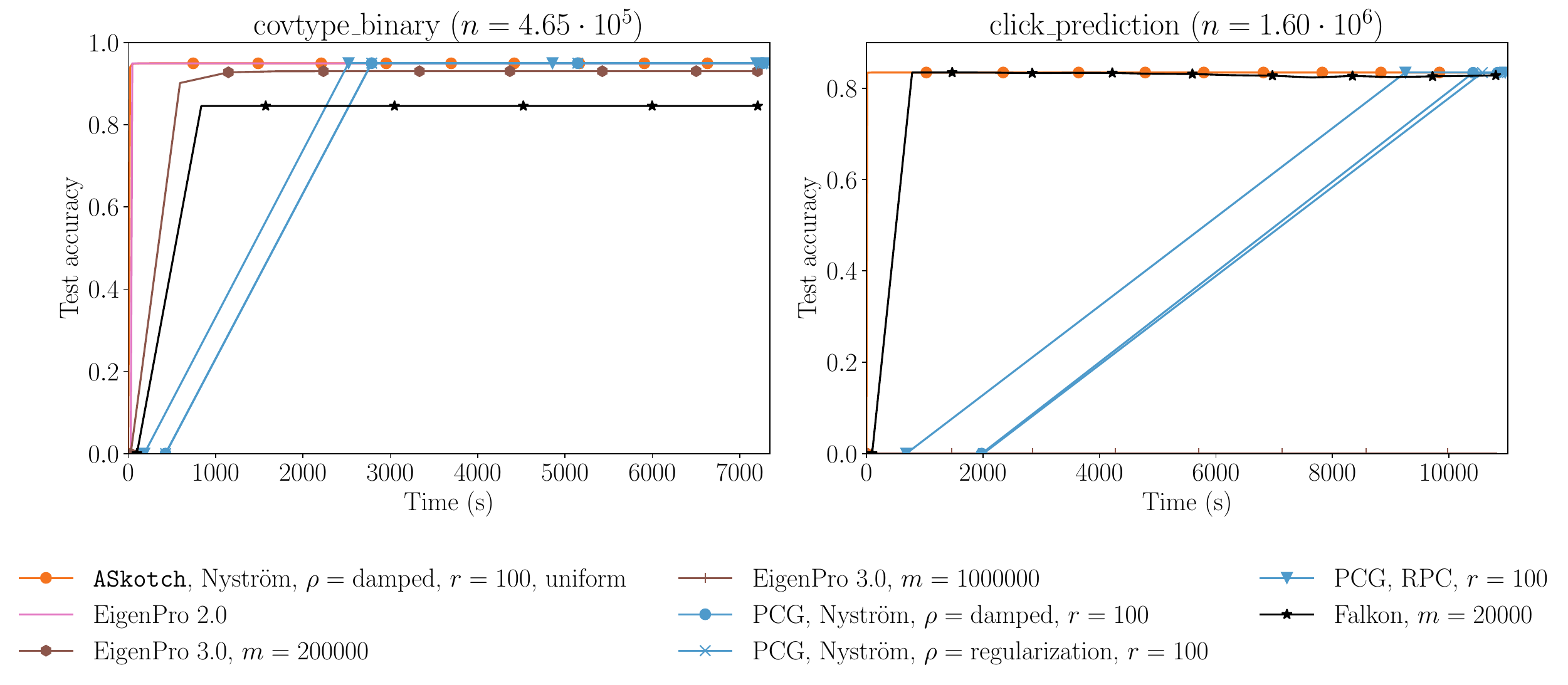}
    \caption{Comparison between \asko{} and competitors on ecological modeling and online advertising tasks.
    \asko{} achieves comparable or higher classification accuracy than \epro{} 3.0, \pcg{}, and \fal{} on both datasets, while requiring less time to do so.
    While \epro{} 2.0 is competitive with \asko{} on covtype\_binary, both \epro{} 2.0 and 3.0 diverge on click\_prediction.}
    \label{fig:tabular_classification_float64}
\end{figure}



\begin{figure}[htbp]
    \centering
    \includegraphics[width=0.9\linewidth]{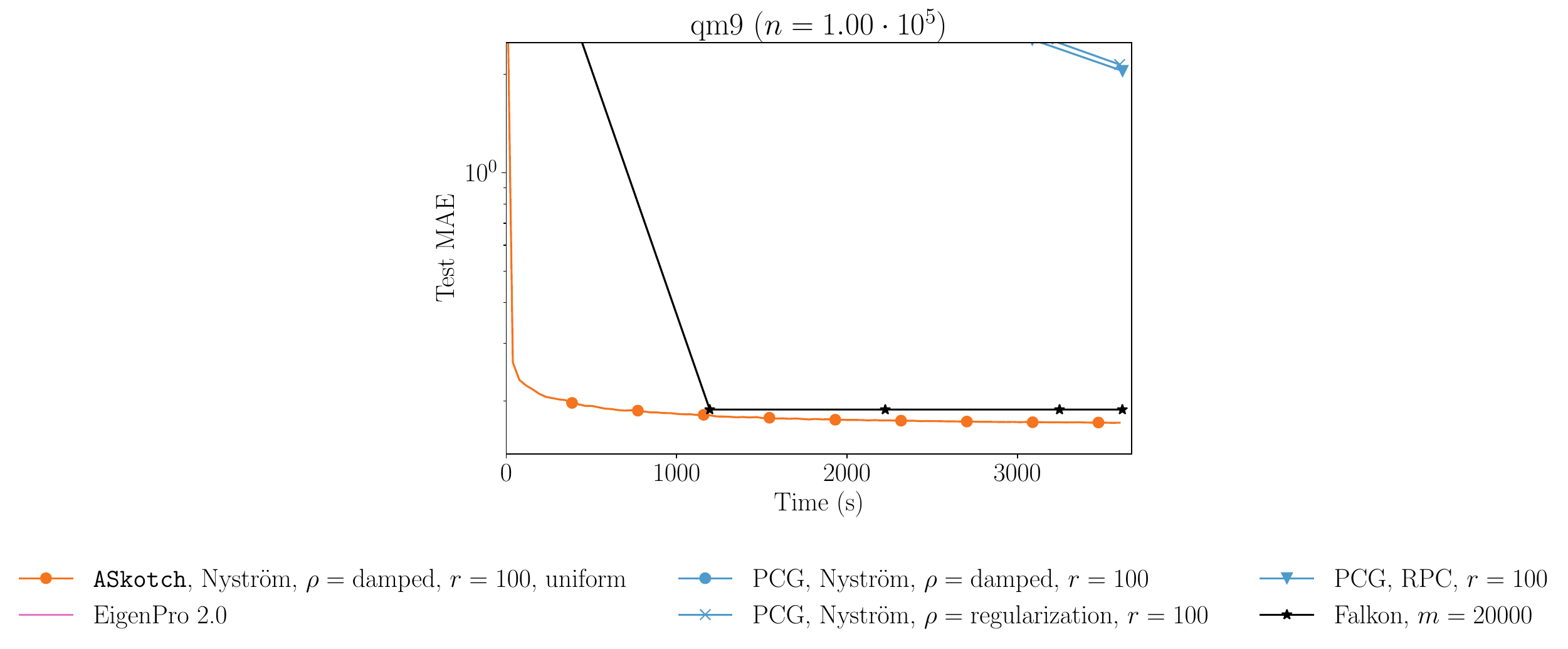}
    \caption{Comparison between \asko{} and competitors on computational chemistry with the qm9 dataset.
    \asko{} attains a lower MAE than the competition.
    Moreover, both \epro{} 2.0 and \epro{} 3.0 (not shown) diverge.}
    \label{fig:qm9_float64}
\end{figure}

\begin{figure}[htbp]
    \centering
    \includegraphics[width=0.9\linewidth]{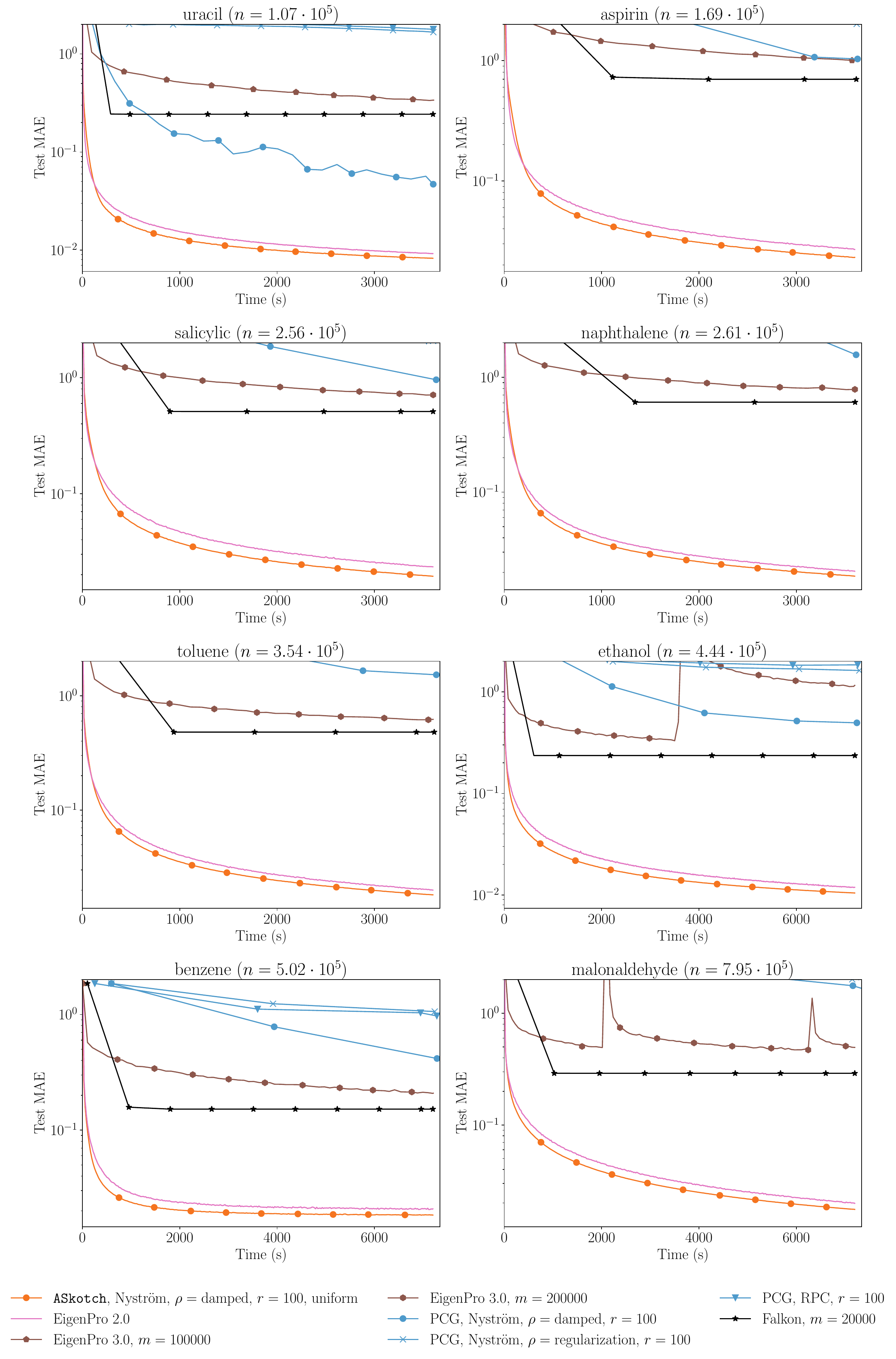}
    \caption{Comparison between \asko{} and competitors on computational chemistry with eight molecule datasets.
    \asko{} attains a lower MAE than the competition.}
    \label{fig:molecules_float64}
\end{figure}


\begin{figure}[htbp]
    \centering
    \includegraphics[width=\linewidth]{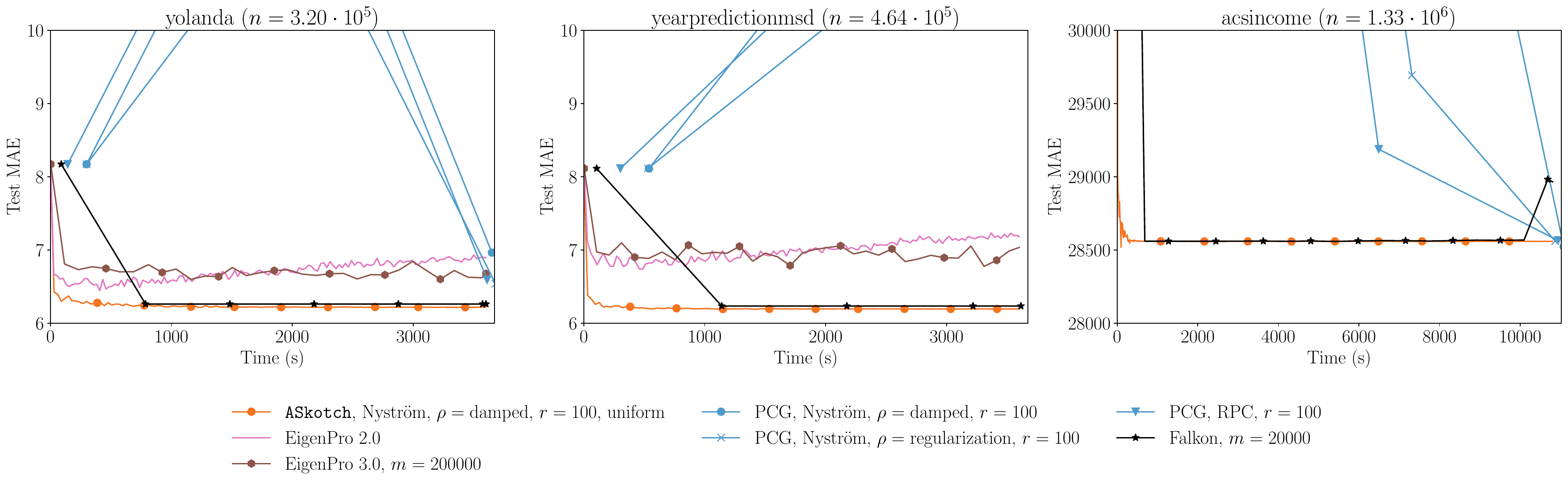}
    \caption{Comparison between \asko{} and competitors on music analysis and socioeconomics tasks.
    \asko{} outperforms all competing methods on yolanda and yearpredictionmsd, while performing similarly to \fal{} on acsincome.
    Both \epro{} 2.0 and 3.0 diverge on acsincome.}
    \label{fig:tabular_regression_float64}
\end{figure}
}%

\subsection{Showcase: Huge-Scale Transportation Data Analysis}
\label{subsec:showcase}
We apply KRR to a subsample of the \href{https://github.com/toddwschneider/nyc-taxi-data}{taxi dataset} to predict taxi ride durations in New York City.
Following \cite{meanti2020kernel}, we use an RBF kernel.
Due to hardware limitations, we set the \blkszstr{} $b$ in \asko{} at $n/2{,}000 = 5 \cdot 10^4$ and vary the rank $r \in \{50, 100, 200, 500\}$.
We set the rank for \pcg{} as low as $r = 50$, but none of the \pcg{} methods complete a single iteration in the time limit.
\preprintcontent{Following \cite{meanti2020kernel}, we use the root mean square error (RMSE) between the predictions $\hat{y}$ and targets $y$:
\[
    \mathrm{RMSE}(\hat{y}, y) = \sqrt{\frac{1}{n_{\mathrm{tst}}} \sum_{i = 1}^{n_{\mathrm{tst}}} \frac{(\hat{y}_i - y_i)^2}{2}}.
\]
}%
\nonpreprintcontent{Following \cite{meanti2020kernel}, we use the root mean square error (RMSE) between the predictions and targets.}%

The results are shown in \cref{fig:taxi_intro}.
\asko{} outperforms \fal{} for each value of $r$.
Both \epro{} 2.0 and \epro{} 3.0 (not shown) diverge.
Once again, our findings demonstrate the value of using full KRR over inducing points KRR for large-scale regression tasks.

\subsection{\asko{} Converges Linearly to the Optimum}
\label{subsec:lin_cvg}

We demonstrate that \asko{} obtains linear convergence to the optimum for large-scale full KRR.
To do so, we plot the relative residual
\[
\|\Klambd w - y\| / \| y \|
\]
obtained by \asko{}.
We run \asko{} in double precision.

\cref{fig:lin_cvg_actual_data} shows the relative residual obtained by \asko{} on three large-scale KRR problems.
\asko{} converges linearly for all selections of the rank $r$.
Excitingly, \asko{} reaches machine precision on all three problems, showing its potential as a high-precision linear system solver.

\begin{figure}[htbp]
    \centering
    \nonpreprintcontent{\includegraphics[width=\linewidth]{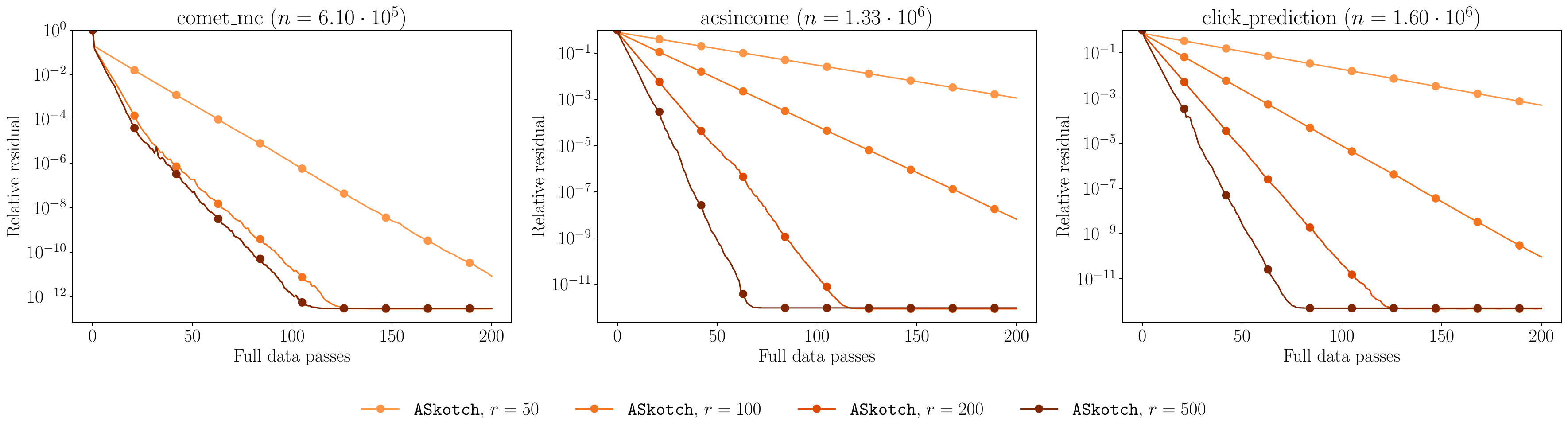}}%
    \preprintcontent{\includegraphics[width=\linewidth]{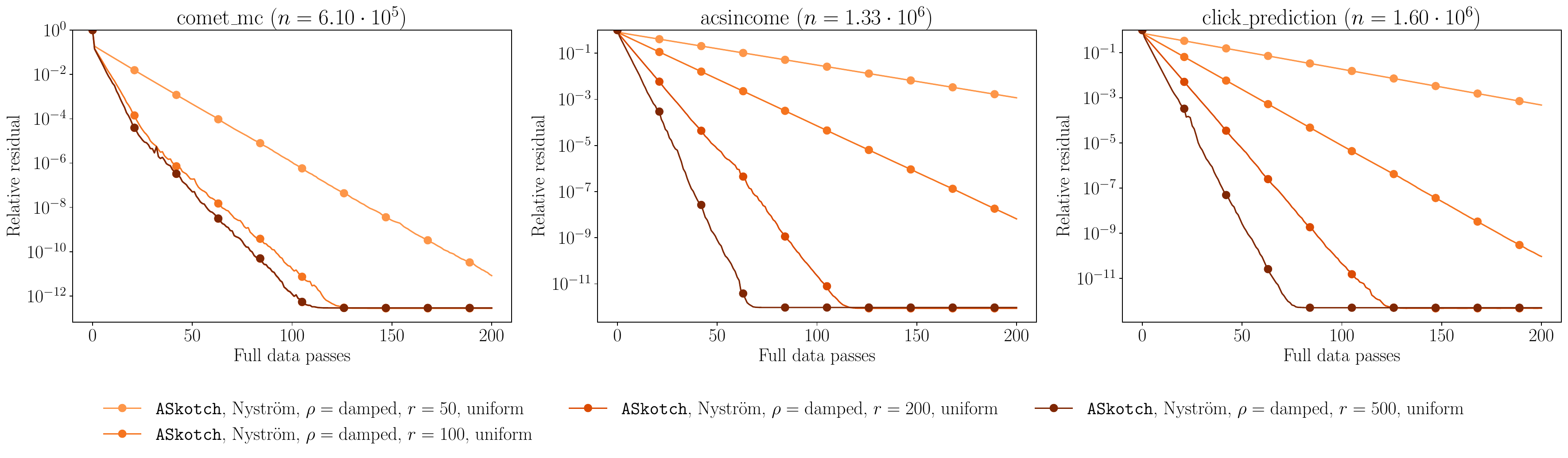}}%
    \caption{\asko{} converges linearly on large-scale full KRR. 
    ``Full data passes'' indicates the number of passes through $\Klambd$: since $b = n/100$, one full data pass is equivalent to 100 iterations of \asko{}.
    \asko{} tends to converge faster as the rank $r$ increases, which matches our theoretical results.
    Interestingly, $r = 200$ and $r = 500$ yield similar results on comet\_mc.
    }
    \label{fig:lin_cvg_actual_data}
\end{figure}

\preprintcontent{
\subsection{Ablation Study of \asko{}}
\label{subsec:ablation}
\asko{} combines several algorithmic techniques,
such as \nys{} approximation, acceleration, and coordinate sampling schemes.
We perform an ablation study to understand how each of these techniques impacts the performance of \asko{}.
We show results for selected classification and regression tasks in \cref{fig:particle_physics_ablation,fig:molecules_ablation}; additional figures are in \cref{subsec:ablation_additional}. 

\textit{\nys{} approximation.} 
We study what happens when we replace the regularized \nys{} approximation in the \asko{} update with the identity matrix.
We also compare the damped vs. regularization settings for $\rho$.
Replacing the \nys{} approximation with the identity matrix in the approximate projection significantly degrades the performance of both \sko{} and \asko{}.
Moreover, the damped setting for $\rho$ leads to better performance on regression tasks, while yielding similar performance to the regularized setting on classification tasks.
This finding aligns with our theory in \cref{sec:asko_cvg}, as \cref{thm:mast_conv_thm} states $\rho$ should be at least as large as $\lambda$.

\textit{Acceleration.}
We compare \asko{} with \sko{}, since \asko{} is the accelerated version of \sko{}.
\asko{} and \sko{} perform similarly on classification tasks, but \asko{} outperforms \sko{} on regression tasks.

\textit{Coordinate sampling schemes.}
We compare uniform vs.~approximate RLS sampling.
The results show that the sampling scheme has little to no impact on the performance of \asko{} and \sko{}.

\begin{figure}[htbp]
    \centering
    \includegraphics[width=1.0\linewidth]{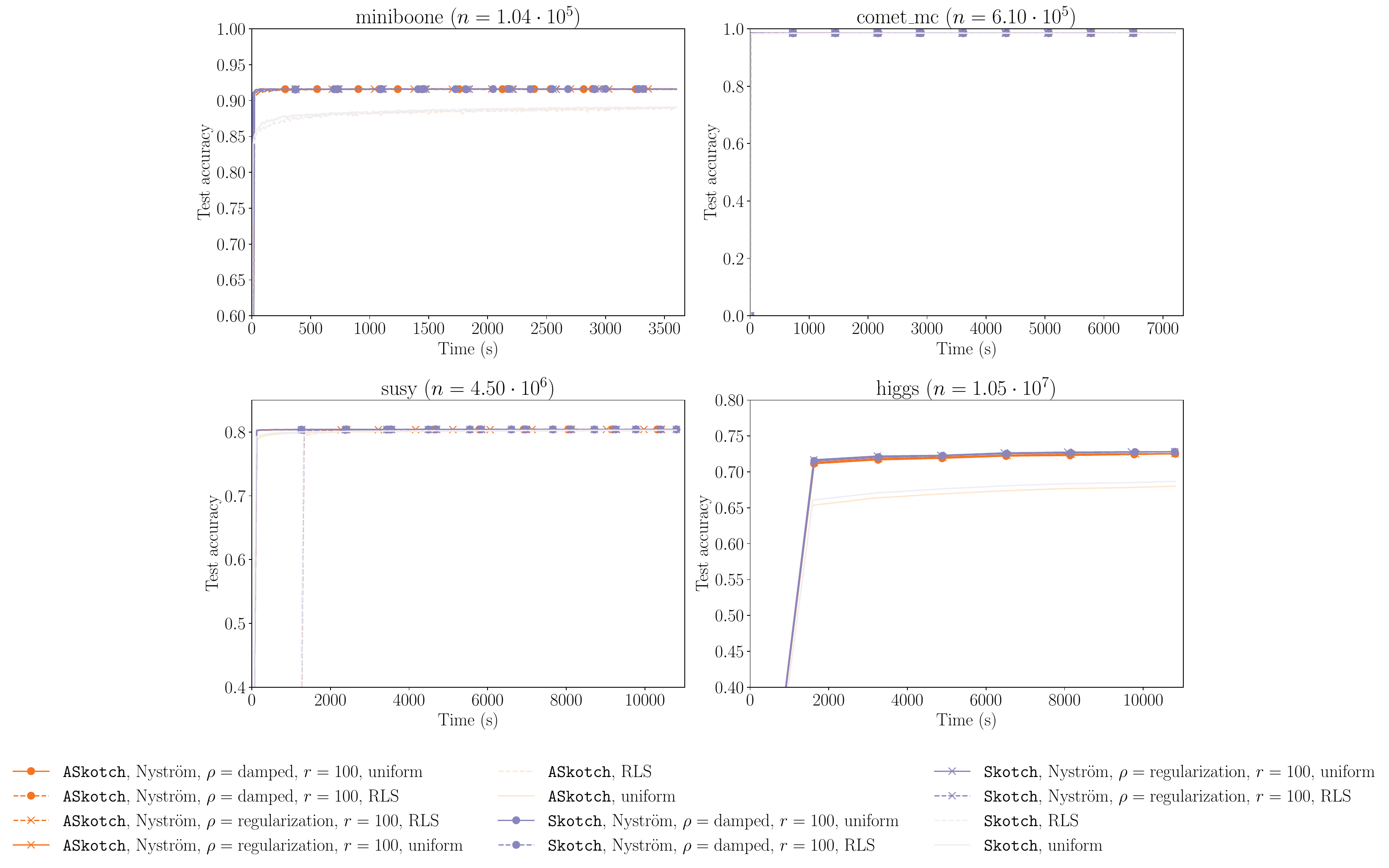}
    \caption{Ablation study of \sko{} and \asko{} on classification tasks from particle physics.
    Using the \nys{} approximation often improves classification accuracy.
    Acceleration and approximate RLS sampling have little to no impact.}
    \label{fig:particle_physics_ablation}
\end{figure} 

\begin{figure}
    \centering
    \includegraphics[width=1.0\linewidth]{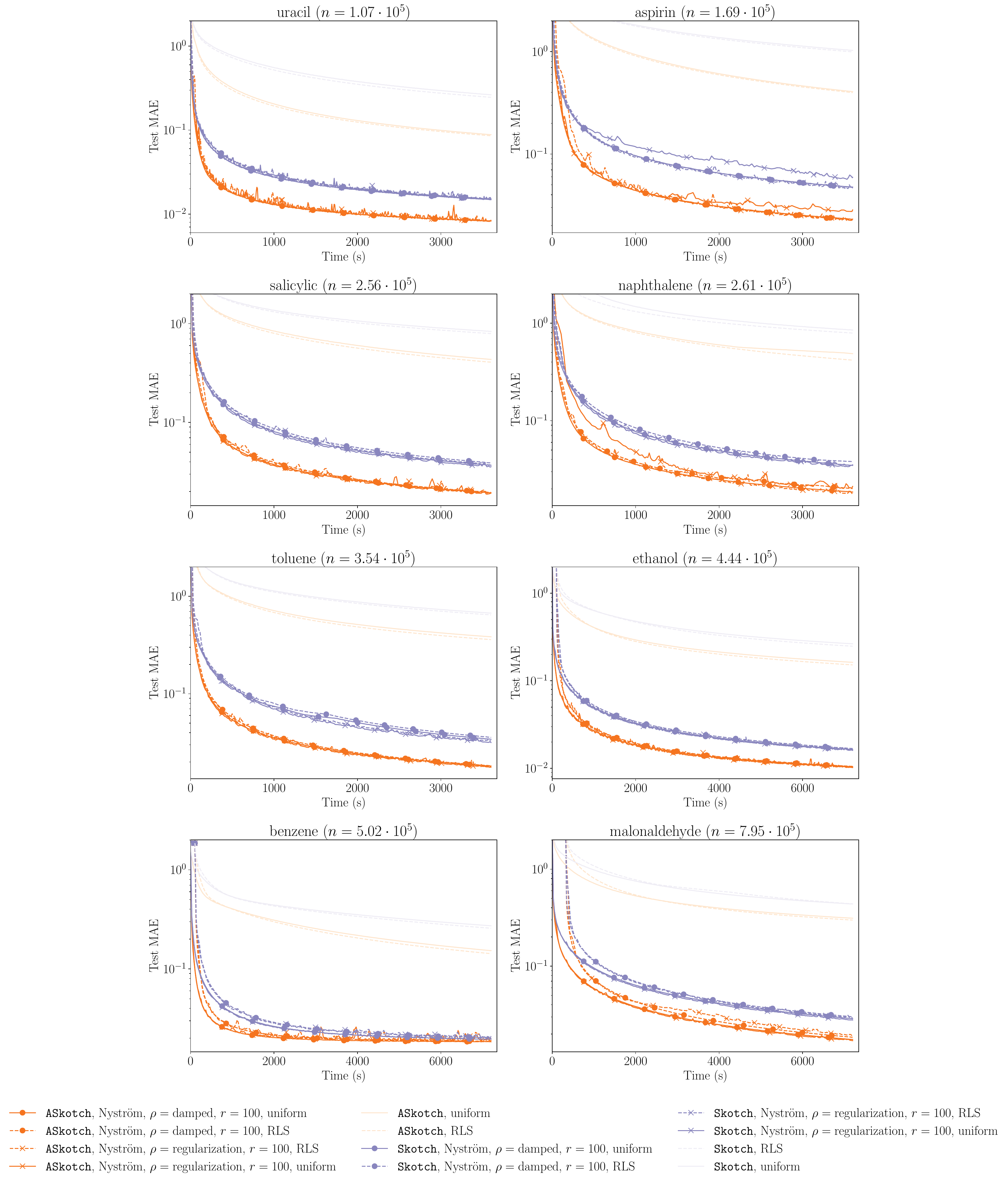}
    \caption{Ablation study of \sko{} and \asko{} on regression tasks for the eight molecules datasets.
    Using acceleration and the \nys{} approximation improves MAE across all eight datasets.
    Furthermore, the damping strategy for the \nys{} approximation leads to faster convergence on MAE.}
    \label{fig:molecules_ablation}
\end{figure}
}%
\section{Conclusion}
\label{sec:conclusion}
We introduce \asko{}, an approximate sketch-and-project method for large-scale full KRR. 
Our theoretical analysis and experiments demonstrate that \asko{} is a promising replacement for existing KRR solvers.
Looking ahead, we aim to develop methods for automatically selecting the acceleration parameters $\hat \mu$ and $\hat \nu$ in \asko{}, create a distributed implementation of \asko{} that scales to datasets with $n \gtrsim 10^9$ training points (such as the full taxi dataset), 
and develop a mixed-precision version of \asko{} that delivers high-quality results with less memory and compute.

Our work, along with previous studies \citep{wang2019exact,frangella2023randomized,diaz2023robust}, shows that full KRR allows better predictive performance than inducing points KRR. 
\asko{} scales full KRR to datasets orders of magnitude larger than previously possible. 
In future work, we look forward to seeing how a fast solver for full KRR enables new applications.
Moreover, following our methodology, it should be possible to solve general pd linear systems within the (approximate) sketch-and-project framework.

\acks{
We thank Rob Webber for helpful discussions.

PR, ZF, and MU gratefully acknowledge support from
the National Science Foundation (NSF) Award IIS-2233762, 
the Office of Naval Research (ONR) Awards N000142212825, 
N000142412306, and N000142312203, the Alfred P. Sloan Foundation, and from IBM Research as a founding member of Stanford Institute for Human-centered Artificial Intelligence (HAI). 
MD and JY gratefully acknowledge support from NSF Award CCF-2338655.
}



\newpage

\appendix
\section{Additional Algorithm Details}

We provide additional details for the algorithms proposed in this paper.
\cref{subsec:nystrom} describes the practical implementation of the randomized \nys{} approximation and provides pseudocode for \texttt{\nys}.
\cref{subsec:get_l} describes how we compute preconditioned smoothness constants and provides pseuodocode for \texttt{get\_L}.

\subsection{Randomized \nys{} Approximation: Implementation}
\label{subsec:nystrom}

Here, we present a practical implementation of the Nystr{\"o}m approximation from \cref{subsec:nys_appx} in \algnys{} (\cref{alg:nystrom}). 
\algnys{} is based on \citet[Algorithm 3]{tropp2017fixed}. 
$\text{eps}(x)$ is defined as the positive distance between $x$ and the next largest floating point number of the same precision as $x$. 
The resulting Nystr\"{o}m approximation $\hat{M}$ is given by $\hat U \diag(\hat \Lambda) \hat U^T$, where $\hat{U} \in \R^{p \times r}$ is an orthogonal matrix that contains the approximate top-$r$ eigenvectors of $M$, and $\hat \Lambda \in \R^r$ contains the top-$r$ eigenvalues of $M$.
The Nystr{\"o}m approximation is psd but may have eigenvalues that are equal to $0$.
In our algorithms, this approximation is always used in conjunction with a regularizer to ensure positive definiteness.

\begin{algorithm}
    \caption{\texttt{\nys}}
    \label{alg:nystrom}
    \scriptsize
    \begin{algorithmic}
        \Require psd matrix $M \in \R^{p \times p}$, approximation rank $r \leq p$
        \State $\Omega \gets \mathrm{randn}(p, r)$
        \hfill \Comment{Test matrix}
        \State $\Omega \gets \mathrm{thin\_qr}(\Omega )$
        \hfill \Comment{Orthogonalize test matrix}
        \State $\Delta \gets \mathrm{eps}(\Omega.\mathrm{dtype}) \cdot \Tr(M)$ 
        \hfill \Comment{Compute shift}
        \State $Y_{\Delta} \gets (M + \Delta I) \Omega$ 
        \hfill \Comment{Compute sketch, adding shift for stability}
        \State $C \gets \text{chol}(\Omega^TY_\Delta)$ 
        \hfill \Comment{Cholesky decomposition: $C^{T}C = \Omega^{T}Y_\Delta$}
        \State $B \gets YC^{-1}$ 
        \hfill \Comment{Triangular solve}
        \State $[\hat U, \Sigma, \sim] \gets \text{svd}(B, 0)$ 
        \hfill \Comment{Thin SVD}
        \State $\hat \Lambda \gets \text{max}\{0, \diag(\Sigma^2 - \Delta I)\}$ 
        \hfill \Comment{Compute eigs, and remove shift with element-wise max}
        \State \Return $\hat{U} \diag(\hat \Lambda) \hat{U}^T$
    \end{algorithmic}
\end{algorithm}

The dominant costs are in computing the (shifted) sketch $Y_\Delta$, which has complexity $\bigO(p^2 r)$, and computing an SVD of $B$, which has complexity $\bigO(p r^2)$.
In total, the overall complexity of the algorithm is $\bigO(p^2 r + p r^2)$.

\subsubsection{Applying the \nys{} Approximation to a Vector}
In our algorithms, we 
often perform computations of the form $(\hat{M} + \rho I)^{-1} g = (\hat{U} \diag(\hat \Lambda) \hat{U}^T + \rho I)^{-1} g$.
This computation can be performed in $\bigO(rp)$ time using the Woodbury formula \citep{higham2002accuracy}:
\begin{align}
\label{eq:inv_nys_woodbury}
        (\hat{U} \diag(\hat \Lambda) \hat{U}^T + \rho I)^{-1} g
        &= \hat{U} \left( \diag (\hat \Lambda) + \rho I \right)^{-1} \hat{U}^{T} g + \frac{1}{\rho} (g - \hat{U} \hat{U}^{T} g).
\end{align}

We also use the randomized \nys{} approximation to compute preconditioned smoothness constants in \texttt{get\_L} (\cref{alg:get_L}).
This computation requires the calculation $(P + \rho I)^{-1/2} v$ for some $v \in \R^p$, which can also be performed in $\bigO(pr)$ time using the Woodbury formula:
\begin{align}
\label{eq:inv_sqrt_nys_woodbury}
        (\hat{U} \diag(\hat \Lambda) \hat{U}^T + \rho I)^{-1/2} v = \hat{U} \left( \diag (\hat \Lambda) + \rho I \right)^{-1/2} \hat{U}^{T} v + \frac{1}{\sqrt{\rho}} (v - \hat{U} \hat{U}^{T} v).
\end{align}

In single precision, \eqref{eq:inv_nys_woodbury} is unreliable for computing $(P + \rho I)^{-1} g$. This instability arises due to roundoff error: 
the derivation of \eqref{eq:inv_nys_woodbury} assumes that $\hat U^T \hat U = I$, but we empirically observe that orthogonality does not hold in single precision.
To improve stability, we compute a Cholesky decomposition $LL^T$ of $\rho \diag(\hat \Lambda^{-1}) + \hat U^T \hat U$, which takes $\bigO(pr^2)$ time since we form $\hat U^T \hat U$.
Using the Woodbury formula and Cholesky factors, 
\begin{align*}
   (\hat{U} \diag(\hat \Lambda) \hat{U}^T + \rho I)^{-1} g &= \frac{1}{\rho} g - \frac{1}{\rho} \hat U (\rho \diag(\hat \Lambda^{-1}) + \hat U^T \hat U)^{-1} \hat U^T g  \\
   &= \frac{1}{\rho} g - \frac{1}{\rho} \hat U L^{-T} L^{-1} \hat U^T g.
\end{align*}
This computation can be performed in $\bigO(pr)$ time, since the $\bigO(r^2)$ cost of triangular solves with $L^T$ and $L$ is negligible compared to the $\bigO(pr)$ cost of multiplication with $\hat U^T$ and $\hat U$.

Even in single precision, we find that using \eqref{eq:inv_sqrt_nys_woodbury} in \texttt{get\_L} works well in practice. 

\subsection{Computing Preconditioned Smoothness Constants}
\label{subsec:get_l}
Here, we provide the details on the randomized powering procedure (\texttt{get\_L}, \cref{alg:get_L}) from \cref{subsec:auto_stepsize} for automatically computing the preconditioned smoothness constant.
Given a symmetric matrix $H$, preconditioner $P$, and damping $\rho$, \texttt{get\_L} uses randomized powering \citep{martinsson2020randomized} to compute 
\[
\lambda_1((P + \rho I)^{-1/2} H (P + \rho I)^{-1/2}).
\]
The algorithm only requires matrix-vector products with the matrices $H$ and $(P + \rho I)^{-1/2}$. 
When $P$ is calculated using $\texttt{\nys}$, we can efficiently compute a matrix-vector product with $(P + \rho I)^{-1/2}$ using \eqref{eq:inv_sqrt_nys_woodbury}.
In practice, we find that 10 iterations of randomized powering are sufficient for estimating the preconditioned smoothness constant.

\begin{algorithm}
    \caption{\texttt{get\_L}}
    \label{alg:get_L}
    \scriptsize
    \begin{algorithmic}
        \Require symmetric matrix $H$, preconditioner $P$, damping $\rho$, maximum iterations $N \gets 10$
        \State $v^0 \gets \mathrm{randn}(P.\mathrm{shape}[0])$
        \State $v^0 \gets v^0 / \|v^0\|_2$ 
        \hfill \Comment{Normalize}
        \For{$i = 0, 1, \ldots, N - 1$}
        \State $v^{i + 1} \gets (P + \rho I)^{-1/2} v^i$
        \State $v^{i + 1} \gets H v^{i + 1}$
        \State $v^{i + 1} \gets (P + \rho I)^{-1/2} v^{i+1}$
        \State $v^{i + 1} \gets v^{i + 1} / \|v^{i + 1}\|_2$ 
        \hfill \Comment{Normalize}
        \EndFor
        \State $\lambda \gets (v^{N - 1})^T v^N$
        \State \Return $\lambda$
    \end{algorithmic}
\end{algorithm}
\section{Proofs of Results Appearing in the Main Paper}
This appendix provides proofs for every result in the main paper whose proof was omitted. 
In particular, we provide detailed arguments for the RLS-to-DPP reduction in \cref{subsec:appx_rls} and the convergence of \asko{} (\cref{thm:mast_conv_thm}).
\subsection{Proof of \cref{lemma:appx_proj}}
\begin{proof}
    Consider the matrix
    \[
    K_\lambda^{1/2}\Ib^{T}\left(\Knys+\rho I\right)^{-1}\Ib K_{\lambda}^{1/2}.
    \]
    This can be rewritten as
    \[
     K_\lambda^{1/2}\Ib^{T}(\Kbb+\lambda I)^{-1/2}\left[(\Kbb+\lambda I)^{1/2}\left(\Knys+\rho I\right)^{-1}(\Kbb+\lambda I)^{1/2}\right](\Kbb+\lambda I)^{-1/2}\Ib K_{\lambda}^{1/2}.
    \]
    From this, we deduce
    \begin{equation*}
    \sigma_{P_\B}\proj_\lambda \preceq  K_\lambda^{1/2}\Ib^{T}\left(\Knys+\rho I\right)^{-1}\Ib K_{\lambda}^{1/2} \preceq L_{P_B} \proj_\lambda. 
    \end{equation*}
    By definition, $\hat L_{P_\B} \geq L_{P_\B}$, so
    \[
    \sigma_{P_\B}\proj_\lambda \preceq  K_\lambda^{1/2}\Ib^{T}\left(\Knys+\rho I\right)^{-1}\Ib K_{\lambda}^{1/2} \preceq \hat L_{P_B} \proj_\lambda.
    \]
    Multiplying by $\hat L_{P_\B}^{-1}$, we establish the claim.
\end{proof}


\subsection{ARLS-to-DPP Reduction}

Here we prove the ARLS-to-DPP reduction in \cref{lem:rls_sample}.

\textit{Proof idea.} To prove \cref{lem:rls_sample}, we first establish that the leverage score estimates $\{\tilde\ell_i\}_{i=1}^n$ are rough approximations for the marginal probabilities (i.e., \textit{marginal overestimates}) of any given element $i\in[n]$ being sampled into a set distributed according to $\hdpp{j}(A)$ (\cref{lem:arls_marginal}). 
We must then combine this guarantee with existing results on DPP coupling \citep{derezinski2024solving}. 
However, since prior work focuses on uniform sampling rather than \arls{} sampling, we construct a \textit{subdivision}, which uses the marginal overestimates to create an equivalent DPP for which our \arls{} sampling maps to uniform sampling (\cref{lem:subdivision_uniform} and \cref{lem:equivalence_oracle}). 
This equivalence allows us to call upon existing guarantees for coupling a DPP with uniform sampling (\cref{{lem:couple_dy24}}).

The relationship between leverage score estimates and marginal probabilities is shown in \cref{subsubsec:marginal_overestimates} and the subdivision construction is shown in \cref{subsubsec:subdivision_construction}.
The proof of \cref{lem:rls_sample} is given in \cref{appendix:rls_sample}.

\subsubsection{Leverage Score Estimates are Rough Approximations of Marginals}
\label{subsubsec:marginal_overestimates}
We start by showing that approximate RLSs are marginal overestimates of a $\hdpp{j}(A)$.

\begin{lemma}[ARLS are marginal overestimates]
\label[lemma]{lem:arls_marginal}
Given $A\in\psd^n$, $k=\Omega(\log n)$, and $\tilde\lambda>0$ such that   $\deff{\tilde\lambda}(A)\geq 4k$, let $\{\tilde{\ell}_i\}_{i=1}^n$ be $c$-approximations of $\tilde\lambda$-ridge leverage scores of $A$. For $j \in \I = \left[ 2k - \sqrt{6k \log \left( \frac{2}{\delta} \right)}, 2k + \sqrt{6k \log \left( \frac{2}{\delta} \right)} \right]$, let $\ell_{i \mid j} \coloneqq \prfn{i \in \B_{\hdpp{j}}}$ be the $i$-th marginal probability of $\B_{\hdpp{j}}\sim\hdpp{j}(A / \tilde \lambda)$. Then, $2\tilde{\ell}_i \geq \ell_{i \mid j}$ for all $i$.
\end{lemma}

The proof of \cref{lem:arls_marginal} requires (i) \cref{lem:dpp-size}, which shows that the size of a DPP sample is close to its expected size with high probability and (ii) \cref{lem:p_ij}, which shows that $\ell_{i \mid j}$ is non-decreasing in $j$:

\begin{lemma}\label[lemma]{lem:dpp-size}
Given any $A\in\psd^n$, let $\bdpp \sim \dpp(A)$ and $\E[|\bdpp|]$ be its expected size. 
Then with probability $1-\delta$, we have $\big||\bdpp| - \E[|\bdpp|]\big|\leq \sqrt{3\E[|\bdpp|] \log \left( \frac{2}{\delta} \right)}$. 
\end{lemma}

\begin{lemma}[Monotonicity]
\label[lemma]{lem:p_ij}
For any $0 < j_1 \leq j_2$, we have $\ell_{i \mid j_1} \leq \ell_{i \mid j_2}$.
\end{lemma}

The proofs of \cref{lem:dpp-size,lem:p_ij} are deferred to \cref{appendix:dpp-size,subsubsec:lem:p_ij_pf}, respectively.

\begin{proof}[Proof of \cref{lem:arls_marginal}]
Define $\tbdpp \sim \dpp(A / \tilde \lambda)$ with $\E[|\tbdpp|] = \deff{\tilde\lambda}:=\deff{\tilde\lambda}(A) \geq 4k$, and denote its $i$-th marginal probabilities as $\ell_i^{\tilde\lambda}=\prfn{i \in \tbdpp}$. 
By \cref{lem:dpp_rls}, $\{\ell_i^{\tilde\lambda}\}_{i=1}^n$ are the $\tilde\lambda$-ridge leverage scores of $A$.
Also let $w_{\tilde\lambda,l} \coloneqq \prfn{|\tbdpp| = l}$ be the probability of the set $\tbdpp$ having size $l$. Then, we can express the marginals of $\tbdpp$ via the law of total probability: 
\begin{align*}
\ell_i^{\tilde\lambda} = \prfn{i \in \tbdpp} = 
\sum_l \prfn{i \in \tbdpp \mid |\tbdpp| = l} \cdot \prfn{|\tbdpp| = l} = 
\sum_l \ell_{i \mid l} \cdot w_{\tilde\lambda, l},
\end{align*}
where we use the fact that $\prfn{i \in \tbdpp \mid |\tbdpp| = l} = \prfn{i \in \B_{\hdpp{l}}} = \ell_{i \mid l}$. Defining the interval $\tilde\I \coloneqq \left[\deff{\tilde\lambda} - \sqrt{3 \deff{\tilde\lambda} \log \left(\frac{2}{\delta}\right)},\deff{\tilde\lambda} + \sqrt{3 \deff{\tilde\lambda} \log \left(\frac{2}{\delta}\right)}\right]$, by applying \cref{lem:dpp-size} to $\tbdpp$, we have $|\tbdpp| \in \tilde\I$ holds with probability $1-\delta$.

Since we assume $\delta = n^{-\bigO(1)}$ and $\deff{\tilde\lambda} \geq 4k = \Omega(\log n)$ (with a sufficiently large constant), the infimum of $\tilde\I$ can be bounded as $\deff{\tilde\lambda} - \sqrt{3 \deff{\tilde\lambda} \log\left(\frac{2}{\delta}\right)} \geq \frac34\deff{\tilde\lambda}\geq 3k$. Similarly, the supremum of $\I$ can be bounded as $2k + \sqrt{6k \log \left( \frac{2}{\delta} \right)} \leq 3k$, showing that $\sup\I \leq \inf \tilde \I$. By combining this result with \cref{lem:p_ij}, we have the following for any $j \in \I$:
\begin{align*}
\ell_{i}^{\tilde\lambda} = \sum_{l \in \tilde\I} \ell_{i \mid l} \cdot w_{\tilde\lambda, l} + \sum_{l \notin \tilde\I} \ell_{i \mid l} \cdot w_{\tilde\lambda, l} \geq \sum_{l \in \tilde\I} \ell_{i \mid l} \cdot w_{\tilde\lambda, l} \geq (1-\delta) \cdot \ell_{i \mid 3k} \geq (1-\delta) \cdot \ell_{i \mid j}.
\end{align*}
Suppose we have RLS $c$-approximations $\{\tilde{\ell}_i\}_{i=1}^n$ such that $\tilde{\ell}_i \geq \ell_{i}^{\tilde\lambda}$ (see \cref{def:rls-approx}).
Since $\delta\leq 1/2$, this implies that $2\tilde{\ell}_i \geq \ell_i^{\tilde\lambda}/(1-\delta) \geq \ell_{i \mid j}$ for any $j \in \I$, i.e., the RLS approximations are (up to a factor of 2) marginal overestimates for $\hdpp{j}(A)$.
\end{proof}

\subsubsection{Relating ARLS Sampling to DPPs Using Subdivisions}
\label{subsubsec:subdivision_construction}

The remainder of our analysis builds on the following lemma from \cite{derezinski2024solving}, which couples a fixed-size DPP, $\B_{\hdpp{j}} \sim \hdpp{j}(A)$, with uniform sampling.
\begin{lemma}[\cite{derezinski2024solving}, Lemma 6.6]
\label[lemma]{lem:couple_dy24}
Let $\B_{\hdpp{j}}\sim \hdpp{j}(A)$ be a fixed-size DPP sample where $A \in \psd^n$ and $\log n < j < n$. Suppose the marginal probabilities of $\hdpp{j}(A)$ are near-uniform, that is, $\prfn{i \in \B_{\hdpp{j}}} \leq cj/n$ holds for all $i \in [n]$ for some $c > 1$. 
Let $\B$ be an i.i.d.~uniform sample from $[n]$ of size $\bigO(c j \log^3 n)$. Then, there is a coupling between $\B_{\hdpp{j}}$ and $\B$, such that the joint random variable $(\B_{\hdpp{j}}, \B)$ with probability $1-n^{-\bigO(1)}$ satisfies $\B_{\hdpp{j}} \subseteq \B$.
\end{lemma}
Unfortunately, we cannot use \cref{lem:couple_dy24} directly, because the ``near-uniform marginals'' assumption does not hold in our setting. 
To address this, we rely on a subdivision process, which transforms a probability distribution (such as a fixed-size DPP) into another distribution which is equivalent (in some sense) and has nearly uniform marginals. 
This approach, first introduced by \cite{anari2020isotropy}, uniformizes the marginals by enlarging the groundset $[n]$ through selective duplication.
\begin{definition}[Subdivision of $\hdpp{j}$, inspired by \cite{anari2020isotropy}]
Given $A \in \psd^n$, define the distribution $\mu \coloneqq \hdpp{j}(A)$. 
Let matrix $X$ be such that $A = XX^\top$, and denote $x_i^\top$ as its $i$-th row. 
Assume we have marginal overestimates $\{\tilde{\ell}_i\}_{i=1}^n$ such that $\tilde{\ell}_i \geq \Pr_{\S \sim \mu}(i \in \S)$ for all $i \in [n]$. 
Denote $\tilde{\ell} = \sum_{i} \tilde{\ell}_i$, let $t_i \coloneqq \left\lceil \frac{n}{\tilde{\ell}} \tilde{\ell}_i \right\rceil$ and $\tilde{n} = \sum_i t_i$. 
For each $i \in [n]$, we create $t_i$ copies of $\frac{1}{\sqrt{t_i}}x_i^\top$ and let the collection of all these vectors form the new matrix $\tilde{X}$. 
Let $\tilde{A} = \tilde{X} \tilde{X}^\top \in \psd^{\tilde{n}}$, we define the subdivision process of $\mu$ as $\mu' = \hdpp{j}(\tilde{A})$.
\end{definition}

We start by stating that the subdivision process produces a distribution with near-uniform marginals \citep{anari2024optimal}.

\begin{lemma}[\cite{anari2024optimal}, Proposition 23]
\label[lemma]{lem:subdivision_uniform}
Let distribution $\mu = \hdpp{j}(A)$, and $\mu' = \hdpp{j}(\tilde A)$ be the subdivision process of $\mu$. Then, $\mu'$ has near-uniform marginals: for all $i^{(j)} \in [\tilde{n}]$ we have $\Pr_{\tilde\S\sim \mu'}(i^{(j)} \in \tilde\S) \leq \frac{\tilde\ell}{n} \leq \frac{2\tilde\ell}{\tilde n}$.
\end{lemma}

The following lemma shows an equivalence between $\mu = \hdpp{j}(A)$ and its subdivision (up to a mapping of the elements), which allows us to transform the problem of sampling from $\mu$ to sampling from its subdivision $\mu'$. 
\begin{lemma}[Equivalence between $\hdpp{j}$ and its subdivision]
\label[lemma]{lem:equivalence_mu}
Let distribution $\mu = \hdpp{j}(A)$, and $\mu' = \hdpp{j}(\tilde A)$ be its subdivision process. Let $\tilde\B \sim \hdpp{j}(\tilde A)$ and define a function $\pi: [\tilde{n}] \to [n]$ which maps the $t_i$ duplicates (in $[\tilde n]$) of element $i\in[n]$ back to $i$. Then, $\pi \big(\tilde \B \big) \sim \hdpp{j}(A)$.
\end{lemma}

\begin{proof}
By the definition of the subdivision, if two identical rows are sampled from $\tilde{X}$, then the determinant of the corresponding principal submatrix is $0$, as is the probability measure $\mu'$. 
Thus without loss of generality, we can assume that we only sample distinct rows from $\tilde{X}$.
Let $\tilde{\B} = \{i_1^{\left(l_1\right)}, \ldots, i_{j}^{\left(l_{j}\right)}\} \subseteq {[\tilde{n}] \choose j}$ be a subset sampled from $\mu'$, and let $\hat{\B} = \pi\big(\tilde{\B}\big) \coloneqq \{\pi(i_1^{\left(l_1\right)}), \ldots, \pi(i_{j}^{\left(l_{j}\right)})\} = \{i_1, \ldots, i_{j}\} \subseteq {[n] \choose j}$. Then, from definition we have
\begin{align*}
\mu'(\tilde{\B}) \propto \det(\tilde{X}_{\tilde{\B}} \tilde{X}_{\tilde{\B}}^\top) = \frac{\det(X_{\hat{\B}} X_{\hat{\B}}^\top)}{t_{i_1} \cdots t_{i_{j}}} \propto \frac{\mu(\hat{\B})}{t_{i_1} \cdots t_{i_{j}}}.
\end{align*}
Note that there are precisely $t_{i_1}\cdots t_{i_{j}}$ different sets of duplicates that map to the same set $\hat\B$, so a random set $\tilde\B\sim \mu'= \hdpp{j}(\tilde A)$ after mapping with $\pi$ is distributed exactly according to $\mu = \hdpp{j}(A)$.
\end{proof}
Having shown that sampling from $\mu$ is equivalent to sampling from its subdivision $\mu'$, we need to establish a sampling scheme on $[n]$. Intuitively, importance sampling on $[n]$ based on probabilities proportional to $\{\tilde{\ell}_i\}_{i=1}^n$ should be equivalent to uniform sampling on $[\tilde{n}]$. 
However, the definition of $t_i$ introduces rounding error. 
To address this issue, we use the ARLS sampling scheme introduced in \cref{def:rls-approx}.
In the following result, we show the equivalence between this sampling scheme and uniform sampling on the enlarged subdivision ground set $[\tilde n]$.

\begin{lemma}[From subdivided uniform sampling to marginal sampling]
\label[lemma]{lem:equivalence_oracle}
Let $\{\tilde{\ell}_i\}_{i=1}^n$ be the marginal overestimates, and define $p_i \coloneqq \frac{\tilde{\ell}}{n} \cdot \big\lceil \frac{n}{\tilde{\ell}} \tilde{\ell}_i \big\rceil$ where $\tilde{\ell} = \sum_i \tilde{\ell}_i$. Let $\orc$ be the importance sampling distribution defined on $[n]$ according to probabilities $\{p_i\}_{i=1}^n$, and $\orc'$ be the uniform sampling distribution defined on $[\tilde n]$. Let $I\sim \orc$, and let $\hat I$ be a uniformly random element in $\pi^{-1}(I)$. Then $\hat I\sim\orc'$.
\end{lemma}

\begin{proof}
For any $i \in [n]$ and $l \in [t_i]$, we have the following:
\begin{align*}
\prfn{\hat{I} = i^{(l)}} &= \prfn{I = i} \cdot \prfn{\hat{I} = i^{(l)} \mid I = i} \\
&= \frac{p_i}{\sum_{j=1}^n p_j} \cdot \frac{1}{t_i} \\
&= \frac{\frac{\tilde{\ell}}{n} t_i}{\frac{\tilde{\ell}}{n} \sum_{j=1}^n t_j} \cdot \frac{1}{t_i} \\
&= \frac{1}{\sum_{j=1}^n t_j} \\
&= \frac{1}{\tilde n}.
\end{align*}
Therfore, $\hat I$ is uniformly random in $[\tilde n]$.
\end{proof}

\subsubsection{Proof of \cref{lem:rls_sample}}
\label{appendix:rls_sample}

With all the pieces in place, we now prove \cref{lem:rls_sample}.
\begin{proof}
Suppose that $\{\tilde{\ell}_i\}_{i=1}^n$ are $c$-approximations of $\tilde\lambda$-ridge leverage scores of $A$ as in \cref{def:rls-approx}. Then, according to \cref{lem:arls_marginal}, $\{2\tilde{\ell}_i\}_{i=1}^n$ are the marginal overestimates of $\hdpp{j}(A)$ for any $j \in \I$. Thus by \cref{lem:equivalence_oracle}, our $\arls[\tilde\lambda][c]$-sampling is equivalent (up to the mapping $\pi$) to uniform sampling on the subdivision ground set $[\tilde n]$. 

Now for $\mu = \hdpp{j}(A)$, we look at its subdivision process $\mu' = \hdpp{j}(\tilde A)$. According to \cref{lem:subdivision_uniform}, $\mu'$ has near-uniform marginals, i.e., $\Pr_{\S \sim \mu'}(i \in \S)\leq 2cj/n$. 
Thus by \cref{lem:couple_dy24}, if we let $\B'$ be a i.i.d.~uniform sample from $[\tilde n]$ of size $\bigO(cj \log^3 \tilde{n}) = \bigO(ck \log^3 n)$, then there is a coupling $(\B', \B_{\hdpp{j}}')$, such that with probability $1- n^{-\bigO(1)}$, we have $\B_{\hdpp{j}}' \subseteq \B'$, where $\B_{\hdpp{j}}' \sim \mu'$.

For one side, by \cref{lem:equivalence_mu} we know that $\pi(\B_{\hdpp{j}}')$ is distributed identically to the sample $\B_{\hdpp{j}}$, where $\pi$ is the mapping from \cref{lem:equivalence_mu}. For the other side, by \cref{lem:equivalence_oracle}, we know that $\pi(\B')$ is distributed according to $\arls[\tilde\lambda][c]$-sampling, i.e., the same as $\B$ from the statement of \cref{lem:rls_sample}. We conclude that for any $j \in \I$, if we do $\arls[\tilde\lambda][c]$-sampling with sample size $b = \bigO(ck \log^3 n)$ and obtain $\B$, then we can couple $\B$ with $\B_{\hdpp{j}}$ such that $\B_{\hdpp{j}} \subseteq \B$ with probability $1-n^{-\bigO(1)}$.
\end{proof}

\subsubsection{Proof of \cref{lem:dpp-size}}
\label{appendix:dpp-size}
To prove \cref{lem:dpp-size}, we start with the following result from \cite{kulesza2012determinantal}, which relates the cardinality of a random-size DPP to a sum of independent Bernoulli random variables.

\begin{lemma}[\cite{kulesza2012determinantal}, Algorithm 1 and Theorem 2.3]
\label[lemma]{lem:dpp_sample}
For $A \in \psd^n$, let $A = \sum_{i} \lambda_i u_i u_i^\top$ be its eigendecomposition where $\lambda_1 \geq \lambda_2 \geq \cdots \geq \lambda_n$. 
Suppose we independently sample $\gamma_i \sim \mathrm{Bernoulli} \left( \frac{\lambda_i}{\lambda_i + 1} \right)$ for $i \in [n]$ and we also sample $\bdpp \sim \dpp(A)$.
Then $|\bdpp| \overset{d}{=} \sum_{i = 1}^n \gamma_i$.
\end{lemma}

We also require the following Chernoff bound, which is a classic result in probability theory:

\begin{lemma}[Chernoff bound]
\label[lemma]{lem:chernoff}
Let $\bar{Z} = \sum_{i=1}^n Z_i$ where $Z_i \sim \mathrm{Bernoulli}(p_i)$ are independent Bernoulli random variables. Let $\mu = \E[\bar{Z}]$, then for all $\epsilon \in (0,1)$ we have
\begin{align*}
\prfn{|\bar{Z} - \mu| \geq \epsilon \mu} \leq 2\exp(-\epsilon^2 \mu /3).
\end{align*}
\end{lemma}

\begin{proof}[Proof of \cref{lem:dpp-size}]
Let $\{\lambda_i\}_{i=1}^n$ be the eigenvalues of $A$ in decreasing order, and let $\gamma_i \sim \mathrm{Bernoulli} \left( \frac{\lambda_i}{\lambda_i+1} \right)$ be $n$ independent random variables. 
By \cref{lem:dpp_sample}, $\E[|\B_{\dpp}|] = \sum_i \E[\gamma_i] = \sum_i \frac{\lambda_i}{\lambda_i+1}$. 
Applying  \cref{lem:chernoff} to $\{\gamma_i\}_{i=1}^n$, we obtain
\begin{align*}
\Pr \left( \left |\sum_i \gamma_i - \sum_i \E[\gamma_i] \right| \geq \epsilon \sum_i \E[\gamma_i] \right) \leq 2 \exp\left(-\epsilon^2 \sum_i \E[\gamma_i] /3 \right),
\end{align*}
which gives that with probability $1-\delta$, $\big||\bdpp| - \E[|\bdpp|]\big| \leq \sqrt{3\E[|\bdpp|] \log\left(\frac{2}{\delta}\right)}$.
\end{proof}

\subsubsection{Proof of \cref{lem:p_ij}}
\label{subsubsec:lem:p_ij_pf}
The proof of \cref{lem:p_ij} relies on elementary symmetric polynomials (\cref{def:esp}) and Newton's inequalities (\cref{lem:newton}).


\begin{definition}[Elementary symmetric polynomial]
\label[definition]{def:esp}
Given vector $\Lambda = (\lambda_1, \ldots, \lambda_n) \in \R^n$, we define its $i$-th elementary symmetric polynomial as 
\begin{align*}
s_{i}(\Lambda) := \sum_{\S \in \tbinom{[n]}{i}}\prod_{j \in \S} \lambda_j.
\end{align*}
\end{definition}


\begin{lemma}[Newton's inequalities]
\label[lemma]{lem:newton}
For non-negative real numbers $\lambda_1, \ldots, \lambda_n$, let $s_j$ be the $j$-th elementary symmetric polynomial in $\lambda_1, \ldots, \lambda_n$. If we denote $\bar{s}_j = s_j / \tbinom{n}{j}$, then $\bar{s}_j^2 \geq \bar{s}_{j+1}\bar{s}_{j-1}$.
\end{lemma}

\begin{proof}[Proof of \cref{lem:p_ij}]
Recall that we define $\ell_{i \mid j} = \prfn{i \in \B_{\hdpp{j}}}$, and it can also be written as $\ell_{i \mid j} = \prfn{i \in \bdpp \mid |\bdpp| = j}$, where $\bdpp \sim \dpp(A / \bar{\lambda})$ is a random-size DPP sample. 
Next, we use \citet[Eq.~5.33]{kulesza2012determinantal}:
\begin{align*}
\ell_{i \mid j} = \prfn{i \in \bdpp \mid |\bdpp| = j} =
\sum_{p=1}^n (e_i^T v_p)^2 \lambda_p \frac{s_{j-1}(\Lambda_{-p})}{s_j(\Lambda)},
\end{align*}
where $\{e_i\}_{i=1}^n$ are the standard basis vectors, $\{v_p\}_{p=1}^n$ are the eigenvectors of $A$, and $\lambda_p \coloneqq \lambda_p(A / \bar \lambda)$.

In order to show that $\ell_{i \mid j}$ is non-decreasing in $j$, it suffices to show that $\frac{s_{j-1}(\Lambda_{-p})}{s_j(\Lambda)}$ is non-decreasing in $j$ for any given $p \in [n]$. Denote $\bar{s}_j = s_j(\Lambda) / \tbinom{[n]}{j}$ as the mean of the $j$-th elementary symmetric polynomial, then by \cref{lem:newton} we have
\begin{align*}
1 \geq \frac{\bar{s}_{j-1} \bar{s}_{j+1}}{\bar{s}_j^2} = \frac{\tbinom{[n]}{j}^2}{\tbinom{[n]}{j-1} \tbinom{[n]}{j+1}} \frac{s_{j-1}(\Lambda) s_{j+1}(\Lambda)}{s_j^2(\Lambda)} = \frac{(j+1)(n-j+1)}{j(n-j)}\frac{s_{j-1}(\Lambda) s_{j+1}(\Lambda)}{s_j^2(\Lambda)},
\end{align*}
which gives
\begin{align*}
\frac{s_{j+1}(\Lambda)}{s_j(\Lambda)} \leq \frac{j(n-j)}{(j+1)(n-j+1)} \frac{s_j(\Lambda)}{s_{j-1}(\Lambda)} < \frac{s_j(\Lambda)}{s_{j-1}(\Lambda)}.
\end{align*}
Notice that $s_{j+1}(\Lambda) = s_{j+1}(\Lambda_{-p}) + \lambda_p \cdot s_j(\Lambda_{-p})$, and by using this fact we have
\begin{align*}
\frac{s_j(\Lambda_{-p})}{s_{j+1}(\Lambda)} = \frac{s_j(\Lambda_{-p})}{s_{j+1}(\Lambda_{-p}) + \lambda_p\cdot s_j(\Lambda_{-p})} = \frac{1}{\lambda_p + \frac{s_{j+1}(\Lambda_{-p})}{s_j(\Lambda_{-p})}}.
\end{align*}
We have shown that $\frac{s_{j+1}(\Lambda_{-p})}{s_j(\Lambda_{-p})}$ is decreasing in $j$ when $p$ is fixed, so
\begin{align*}
\frac{s_j(\Lambda_{-p})}{s_{j+1}(\Lambda)} = \frac{1}{\lambda_p + \frac{s_{j+1}(\Lambda_{-p})}{s_j(\Lambda_{-p})}} > \frac{1}{\lambda_p + \frac{s_j(\Lambda_{-p})}{s_{j-1}(\Lambda_{-p})}} = \frac{s_{j-1}(\Lambda_{-p})}{s_j(\Lambda)},
\end{align*}
which shows that $\frac{s_{j-1}(\Lambda_{-p})}{s_j(\Lambda)}$ is increasing in $j$ for any $p \in [n]$.
\end{proof}

\subsection{Proof of \cref{corollary:cnd_free_hat_mu}}
\begin{proof}
    \cref{thm:appx_proj_small_eig} guarantees that
    \[
    \hat \mu \geq \frac{\lambda}{16\rho\tailcond{k}{n}}\frac{k}{n}.
    \]
    We upper bound $\tailcond{k}{n}$ to deduce the claim.
    By definition,
    \begin{align*}
        \tailcond{k}{n} & = \frac{1}{n-k}\sum_{j>k}\frac{\lambda_j(K_\lambda)}{\lmin(K_\lambda)}
        \\ & = \frac{1}{n-k}\sum_{j>k}\frac{\lambda_j(K)+\lambda}{\lmin(K)+\lambda} \\
        &\leq \frac{1}{n-k}\sum_{j>k}\frac{2\lambda}{\lmin(K)+\lambda} \\
        &\leq 2.
    \end{align*}
    Here the third inequality uses $\lambda_j(K)\leq \lambda$ whenever $j\geq 2\deff{\lambda}(K)$ \citep[Lemma 5.4]{frangella2023randomized}.
    Thus, $\tailcond{k}{n}^{-1}\geq 1/2$, which yields the claim.
\end{proof}

\subsection{Proof of \cref{prop:appx_proj}}
\label{subsec:appx_proj_pf}
We begin with the following preliminary technical result. 
\begin{lemma}
    \label[lemma]{lem:SKS_eff_dim}
    Let $K \in \psd^n$, $\Ib \in \R^{b\times n}$ be a row-selection matrix generated via an arbitrary sampling scheme.
    Set $\Kbb = \Ib K \Ib^{T}.$
    Then, for any $\rho>0$, 
    \[
    \deff{\rho}(\Kbb) \leq \sum_{j=1}^{b}\frac{\lambda_j(K)}{\lambda_j(K)+\rho}.
    \]
\end{lemma}
\begin{proof}
    As $\Kbb = I_\B KI_\B^{T}$, it is a compression of $K$ to the subspace spanned by the columns of $I_\B^{T}$.
    Consequently, Cauchy's Interlacing Theorem \citep[Corollary III.1.5]{bhatia2013matrix} yields
    \[
    \lambda_j(\Kbb)\leq\lambda_j(K).
    \]
     Now, by definition of $\deff{\rho}(\cdot)$ and the fact that $f(x)=x/(x+\rho)$ is increasing in $x$ for $\rho>0$, we have
     \begin{align*}
         \deff{\rho}(\Kbb) = \sum_{j=1}^{b}\frac{\lambda_j(\Kbb)}{\lambda_j(\Kbb)+\rho} \leq \sum_{j=1}^{b}\frac{\lambda_j(K)}{\lambda_j(K)+\rho}.
     \end{align*}
    
\end{proof}
We now commence the proof of \cref{prop:appx_proj}.

\begin{proof}[Proof of \cref{prop:appx_proj}]
We begin by observing that \cref{lem:SKS_eff_dim} yields $\deff{\rho}(\Kbb) \leq \deff{\rho}(\lra{K}{b})$.
As $\Knys$ is constructed from a sparse sign embedding with $r = \bigO\left(\deff{\rho}(\lra{K}{b})\log\left(\frac{\deff{\rho}(\lra{K}{b})}{\delta}\right)\right)$ columns and $\zeta = \bigO\left(\log(\frac{\deff{\rho}(\lra{K}{b})}{\delta})\right)$ non-zeros per column, 
Lemma 4.6 in \cite{derezinski2025faster} implies
\[
\|\Kbb - \Knys\| \leq 2 \lambda_{r+1}(\Kbb)+\frac{1}{r}\sum_{j>r}\lambda_j(\Kbb),  ~\text{with probability at least}~1-\delta.
\]
Now, combining Lemma 5.4 in \cite{frangella2023randomized} with $r = \bigO\left(\deff{\rho}(\lra{K}{b})\log\left(\frac{\deff{\rho}(\lra{K}{b})}{\delta}\right)\right)$ yields
\[
\lambda_{r+1}(\Kbb) \leq \frac{\rho}{4}, \quad \frac{1}{r}\sum_{j>r}\lambda_j(\Kbb) \leq \frac{\rho}{2}.
\]
Thus,
\[
\|\Kbb - \Knys\| \leq \rho.
\]
Combining this last display with the relation that $\Knys \preceq \Kbb$ and our hypothesis that $\rho \geq \lambda$, we find
\begin{align*}
     \Kbb+\lambda I &\preceq \Kbb+\rho I = \Knys+ \rho I + \left(\Kbb-\Knys\right) \\
    & \preceq \Knys+ \rho I +\rho I \preceq \left(1+1\right)(\Knys+ \rho I) \\
    & = 2(\Knys+ \rho I).
\end{align*}
Moreover,
\begin{align*}
     \Kbb+ \lambda I = \Kbb + \frac{\lambda}{\rho}\rho I \succeq \frac{\lambda}{\rho}(\Kbb+\rho I).
\end{align*}
Combining these bounds with $\Kbb\succeq \Knys$, we deduce
\[
\frac{\lambda}{\rho} \left(\Knys +\rho I\right) \preceq \Kbb+\lambda I \preceq 2\left(\Knys+\rho I\right).
\]
The preceding display immediately implies
\[
\frac{\lambda}{\rho}\leq \sigma_{P_\B} \leq L_{P_\B} \leq 2.
\]
As $\hat L_{P_\B} = \max\{L_{P_\B}, 1\}$, it follows that $\hat L_{P_\B} \leq 2$.
Invoking \cref{lemma:appx_proj} we conclude the proof.


\end{proof}

\subsection{Proof of \cref{thm:mast_conv_thm}}
\label{subsec:ASkotch_Conv}
In this section, we prove \cref{thm:mast_conv_thm} by analyzing the convergence of \asko{}.
We begin with some preliminaries and notation.
\subsubsection{Preliminaries and Notation}
The convergence analysis of \asko{} is based on a Lyapunov function argument. 
For \nsap{} applied to a symmetric psd matrix $A\in \R^n$, \cite{gower2018accelerated} establishes convergence in terms of the Lyapunov function
\[
\Delta_t \coloneqq \|v_t-\wstar\|^2_{B^{1/2}\pinv{\E[\projB]}B^{1/2}}+\frac{1}{\mu}\|w_t-\wstar\|_{B}^2,
\]
where $B\in \pd^n$, $\projB \coloneqq B^{-1/2}AS^{T} \pinv{\left(SAB^{-1}AS^{T}\right)} SAB^{-1/2}$, and $\mu \coloneqq \lmin^{+}(\E[\projB])$.

In particular, they establish the following result:
\begin{proposition}[\cite{gower2018accelerated}, Eq. 39]
\label[proposition]{prop:nsap}
    Suppose that
    \[
    \null(\E[\projB]) = \null(A).
    \]
    Then at iteration $t$, \nsap{} satisfies
    \[
    \E[\Delta_t \mid w_{t-1}, v_{t-1}, z_{t-1} ] \leq \left(1-\sqrt{\frac{\mu}{\nu}}\right)\Delta_{t-1},
    \]
    where
    \[
    \mu = \lmin\left(\E[\projB]\right), \quad \nu \coloneqq \lmax\left(\E\left[\left(\E[\projB]^{-1/2}\projB \E[\projB]^{-1/2}\right)^2\right]\right).
    \]
\end{proposition}
\cref{prop:nsap} introduces a new parameter, $\nu$, the largest eigenvalue of the second moment of the normalized projection matrix.
Similar to how $\hat \mu$ is an analogue $\mu$ for analyzing \asko{}, there is an analogue to $\nu$ for analyzing \asko{}:
\[
\hat \nu \coloneqq \lmax\left(\E\left[\left(\E[\aprojB]^{-1/2}\aprojB \E[\aprojB]^{-1/2}\right)^2\right]\right).
\]
This new parameter $\hat \nu$ corresponds to the second moment of the normalized approximate projection matrix. 
To establish our convergence result, we need to upper bound $\hat \nu$.
The following lemma does precisely this. 
\begin{lemma}[Upper bound on  $\hat \nu$]
\label[lemma]{lem:hat_nu_bnd}
    Suppose that $\B$ and $\Knys$ are constructed according to the hypotheses of \cref{thm:mast_conv_thm}, then
    \[
    \hat \nu \leq \frac{8(\rho+\bar \lambda)}{\lambda} \leq 8\left(\frac{\rho}{\lambda}+\frac{n}{k}\right),
    \]
    where $\bar \lambda$ is as in \cref{lem:convergence_rls}.
\end{lemma}

\begin{proof}
     The proof parallels the proof of Theorem 3.7 in \cite{derezinski2025randomized}.
    Our hypothesis on the construction of $\B$ and $\Knys$ allows us to invoke \cref{thm:appx_proj_small_eig} to reach
     \[
     \E[\aprojB] \succeq \frac{1}{8\kappa_\rho}K_\lambda(K_\lambda+\bar \lambda)^{-1},
     \]
     where $\kappa_\rho = \rho/\lambda \geq 1$.
     The preceding display implies
    \begin{align*}
        \E[\aprojB]^{-1} \preceq 8\kappa_\rho (I+\bar \lambda K_\lambda^{-1}). 
    \end{align*}
    To upper bound $\hat \nu$, we observe that it may be rewritten as
    \[
    \hat \nu = \left\|\E[\aprojB]^{-1/2}\E\left[\aprojB \E[\aprojB]^{-1}\aprojB\right]\E[\aprojB]^{-1/2}\right\|.
    \]
    Combining this with the upper bound on $\E[\aprojB]^{-1}$ and the fact that $\aprojB\preceq I$, we deduce
    \begin{align*}
    \label{eq:hat_nu_int_bnd}
        \hat \nu & \leq 8\kappa_\rho \left\|\E[\aprojB]^{-1/2}\E\left[\aprojB^2+  \bar\lambda\aprojB K_\lambda^{-1}\aprojB\right]\E[\aprojB]^{-1/2}\right\| \\
        & \leq 8\kappa_\rho\left\|\E[\aprojB]^{-1/2}\E\left[ \aprojB+  \bar\lambda \aprojB K_\lambda^{-1}\aprojB\right]\E[\aprojB]^{-1/2}\right\| \\
        & \leq 8\kappa_\rho\bigg(1+ \bar \lambda \left\|\E[\aprojB]^{-1/2}\E\left[\aprojB K_\lambda^{-1}\aprojB\right]\E[\aprojB]^{-1/2}\right|\bigg) \tag{$*$}.
    \end{align*}
    Now, using the definition of $\aprojB$, we can express the middle term as follows:
    \begin{align*}
        \aprojB K_\lambda^{-1}\aprojB 
        & = \frac{1}{\hat L_{P_\B}^2} \Klambd^{1/2}I_{\B}^{T}(\Knys+\rho I)^{-1}I_\B(\Klambd^{1/2} \Klambd^{-1} \Klambd^{1/2})I_{\B}^{T}(\Knys+\rho I)^{-1}I_\B \Klambd^{1/2} \\
        & = \frac{1}{\hat L_{P_\B}^2} \Klambd^{1/2}I_{\B}^{T}(\Knys+\rho I)^{-1}I_{\B}I_{\B}^{T} (\Knys+\rho I)^{-1}I_\B \Klambd^{1/2} \\
        & = \frac{1}{\hat L_{P_\B}^2} \Klambd^{1/2}I_{\B}^{T} (\Knys+\rho I)^{-2}I_\B \Klambd^{1/2} \\
        & \preceq \frac{1}{\rho \hat L_{P_\B}^2}\Klambd^{1/2}I_{\B}^{T} (\Knys+\rho I)^{-1}I_\B \Klambd^{1/2} \\
        &= \frac{1}{\rho \hat L_{P_\B}}\aprojB \preceq \frac{1}{\rho}\aprojB. \\
    \end{align*}
    Here, the last inequality uses that by definition $\hat L_{P_\B}\geq 1$.
    By plugging the preceding upper bound into \eqref{eq:hat_nu_int_bnd}, we reach
    \begin{align*}
        \hat \nu & \leq 8\kappa_\rho\left(1+\frac{\bar\lambda}{\rho}\left\|\E[\aprojB]^{-1/2}\E\left[\aprojB K_\lambda^{-1}\aprojB\right]\E[\aprojB]^{-1/2}\right\|\right) \\
        & = 8\kappa_\rho+\frac{8\kappa_\rho \bar\lambda}{\rho} \\
        &= \frac{8\left(\rho+\bar \lambda\right)}{\lambda}.
    \end{align*}
    This proves the first inequality. 
    To prove the second inequality, recall that \cref{lem:convergence_rls} implies $\bar \lambda \leq \frac{1}{k}\sum_{j>k}\lambda_j(K_\lambda)$.
    Combining this with the preceding display, we conclude
    \begin{align*}
        \hat \nu & \leq \frac{8\left(\rho+\frac{1}{k}\sum_{j>k}\lambda_j(K_\lambda)\right)}{\lambda} \\
        & = 8 \frac{\rho +\frac{n-k}{k}\lambda+\frac{1}{k}\sum_{j>k}\lambda_j(K)}{\lambda} \\
        &\leq 8 \frac{\rho +\frac{n-k}{k}\lambda+\lambda}{\lambda} \\
        & = 8 \left(\frac{\rho}{\lambda}+\frac{n}{k}\right).
    \end{align*}
    Here, the second inequality follows from the fact that $\frac{1}{k}\sum_{j>k}\lambda_j(K)\leq \lambda$ when $k\geq \deff{\lambda}(K)$ \citep[Lemma 5.4]{frangella2023randomized}.
\end{proof}


\subsubsection{Convergence Proof of \asko{}}
To show the convergence of \asko{}, we use the Lyapunov function
\begin{equation}
\label{eq:askotch_lyapunov}
\tilde \Delta_t \coloneqq \|v_t-\wstar\|^2_{K_\lambda^{1/2}\E[\aprojB]^{-1}K_{\lambda}^{1/2}}+\frac{1}{\hat \mu}\|w_t-\wstar\|_{K_\lambda}^2,
\end{equation}
where $\hat \mu = \lmin(\E[\aprojB])$.
The Lyapunov function \eqref{eq:askotch_lyapunov} is identical to the one for NSAP except $\E[\aprojB]$ replaces $\E[\projB]$ and $\hat \mu$ replaces $\mu$.
The replacement stems from the fact that \asko{} computes the search direction with $\frac{1}{\hat L_{P_\B}}(\Knys+\rho I)^{-1}$ instead of $(\Kbb+\lambda I)^{-1}$.
\begin{proof}[Proof of \cref{thm:mast_conv_thm}, Item 2]
    The only difference between \asko{} and \nsap{} is that $\E[\aprojB]$ replaces $\E[\projB]$, so we can apply the same argument as \cite{gower2018accelerated} to show the convergence for the modified Lyapunov function \eqref{eq:askotch_lyapunov}.
    Thus, we only need to ensure $\E[\aprojB]$ satisfies the following condition from \cref{prop:nsap}:
    \[ \null(\E[\aprojB]) = \null(K_\lambda) = \{0\}.
    \]
    Indeed, this is necessary for the modified Lyapunov function to make sense; otherwise, $\E[\aprojB]$ is singular. 
    Under the hypotheses of \cref{thm:mast_conv_thm}, \cref{thm:appx_proj_small_eig} guarantees $\hat \mu>0$, which immediately implies
    \[
     \null(\E[\aprojB]) = \null(K_\lambda) = \{0\}.   
    \]
    Thus, we can invoke arguments in \cite{gower2018accelerated} to arrive at a modified version of \cref{prop:nsap} where $\mu$ and $\nu$ are replaced by $\hat \mu$ and $\hat \nu$ respectively.
    Applying this modified version of \cref{prop:nsap}, we deduce
    \begin{align*}
         \E[\tilde \Delta_t \mid w_t, v_t, z_t] \leq \left(1-\sqrt{\frac{\hat \mu}{\hat \nu}}\right)\tilde \Delta_{t-1}. 
    \end{align*}
    Applying the law of total expectation yields
    \begin{align*}
        \E[\tilde \Delta_t] \leq \left(1-\sqrt{\frac{{\hat \mu}}{\hat \nu}}\right)^{t}\tilde \Delta_{0}.
    \end{align*}
    Using the definition of $\tilde \Delta_t$, we reach
    \begin{align}
        \E[\|w_t-\wstar\|_{K_\lambda}^2] &\leq \left(1-\sqrt{\frac{\hat \mu}{\hat\nu}}\right)^{t}\left(\hat \mu \|w_0-\wstar\|_{K_\lambda^{1/2}\E[\aproj_\rho]^{-1}K_{\lambda}^{1/2}}^2 + \|w_0-\wstar\|_{K_\lambda}^2\right) \notag \\
        &\leq \left(1-\sqrt{\frac{\hat \mu}{\hat \nu}}\right)^{t}\left(\hat \mu \cdot \frac{1}{\hat \mu}\|w_0-\wstar\|_{K_\lambda}^2 + \|w_0-\wstar\|_{K_\lambda}^2\right) \notag \\
        &= 2\left(1-\sqrt{\frac{\hat \mu}{\hat \nu}}\right)^{t}\|w_0-\wstar\|_{K_\lambda}^2. \label{eq:asko_contraction} \tag{$*$}
    \end{align}
    To obtain a fine-grained convergence rate, we apply our bounds on $\hat \mu$ and $\hat \nu$.
    \cref{corollary:cnd_free_hat_mu} lower bounds $\hat \mu$ as
    \[
    \hat \mu \geq \frac{\lambda}{32\rho}\frac{k}{n}.
    \]
    On the other hand, \cref{lem:hat_nu_bnd} upper bounds $\hat \nu$ as
    \[
    \hat \nu \leq 8\left(\frac{\rho}{\lambda}+\frac{n}{k}\right) \leq 16\max\left\{\frac{\rho}{\lambda},\frac{n}{k}\right\}.
    \]
    Thus,
    \[
    \frac{1}{\hat \nu} \geq \frac{1}{16}\min\left\{\frac{\lambda}{\rho}, \frac{k}{n}\right\}.
    \]
    Combining these lower bounds, we deduce 
    \[
    \sqrt{\frac{\hat \mu}{\hat \nu}} \geq \frac{1}{16\sqrt{2}}\min\left\{\sqrt{\frac{\lambda}{\rho}}\frac{k}{n}, \sqrt{\frac{k}{n}}\frac{\lambda}{\rho}\right\}.
    \]
    Combining the preceding display with \eqref{eq:asko_contraction}, we conclude the following convergence guarantee:
    \[
    \E[\|w_t-\wstar\|_{K_\lambda}^2] \leq 2\left(1-\frac{1}{16\sqrt{2}}\min\left\{\sqrt{\frac{\lambda}{\rho}}\frac{k}{n}, \sqrt{\frac{k}{n}}\frac{\lambda}{\rho}\right\}\right)^{t}\|w_0-\wstar\|_{\Klambd}^2.
    \]
   
\end{proof}

\subsection{Proof of \cref{corr:log_lin_conv}}
\label{subsec:log_lin_conv}
\begin{proof}
    Our hypothesis that $\max_{i \in [n]} \ell^{\tilde{\lambda}}_i(K_\lambda) = \Theta \left( \frac{\deff{\tilde{\lambda}} (K_\lambda)}{n} \right)$ allows us to invoke 
    \cref{corollary:proj_analysis_uniform} to ensure that the conclusion of \cref{thm:mast_conv_thm} holds when \asko{} uses uniform sampling with a blocksize $b=\Theta\left(k\log^{3} n \right)$.
    Moreover, the assumption on $\rho$ guarantees that we are in convergence regime (i).
    Hence, the number of iterations required to obtain an $\epsilon$-approximate solution is $\bigO\left(\sqrt{c_\rho}\frac{n}{k}\log\left(\frac{1}{\epsilon}\right)\right)$.
    When a sparse sign embedding is used to construct $\Knys$, the per-iteration cost of \asko{} is $\bigOt(nb+br^2)$. 
    Combining this with the iteration complexity bound and $b = \Theta(k\log^3 n)$, we conclude that the total cost of \asko{} is given by
    \[
        \bigOt\left(\sqrt{c_\rho}\left(n^2+nr^2\right)\log\left(\frac{1}{\epsilon}\right)\right).    
    \]
    The claim now follows by observing that $r = \bigOt(\sqrt{n})$ as $\deff{\rho}(\lra{K}{b}) \leq \deff{\lambda}(K)= 
    \bigO(\sqrt{n})$.
\end{proof}

\preprintcontent{
\section{Additional Experimental Details}
\label{sec:experiments_appdx}

We provide experimental details (\cref{subsec:exp_kernels,subsec:data_hyperparams}) that are omitted from the main paper.
In addition, we compare \asko{} to running \pcg{} in single precision (\cref{subsec:pcg_single_precision}) and show ablation studies that were omitted from the main paper (\cref{subsec:ablation_additional}).

\subsection{Kernels Used in Experiments}
\label{subsec:exp_kernels}

Our experiments use Laplacian, \mtrn{}-5/2, and radial basis function (RBF) kernels, which are all defined by a bandwidth $\sigma$. The expressions for these kernels are
\begin{itemize}
    \item Laplacian:
    \[
    k(x, x') = \exp \left( -\frac{\|x - x'\|_1}{\sigma} \right).
    \]
    \item \mtrn{}-5/2: 
    \[k(x, x') = \left( 1 + \frac{\sqrt{5} \|x - x'\|_2}{\sigma} + \frac{5 \|x - x'\|_2^2}{3 \sigma^2} \right) \exp \left( - \frac{\sqrt{5} \|x - x'\|_2}{\sigma} \right).
    \]
    \item RBF: 
    \[
    k(x, x') = \exp \left( - \frac{\|x - x'\|_2^2}{2 \sigma^2} \right).
    \]
\end{itemize}

\subsection{Datasets and Hyperparameters Used in Experiments}
\label{subsec:data_hyperparams}
We provide more details regarding KRR hyperparameters, optimizer hyperparameters, inducing points, descriptions of classification and regression tasks, dataset generation/preprocessing, and time limits.
\cref{tab:datasets_ker_hyperparams} summarizes the datasets and KRR hyperparameters used in this paper.

\begin{table}[H]
    \centering
    \tiny
    \begin{tabular}{C{1.9cm}C{1.4cm}C{1.2cm}C{0.9cm}C{0.5cm}C{1.3cm}C{0.8cm}C{1.1cm}C{3cm}}
        Dataset & Task & $n$ & $n_{\mathrm{tst}}$ & $p$ & Kernel & $\sigma$ & $\lamunsc$ & Source \\
        \hline
        acsincome & Regression & 1,331,600 & 330,900 & 11 & RBF & Median & $10^{-6}$ & \openml{} (ID: 43141) \\
        \hline
        aspirin & Regression & 169,409 & 42,353 & 210 & \mtrn{}-5/2 & $\sqrt{p}$ & $10^{-9}$ & \sgdml \\
        \hline
        benzene & Regression & 502,386 & 122,597 & 66 & \mtrn{}-5/2 & $\sqrt{p}$ & $10^{-9}$ & \sgdml \\
        \hline
        cifar10 & Classification & 50,000 & 10,000 & 1,280 & Laplacian & 20 & $10^{-6}$ & \torchv \\
        \hline
        click\_prediction & Classification & 1,597,928 & 399,482 & 11 & RBF & Median & $10^{-6}$ & \openml{} (ID: 1218) \\
        \hline
        comet\_mc & Classification & 609,552 & 152,388 & 4 & RBF & Median & $10^{-6}$ & \openml{} (ID: 23397) \\
        \hline
        covtype\_binary & Classification & 464,809 & 116,203 & 54 & RBF & 0.1 & $3.8 \cdot 10^{-7}$ & \libsvm \\
        \hline
        ethanol & Regression & 444,073 & 111,019 & 36 & \mtrn{}-5/2 & $\sqrt{p}$ & $10^{-9}$ & \sgdml \\
        \hline
        fashion\_mnist & Classification & 60,000 & 10,000 & 1,280 & Laplacian & 20 & $10^{-6}$ & \torchv \\
        \hline
        higgs & Classification & 10,500,000 & 500,000 & 28 & RBF & 3.8 & $3.0 \cdot 10^{-8}$ & \libsvm \\
        \hline
        malonaldehyde & Regression & 794,589 & 198,648 & 36 & \mtrn{}-5/2 & $\sqrt{p}$ & $10^{-9}$ & \sgdml \\
        \hline
        miniboone & Classification & 104,051 & 26,013 & 50 & RBF & 5 & $10^{-7}$ & \openml{} (ID: 41150) \\
        \hline
        mnist & Classification & 60,000 & 10,000 & 1,280 & Laplacian & 20 & $10^{-6}$ & \torchv \\
        \hline
        naphthalene & Regression & 261,000 & 65,250 & 153 & \mtrn{}-5/2 & $\sqrt{p}$ & $10^{-9}$ & \sgdml \\
        \hline
        qm9 & Regression & 100,000 & 33,728 & 435 & Laplacian & 5,120 & $10^{-8}$ & \cite{ruddigkeit2012enumeration,ramakrishnan2014quantum} \\
        \hline
        salicylic & Regression & 256,184 & 64,047 & 120 & \mtrn{}-5/2 & $\sqrt{p}$ & $10^{-9}$ & \sgdml \\
        \hline
        susy & Classification & 4,500,000 & 500,000 & 18 & RBF & 3 & $10^{-6}$ & \libsvm \\
        \hline
        svhn & Classification & 73,256 & 26,032 & 1,280 & Laplacian & 20 & $10^{-6}$ & \torchv \\
        \hline
        taxi & Regression & 100,000,000 & 1,000,000 & 9 & RBF & 1 & $2 \cdot 10^{-7}$ & \href{https://github.com/toddwschneider/nyc-taxi-data}{nyc-taxi-data} \\
        \hline
        toluene & Regression & 354,232 & 88,558 & 105 & \mtrn{}-5/2 & $\sqrt{p}$ & $10^{-9}$ & \sgdml \\
        \hline
        uracil & Regression & 107,016 & 26,754 & 66 & \mtrn{}-5/2 & $\sqrt{p}$ & $10^{-9}$ & \sgdml \\
        \hline
        yearpredictionmsd & Regression & 463,715 & 51,630 & 90 & RBF & 7 & $2 \cdot 10^{-6}$ & \libsvm \\
        \hline
        yolanda & Regression & 320,000 & 80,000 & 100 & RBF & Median & $10^{-6}$ & \openml{} (ID: 42705) \\
        \hline
    \end{tabular}
    \caption{Datasets and hyperparameters for KRR problems in \cref{sec:experiments}. 
    $n$ is the number of training samples, $n_{\mathrm{tst}}$ is the number of test samples, $p$ is the dimension of the dataset, $\sigma$ is the kernel bandwidth, and $\lamunsc$ is the regularization parameter (before scaling by $n$). 
    ``Median'' corresponds to the median heuristic in \cite{gretton2012kernel}.}
    \label{tab:datasets_ker_hyperparams}
\end{table}

\subsubsection{KRR Hyperparameters}
We set the regularization $\lambda = n \lamunsc$, where $\lambda_{\mathrm{unsc}}$ is the unscaled regularization parameter in \cref{tab:datasets_ker_hyperparams}.
By default, we set the kernel bandwidth $\sigma$ based on the median heuristic from \cite{gretton2012kernel} and set the unscaled regularization parameter $\lamunsc$ equal to $10^{-6}$.
If $\sigma$ and/or $\lamunsc$ have been specified for a given dataset in previous work, we use those values instead; see \cref{tab:krr_hyperparam_sources} for more details.

\begin{table}[H]
    \centering
    \begin{tabular}{C{8cm}C{5cm}}
        Datasets & Source for $\sigma$ and/or $\lamunsc$ \\
        \hline
        aspirin, benzene, ethanol, malonaldehyde, napthalene, salicylic, toluene, uracil & \cite{epperly2025embrace} \\
        \hline
        cifar10, fashion\_mnist, mnist, svhn & \cite{abedsoltan2023toward} \\
        \hline
        covtype\_binary & \cite{avron2017faster} \\
        \hline
        higgs, susy, taxi, yearpredictionmsd & \cite{meanti2020kernel} \\
        \hline
        miniboone & \cite{frangella2023randomized} \\
        \hline
        qm9 & \cite{diaz2023robust} \\
        \hline
    \end{tabular}
    \caption{Previous works used to set KRR hyperparameters.}
    \label{tab:krr_hyperparam_sources}
\end{table}

\subsubsection{Inducing Points}
We sample inducing points uniformly without replacement from the training set.
\fal{} is run with $m \in \{10^4, 2 \cdot 10^4, 5 \cdot 10^4, 10^5, 2 \cdot 10^5\}$ inducing points for each dataset.
However, \fal{} does not scale beyond $5 \cdot 10^4$ inducing points in our experiments due to memory constraints.
\epro{} 3.0 is run with $m \in \{10^4, 2 \cdot 10^4, 5 \cdot 10^4, 10^5, 2 \cdot 10^5, 5 \cdot 10^5, 10^6\}$ inducing points for each dataset.
The run for \fal{} displayed for each dataset is the one that uses the largest possible number of inducing points.
The run for \epro{} 3.0 displayed for each dataset is the one that uses the largest possible number of inducing points without diverging.

\subsubsection{Descriptions of Classification and Regression Tasks}
\textit{Classification: computer vision.}
We use the cifar10 \citep{krizhevsky2009learning}, fashion\_mnist \citep{xiao2017fashionmnist}, mnist \citep{lecun1998mnist}, and svhn \citep{netzer2011reading} datasets, which are popular benchmarks in computer vision.
Following \cite{abedsoltan2023toward}, we extract features from these datasets using pretrained MobileNetV2 \citep{sandler2018mobilenetv2} and use a Laplacian kernel. 
For simplicity, we perform one vs. all classification, separating the examples with the smallest label (e.g., $0$ in mnist) from all other examples in the data.

\textit{Classification: particle physics.}
We use the miniboone \citep{roe2005boosted}, comet\_mc \citep{cometmontecarlo}, susy, and higgs \citep{baldi2014searching} datasets from particle physics.
The tasks associated with these datasets are detecting neutrino oscillations, identifying muon-to-electron conversions, discovering the Higgs boson, and searching for supersymmetry particles.
Following \cite{meanti2020kernel}, we use an RBF kernel.

\textit{Classification: other applications.}
We use the covtype\_binary \citep{collobert2001parallel} and click\_prediction \citep{aden2012kdd} datasets.
covtype\_binary is from ecological modeling, and the task is to predict the type of forest cover in a particular region of the United States.
click\_prediction is from online advertising, and the task is to predict whether or not a user will click on a given advertisement.
We use an RBF kernel for both datasets.

\textit{Regression: computational chemistry.}
We use the qm9 dataset \citep{ruddigkeit2012enumeration,ramakrishnan2014quantum} and eight different datasets corresponding to important molecules in chemistry (aspirin, benzene, ethanol, malonaldehyde, naphthalene, salicylic acid, toluene, and uracil) \citep{chimela2017machine}.
The task for qm9 is to predict the highest-occupied-molecular-orbital energy for organic molecules.
The task for the eight molecule datasets is to predict potential energies of atomic systems, which could be used to accelerate molecular dynamics simulations.
Following \cite{stuke2019chemical,diaz2023robust}, we use a Laplacian kernel for qm9; 
following \cite{chimela2017machine,epperly2025embrace}, we use a \mtrn{}-5/2 kernel for the molecule datasets.

\textit{Regression: other applications.}
We use the yolanda \citep{guyon2019analysis}, yearpredictionmsd \citep{bertinmahieux2011yearpredictionmsd}, and acsincome \citep{ding2021retiring} datasets.
Both yolanda and yearpredictionmsd are from music analysis, and the task is to predict the release year of songs given audio features.
acsincome is from socioeconomics, and the task is to predict incomes given demographic features such as age, employment, and education.
We use an RBF kernel for all three datasets.

\subsubsection{Dataset Generation and Preprocessing}
For taxi, we use yellow Taxi data between January 2009 and December 2015, removing outliers (trips with a duration of more than 5 hours) and trips whose pickup or dropoff locations fall outside of New York City\footnote{This preprocessing is consistent with \cite{meanti2020kernel}.}.
We then randomly subsample $10^8$ points to form a training set and $10^6$ points to form a test set.

We preprocess qm9 according to \cite{diaz2023robust} and preprocess the molecule datasets according to \cite{epperly2025embrace}. 

We always standardize features of each dataset before running experiments.
We also subtract the means of the targets from aspirin, benzene, ethanol, malonaldehyde, napthalene, salicylic, toluene, uracil, yearpredictionmsd, yolanda, and acsincome before running experiments.
By default, we use a 0.8/0.2 train/test split for each dataset, unless such a split has already been described in the literature.
All additional preprocessing details can be found in the \href{\codeurl}{codebase} for this paper.

\subsubsection{Time Limits for Datasets}
The time limit for each dataset is given in \cref{tab:time_limits}.
These time limits increase with the number of samples $n$.

\begin{table}[H]
    \centering
    \begin{tabular}{C{8cm}C{3cm}}
        Datasets & Time limit (s) \\ \hline
        cifar10, fashion\_mnist, mnist, svhn & 1800 \\ \hline
        aspirin, miniboone, naphthalene, qm9, salicylic, toluene, uracil, yearpredictionmsd, yolanda & 3600 \\ \hline
        benzene, comet\_mc, covtype\_binary, ethanol, malonaldehyde & 7200 \\ \hline
        acsincome, click\_prediction, higgs, susy & 10800 \\ \hline
        taxi & 86400 \\ \hline
    \end{tabular}
    \caption{Time limits for each dataset.}
    \label{tab:time_limits}
\end{table}


\subsection{Comparisons to \fal{} and \pcg{} in Single Precision}
\label{subsec:pcg_single_precision}
The main paper shows \fal{} and \pcg{} in double precision.
However, using double precision causes \fal{} and \pcg{} to use twice as much memory and four times as much compute per iteration.
Here we demonstrate that \asko{} still outperforms these methods when they are run in single precision on the datasets used in the performance comparisons (\cref{subsec:performance_comparisons}).
We present a performance plot for this setting in \cref{fig:perf_float32}.
On classification tasks, \asko{} clearly outperforms \pcg{} and \fal{}.
On regression tasks, \asko{} initially appears to be worse than \pcg{} for most of the time budget.
However \asko{} eventually takes the lead over PCG.
This sudden improvement towards the end of the time budget occurs since \asko{} does not saturate the test MAE on the molecule datasets within the time budget.

\begin{figure}[htbp]
    \centering
    \includegraphics[width=0.8\linewidth]{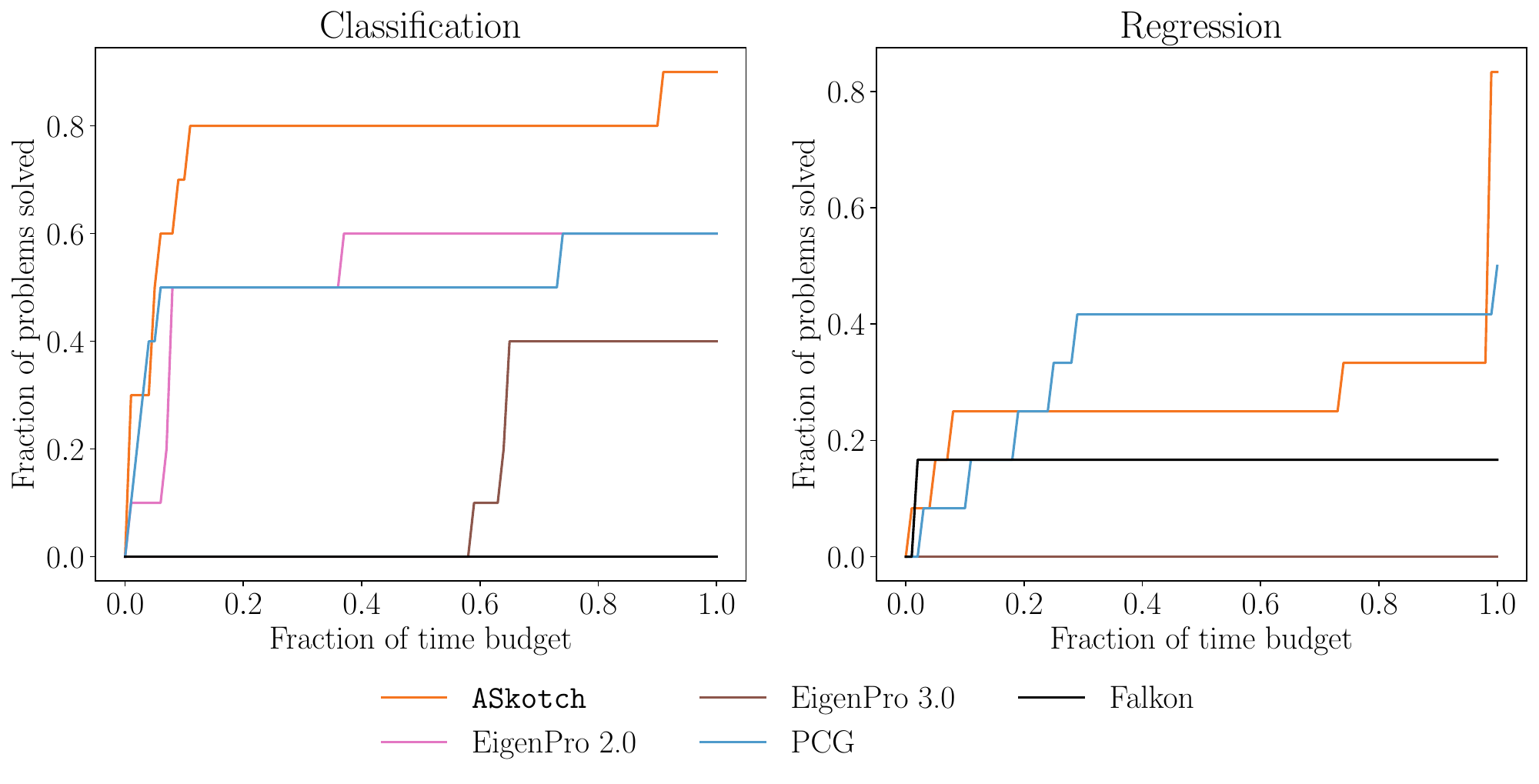}
    \caption{Performance comparison between \asko{} and competitors on 10 classification and 13 regression tasks.
    We designate a classification problem as ``solved'' when the method reaches within 0.001 of the highest test accuracy found across all the optimizer + hyperparameter combinations.
    We designate a regression problem as ``solved'' when the method reaches within 1\% of the lowest test MAE (in a relative sense) found across all the optimizer + hyperparameter combinations.
    \pcg{} and \fal{} are run in single precision.
    \asko{} outperforms the competition on both classification and regression.}
    \label{fig:perf_float32}
\end{figure}

\subsection{Ablation Plots Omitted from Main Paper}
\label{subsec:ablation_additional}
We provide results from the ablation study that were omitted from \cref{subsec:ablation}.
We provide an ablation study for each set of performance comparisons in \cref{subsec:performance_comparisons}.
The results are shown in \cref{fig:vision_ablation,fig:tabular_classification_ablation,fig:qm9_ablation,fig:tabular_regression_ablation}.
The takeaways from these figures are similar to those in \cref{subsec:ablation}:
\nys{} approximation with ``damped'' rho consistently yields the best performance,
acceleration is often beneficial, and approximate RLS sampling does not provide a noticeable improvement over uniform sampling.

\begin{figure}[htbp]
    \centering
    \includegraphics[width=\linewidth]{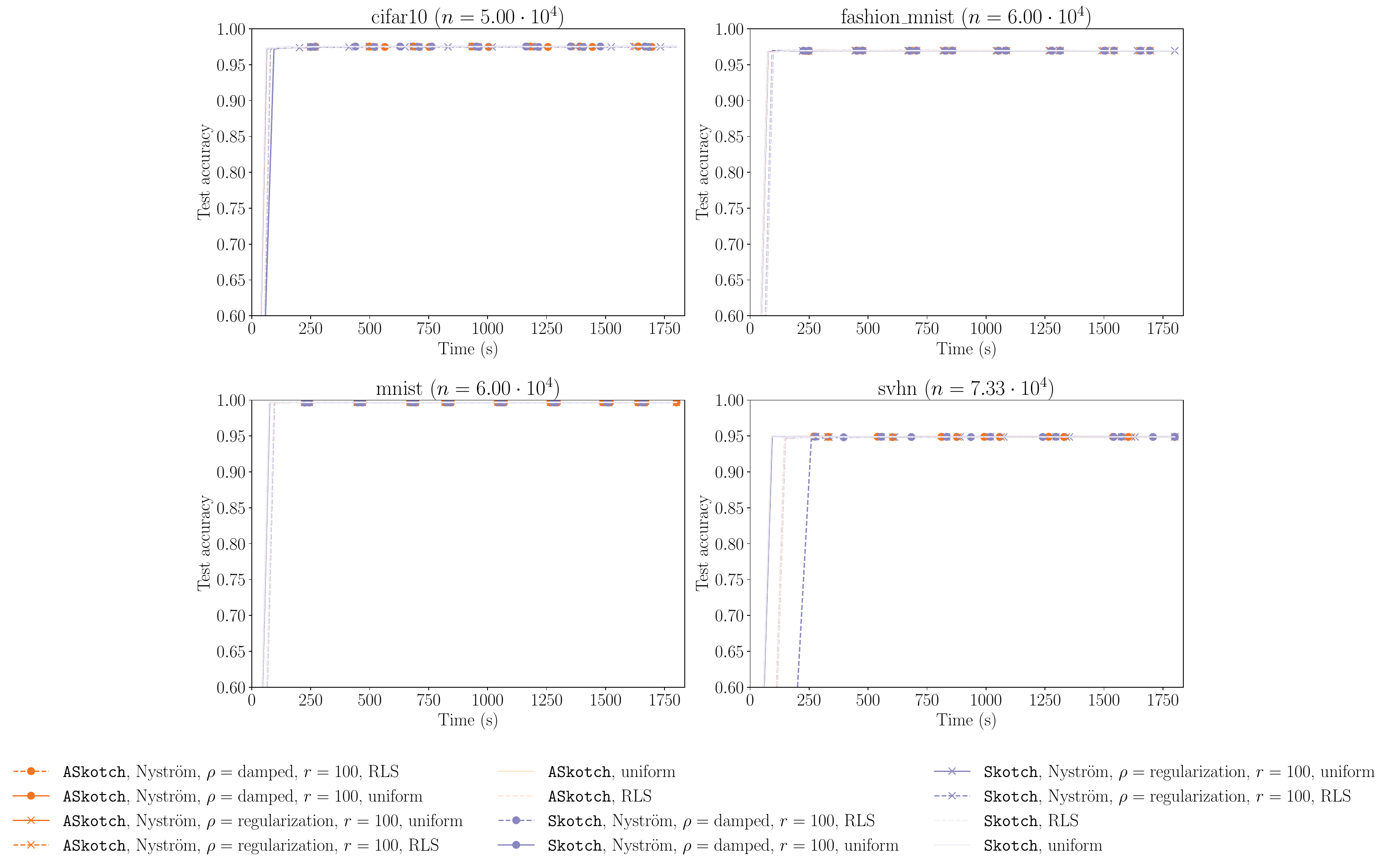}
    \caption{Ablation study of \sko{} and \asko{} on classification tasks from computer vision.}
    \label{fig:vision_ablation}
\end{figure}

\begin{figure}[htbp]
    \centering
    \includegraphics[width=\linewidth]{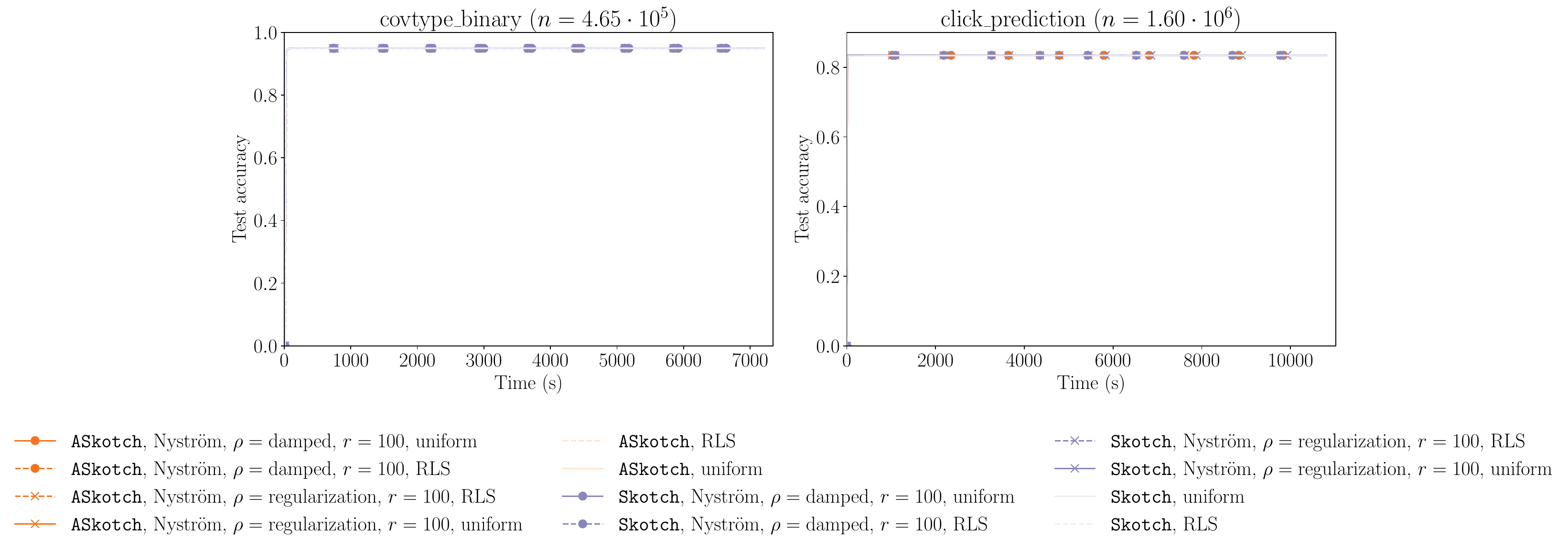}
    \caption{Ablation study of \sko{} and \asko{} on classification tasks from ecological modeling and online advertising.}
    \label{fig:tabular_classification_ablation}
\end{figure}

\begin{figure}[htbp]
    \centering
    \includegraphics[width=\linewidth]{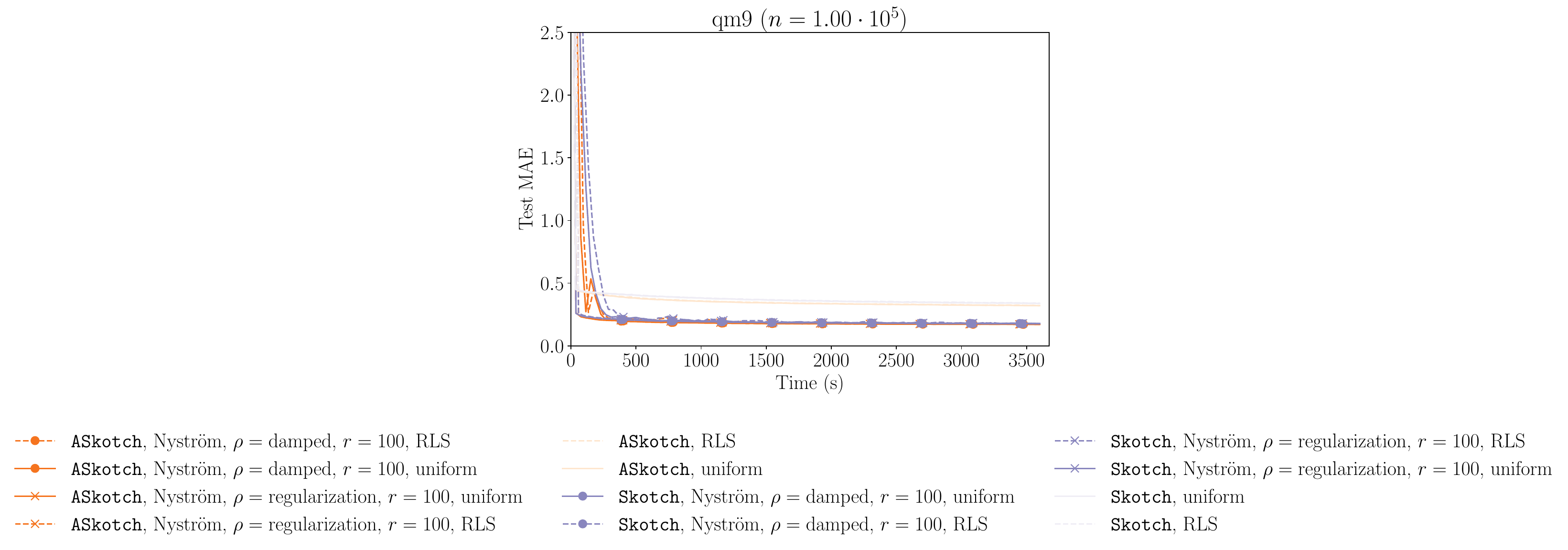}
    \caption{Ablation study of \sko{} and \asko{} on regression tasks for the qm9 dataset.}
    \label{fig:qm9_ablation}
\end{figure}

\begin{figure}[htbp]
    \centering
    \includegraphics[width=\linewidth]{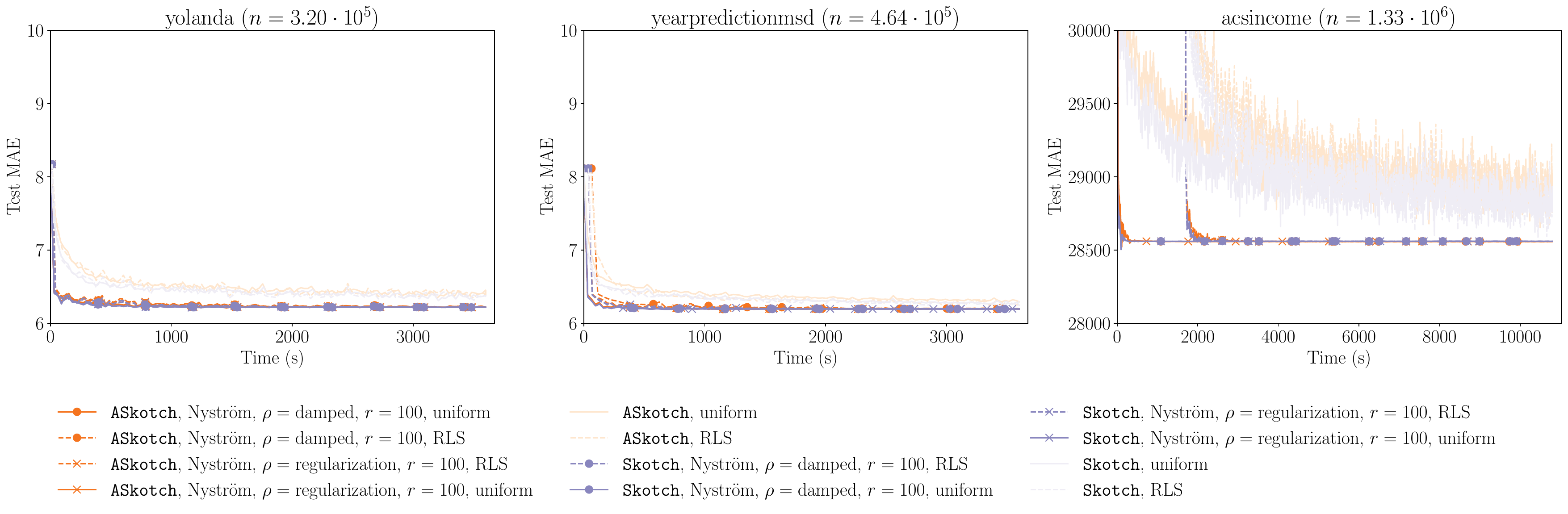}
    \caption{Ablation study of \sko{} and \asko{} on regression tasks from music analysis and socioeconomics.}
    \label{fig:tabular_regression_ablation}
\end{figure}
}

\vskip 0.2in
\bibliography{references}

\end{document}